\documentclass{article}
\usepackage[margin=1in]{geometry} 

\usepackage{bm} 
\usepackage{color,soul}

\usepackage{microtype}
\usepackage[utf8]{inputenc} 
\usepackage[T1]{fontenc}    
\usepackage{url}            
\usepackage{nicefrac}       
\usepackage{xcolor}         
\usepackage{sectsty}
\usepackage{microtype,mathrsfs}
\usepackage{graphicx}
\usepackage{subfigure}
\usepackage{booktabs} 
\usepackage{float}
\usepackage{amsfonts}       
\usepackage{caption}
\usepackage{mathtools,amssymb,comment}
\usepackage{amsmath}
\usepackage{bbm}
\usepackage{breqn} 
\usepackage{enumitem}
\usepackage{tcolorbox}
\usepackage{hyperref}
\sectionfont{\fontsize{12}{15}\selectfont}
\subsectionfont{\fontsize{10}{15}\selectfont}

\definecolor{darkred}{RGB}{230,0,0}
\definecolor{darkgreen}{RGB}{0,170,0}
\definecolor{darkblue}{RGB}{0,0,150}

 \hypersetup{colorlinks=true, linkcolor=darkred, citecolor=darkgreen, urlcolor=darkblue}
\newcommand\blfootnote[1]{%
  \begingroup
  \renewcommand\thefootnote{}\footnote{#1}%
  \addtocounter{footnote}{-1}%
  \endgroup
}
\usepackage{amsthm}

\usepackage{stfloats}

\newcommand{\widesim}[2][1.5]{
  \mathrel{\overset{#2}{\scalebox{#1}[1]{$\sim$}}}}

\newcommand{\Sn}{\Large{\bm{\Sigma}}_n}

\newcommand{\db}{\bm{\delta}}

\makeatletter
\newcommand*{\rom}[1]{\expandafter\@slowromancap\romannumeral #1@}
\makeatother

\newcommand{\lac}{{\bm{\lambda}}}
\newcommand{\rhoc}{{\bm{\rho}}}
\newcommand{\thc}{{\bm{\theta}}}

\newcommand{\mathleft}{\@fleqntrue\@mathmargin0pt}
\newcommand{\mathcenter}{\@fleqnfalse}

\makeatletter
\newcommand{\ssymbol}[1]{^{\@fnsymbol{#1}}}
\makeatother
\newcommand{\env}[3]{\mathcal{M}_{{#1}}\left({#2};{#3}\right)}
\newcommand{\prox}[3]{\mathcal{P}_{{#1}}\left({#2};{#3}\right)}

\newcommand{\Th}{\Theta_n}
\newcommand{\Thp}{\Theta^\perp_n}

\newcommand{\envdx}[3]{\mathcal{M}^{\prime}_{{#1},1}\left({#2};{#3}\right)}

\newcommand{\envdla}[3]{\mathcal{M}^{\prime}_{{#1},2}\left({#2};{#3}\right)}

\newcommand{\R}{\mathbb{R}}

\newcommand{\al}{\alpha}

\newcommand{\Lm}{\mathcal{L}}

\newcommand{\rP}{\stackrel{P}{\longrightarrow}}

\providecommand{\abs}[1]{\lvert#1\rvert}
\providecommand{\norm}[1]{\lVert#1\rVert}

\newcommand{\simiid}{\widesim{\text{\small{iid}}}}

\newcommand{\Pro}{\mathbb{P}}

\theoremstyle{theorem}
\newtheorem{ass}{Assumption}
\theoremstyle{remark}
\newtheorem{remark}{Remark}

\newcommand{\eps}{\varepsilon}

\newcommand{\sign}{\mathrm{sign}}

\newcommand{\vb}{\mathbf{v}}
\newcommand{\ub}{\mathbf{u}}

\newcommand{\one}{\mathbf{1}}

\newcommand{\E}{\mathbb{E}}                    
\newcommand{\la}{r}                     

\newcommand{\nn}{\notag}

\newcommand{\G}{\mathbf{G}}

\newcommand{\ellb}{\boldsymbol{\ell}}
\newcommand{\Ellb}{\pmb{\mathcal{L}}}

\newcommand{\x}{\mathbf{x}}

\newcommand{\w}{\mathbf{w}}

\newcommand{\g}{\mathbf{g}}

\newcommand{\y}{\mathbf{y}}
\newcommand{\s}{\mathbf{s}}
\newcommand{\z}{\mathbf{z}}

\newcommand{\h}{\mathbf{h}}

\newcommand{\Nn}{\mathcal{N}}

\newcommand{\beq}{\begin{equation}}
\newcommand{\eeq}{\end{equation}}
\newcommand{\bea}{\begin{align}}
\newcommand{\eea}{\end{align}}

\def\bea#1\eea{\begin{align}#1\end{align}}
\usepackage{xcolor}
\usepackage{parskip}
\theoremstyle{plain}
\newtheorem{theorem}{Theorem}
\newtheorem{corollary}{Corollary}[theorem]
\newtheorem{lemma}[theorem]{Lemma}
\newtheorem{proposition}[theorem]{Proposition}

\usepackage{hyperref}

\date{}

\title{Asymptotic Behavior of Adversarial Training \\in Binary Classification}
\author{Hossein Taheri, Ramtin Pedarsani, and Christos Thrampoulidis\blfootnote{All authors are with the Department of Electrical and Computer Engineering, University of California, Santa Barbara.}}
\begin{document}
\maketitle

\begin{abstract}
It has been consistently reported that many machine learning models are susceptible to adversarial attacks i.e., small additive adversarial perturbations applied to data points can cause misclassification. Adversarial training using empirical risk minimization is considered to be the state-of-the-art method for defense against adversarial attacks. Despite being successful in practice, several problems in understanding generalization performance of adversarial training remain open. In this paper, we derive precise theoretical predictions for the performance of adversarial training in binary classification. We consider the high-dimensional regime where the dimension of data grows with the size of the training data-set at a constant ratio. Our results provide exact asymptotics for standard and adversarial test errors of the estimators obtained by adversarial training with $\ell_q$-norm bounded perturbations ($q \ge 1$) for both discriminative binary models and generative Gaussian-mixture models with correlated features. Furthermore, we use these sharp predictions to uncover several intriguing observations on the role of various parameters including the over-parameterization ratio, the data model, and the attack budget on the adversarial and standard errors.
\end{abstract}
\addtocontents{toc}{\protect\setcounter{tocdepth}{0}}

\section{Introduction}
Most machine learning algorithms ranging from simple linear classifiers to complex deep neural networks have been shown to be prone to adversarial attacks, i.e., small additive perturbations to the data that cause the model to predict a wrong label \cite{szegedy2013intriguing,moosavi2016deepfool}. The requirement for robustness against adversaries is crucial for the safety of systems that rely on decisions made by these algorithms (e.g., in self-driving cars). With this motivation, over the past few years, there have been remarkable efforts by the research community to construct defenses against those adversaries, e.g., see \cite{silva2020opportunities,chakraborty2018adversarial} for a survey.
%
Among many defense methods that have been proposed, the state-of-the-art defense
is to train the machine learning model with adversarial examples \cite{goodfellow2014explaining,madry2017towards}, which is also known as adversarial training. However, despite major recent progress in the study and implementation of adversarial training, its efficacy has been mainly shown empirically without providing much theoretical understanding. Indeed, many questions regarding its theoretical properties remain open even for simple stylized models. For instance, how does the adversarial/standard error depends on the adversary's budget during training time and test time? How do they depend on the over-parameterization ratio that is the ratio of dimension to number of data points? What is the role of the chosen loss function?



In this paper, we consider the adversarial training problem for $\ell_q$-norm bounded perturbations in classification tasks, which solves the following robust empirical risk minimization (ERM) problem:
\bea\label{eq:adv_train_erm}
\min_{\thc_n \in \R^n} \;\frac{1}{m}\sum_{i=1}^m  \max_{\|\db_i\|_{q}\le\, \eps_{\rm tr}}\Lm \left(y_i, f_{\thc_n}(\x_i+\db_i)\right) + \la\|\thc_n\|_2^2.
\eea
Here,  $\{(\x_i,y_i)\}_{i\in[m]}$ is the training set, $\db_i\in \R^p$ are the perturbations with $p$ the dimension of the feature space, $f_{\thc_n}:\R^p\rightarrow\R$ is a model parameterized by a vector $\thc_n\in\R^n$, $\eps_{\rm tr}$ is a user-specified tunable parameter that can be interpreted as the adversary's budget during training, and $\la$ is the ridge-regularization parameter. Once the robust classifier $\widehat{\thc_n}$ is obtained by \eqref{eq:adv_train_erm}, the 
\emph{adversarial error / robust classification error} is given by 
$\E_{\x,y}\big[\max_{\|\db\|_q\leq \eps_{\rm ts}} \one_{\{y\neq \widehat{f}_{\widehat{\thc_n}}(\x+\db)\}}\big],$
where $\one_{\{\cdot\}}$ is the $0/1$-indicator function, $(\x,y) \in \R^p\times\{\pm 1\}$ is a test sample drawn from the same distribution as that of the training dataset, $\eps_{\rm ts}$ is the budget of the adversary, and $\widehat{f}_{\widehat{\thc_n}}$ uses the trained parameters $\widehat{\thc_n}$ and the fresh sample $\x$ to output a label guess.  The standard classification error is given by the same formula by simply setting $\eps_{\rm ts}=0$.

The goal of this paper is to precisely analyze the performance of adversarial training in \eqref{eq:adv_train_erm} for binary classification of certain statistical data models. In our proof we use the Convex-Gaussian-Min-max-Theorem (CGMT) \cite{sto,stoLASSO,COLT} and in particular its applications to the convex ERM that enables its precise analysis \cite{master,montanari2019generalization,salehi2019impact,taheri2020sharp,taheri2021fundamental}. However, compared to previous works, we develop a new analysis for robust optimization. 

Our main contributions are summarized as follows: 

\noindent$\bullet$~We precisely analyze, for the first time, the performance of adversarial training with $\ell_2$ and $\ell_\infty$ attacks in binary classification for two important data models of Gaussian Mixtures and Generalized Linear Models. See Sections \ref{sec:ell_inf} and \ref{sec:ell_2}.

\noindent$\bullet$~Our approach is general, allowing us to characterize the role of feature correlation, regularization and general $\ell_q$ attacks with $q\ge1$. In particular, our proof technique allows for non-isotropic features, yielding novel theoretical results even for non-adversarial convex regularized ERM settings (i.e., when $\eps_{\rm tr}=\eps_{\rm ts}=0$). We will elaborate on our technical approach in Section \ref{sec:sketch}.

\noindent$\bullet$~Numerical illustrations in Section \ref{sec:num} show tight agreements between our theoretical and empirical results and also allow us to draw intriguing conclusions regarding the behavior of adversarial and standard errors as functions of key problem parameters such as the sampling ratio $\delta:=m/n$, the budget of the adversary $\eps_{\rm ts}$, and the robust-optimization hyper-parameter $\eps_{\rm tr}$ in our studied settings. Moreover, we observe interesting phonemena by comparing our results with the Bayes optimal robust errors.

\subsection{Prior Works}
Relevant to the flavour of our results, the recent work \cite{javanmard2020reg} studies precise tradeoffs and performance analysis in adversarial training with linear regression with $\ell_2$ perturbations and isotropic Gaussian data. Compared to \cite{javanmard2020reg}, our results hold for binary models, general $\ell_q$ perturbations with $q\ge1$, and non-isotropic features with mild assumptions on the covariance matrix. Moreover, we consider the regularized ERM allowing us to study the behavior of adversarial training in the over-parameterized regime. Similar results are only derived in a \emph{contemporaneous} work by \cite{javanmard2020precise}. On the one hand, compared to \cite{javanmard2020precise} our analysis applies to both discriminative and generative data models and also to the \emph{regularized} ERM with Positive Definite (PD) covariance matrices. We also compare the results of adversarial training with the Bayes robust ones. On the other hand, \cite{javanmard2020precise} extends their analysis to the support vector machines. Our analysis of correlated features was motivated by \cite{montanari2019generalization}. However, the work of \cite{montanari2019generalization} considers standard support vector machines, whereas we consider regularized ERM methods for adversarial training. The recent works \cite{mei2019generalization,goldt2020modeling,dhifallah2020precise,dhifallah2021inherent,ghorbani2019limitations} have characterized precise error of random features and neural tangent models in high-dimensions. The Adversarial Bayes risk for the Gaussian-mixture model has been characterized in \cite{bhagoji2019lower,dan2020sharp,dobriban2020provable}. The references \cite{charles2019convergence,allen2020feature,xing2020generalization} address few theoretical properties of adversarial training. 
The prior work \cite{min2020curious} considers adversarial training with linear loss in order to analyze the sample complexity of robust estimators. Another line of work studies the trade-offs between the standard and adversarial errors e.g., see \cite{tsipras2018robustness,raghunathan2019adversarial,zhang2019theoretically,dobriban2020provable}. The benefits of unlabeled data in robustness have been investigated in \cite{raghunathan2020understanding,carmon2019unlabeled}, among several other works. 

 \subsection*{Notation}
Letting $\delta(x)$ denote a Dirac delta mass on $x$, the empirical distribution of a vector $\x\in\R^n$ is given by $\frac{1}{n}\sum_{i=1}^n \delta(\x_i)$. The Wasserstein-$k$ distance between two measures $\rho_1,\rho_2$ is defined as $W_k(\rho_1,\rho_2)\triangleq\left(\inf _{\rho\in\mathrm{P}} \mathbb{E}_{(X, Y) \sim \rho}|X-Y|^{k}\right)^{1 / k}$, where $\mathrm{P}$ denotes all couplings of $\rho_1$ and $\rho_2$. We say that a sequence of probability distributions $\mu_n$ converges in Wasserstein-$k$ distance to a probability distribution $\mu$, if $W_k(\mu_n,\mu)\rightarrow 0$ as $n\rightarrow\infty$. The Gaussian $Q$-function is denoted by $Q(\cdot)$. $\odot$ denotes the element-wise multiplication. For a sequence of random variables $X_{m,n}$ that converges in probability to some constant $c$ in the proportional asymptotic limit, we write $X_{m,n}\rP c$. We further need to recall the definition of the \emph{Moreau envelope function}. We write
\bea\label{eq:ME_def}
\env{f}{x}{\kappa}\triangleq\min_{v}\frac{1}{2\kappa}(x-v)^2 + f(v),
\eea
for the Moreau envelope of the function $f:\R\rightarrow\R$ at $x\in\R$ with parameter $\kappa>0$. 

\section{Problem Formulation}\label{sec:pf}
In this section, we describe the data model, the specific form of \eqref{eq:adv_train_erm}, and the asymptotic regime for which our results hold. After this section, it is understood that all our results hold in the setting described here without any further explicit reference.  

\subsection{Data Model}\label{sec:data_models}
We study two stylized models for binary classification.
\paragraph{Gaussian Mixture Models.} The first model is a Gaussian Mixture model (GMM) where the conditional distribution of the feature vectors is an isotropic Gaussian with mean $\pm\thc^\star$, depending on the label $y_i$. Formally, the GMM model assumes
\bea\label{eq:G_mix}
\x_i | y_i \;\sim\; \mathcal{N}\left(y_i\thc^\star_n,\Sn\right), \;\;\;\;\Pro(y_i=1) = \pi\in[0,1].
\eea

\paragraph{Generalized Linear Models.} The second model is a generalized linear model (GLM) with binary link function. Specifically, assume that the label $y_i\in\{\pm1\}$ associated with the feature vector $\x_i$ is generated as
%
\bea\label{eq:binarymodel}
y_i = \psi\left(\langle\thc^\star_n,\x_i\rangle\right), \;\;\;\;\x_i\sim\Nn(\mathbf{0},\Sn),
\eea
for a possibly random link function $\psi:\R\rightarrow\{\pm1\}$. This includes the well-known Logistic and Signed models, by letting $\Pro(\psi(x)=1) = 1/(1+\exp(-x))$ and $\psi(x) = \sign(x)$, respectively. 

We assume that the underlying (unknown) vector of regressors $\thc_{n}^\star\in\R^n$, and the covariance matrix $\Sn\in\R^{n\times n}$, satisfy the following assumptions.
\begin{ass}\label{ass:1}
The minimum and maximum eigenvalues of the covariance matrices $\Sn$ satisfy $0<c<\lambda_{\min}(\Sn)$ and $\lambda_{\max}(\Sn)<C<\infty$. 
\end{ass}
\begin{ass}\label{ass:2}
Denoting $\zeta_n\triangleq ({\thc^\star_n}^\top\Sn\thc^\star_n)^{1/2}$ for GLM and $\widetilde{\zeta_n}\triangleq ({\thc^\star_n}^\top\Sn^{-1}\thc^\star_n)^{1/2}$ for GMM, we define their high-dimensional limits as $\zeta$ and $\widetilde{\zeta}$, i.e., $\zeta_n\rP \zeta$ and $\widetilde{\zeta_n}\rP \widetilde\zeta$. Moreover, for both models we assume without loss of generality $\|\thc_n^\star\|_2\,\rP1$.
\end{ass}
\begin{ass}\label{ass:3}
Let $\Sn=\mathbf{U}_n\mathbf{\Lambda}_n\mathbf{U}_n^\top$ be the eigen-decomposition of $\Sn$ and let $\lambda_{n,i}$ denote the $i$'th entry on the diagonal of $\mathbf{\Lambda_n}$. Denote $\vb_n\triangleq\mathbf{U}_n^\top{\thc_{n}^\star}$. Then the joint distribution of $(\sqrt{n}\thc^\star_{n,i},\lambda_{n,i},\sqrt{n}\vb_{n,i})$, $i\in[n]$, converges in Wasserstein-2 distance to a probability distribution $\Pi$ in $\R\times\R_+\times\R$, i.e., 
$$
\frac{1}{n}\sum_{i=1}^n \delta(\sqrt{n}\thc^\star_{n,i},\lambda_{n,i},\sqrt{n}\vb_{n,i})\;\xrightarrow{W_2} \;\Pi.
$$
\end{ass}
The assumption on $\|\thc_n^\star\|_2$ is without loss of generality for GLM since $\|\thc_n^\star\|_2$ can be absorbed in the link function $\psi$. Similarly for GMM, it can be relaxed in a straightforward way. We remark that while the Gaussian distribution assumption on feature vectors is crucial for theoretical analysis, our empirical results suggest that this assumption can be relaxed to include the family of sub-Gaussian data distributions. We discuss this universality property in Appendix \ref{sec:add_num_exp}.
%
\subsection{Asymptotic Regime}
We consider an asymptotic regime in which the size  of the training set $m$ and the dimension of the feature space $n$  grow large at a proportional rate. Formally, $m, n \rightarrow \infty$ at a fixed ratio $\delta= m/n$.
\subsection{Robust Learning}
Let $\widehat{\thc_n}$ be a linear classifier trained on data generated according to either models \eqref{eq:G_mix} or \eqref{eq:binarymodel}. As is typical, given $\widehat{\thc_n}$, a decision is made about the label of a fresh sample $\x$ based on $\sign(\langle\x,\widehat{\thc_n}\rangle)$. Thus, letting $y$ be the label of a fresh sample $\x$, the \emph{Standard Test Error} is given by
\bea\label{eq:gen_error_binarymodel}
\; \mathcal{E}(\widehat{\thc_n})\triangleq  \E_{\x,y}\left[ \one_{\left\{y\neq \sign\left(\langle\x,\widehat{\thc_n}\rangle\right)\right\}}\right].
\eea
Here, the expectation is over a fresh pair $(\x,y)$ also generated according to either the GLM or the GMM model.
Next, we define the adversarial error with respect to a worst-case $\ell_q$-norm bounded additive perturbation. Let $\eps_{\rm ts}\ge0$ be the budget of the adversary. Then, the \emph{Adversarial Test Error} is defined as follows:
%
%
\bea\label{eq:adv_error_lin}
 \;\; \mathcal{E}_{\ell_q,\eps_{\rm ts}}(\widehat{\thc_n})\triangleq \E_{\x,y}\left[\max_{\|\db\|_q\leq\eps_{\rm ts}} \one_{\left\{y\neq \sign\left(\langle\x+\db,\widehat{\thc_n}\rangle\right)\right\}}\right].
\eea

Adversarial training leads to a classifier $\widehat{\thc_n}$ that solves the following robust optimization problem tailored to binary classification:
\bea\label{eq:main_erm}
\widehat{\thc_n} := \arg\min_{\thc_n \in \R^n}\sum_{i=1}^m\;\max_{\|\db_i\|_{q}\le\, \eps_{\rm tr}}\Lm \left(y_i \left\langle\x_i+\db_i,\thc_n\right\rangle\right) + \la\|\thc_n\|_2^2\, .
\eea
The loss function $\Lm:\R\rightarrow\R$ is chosen as a convex approximation to the 0/1 loss. Specifically, throughout the paper, we assume that $\Lm$ is convex and decreasing. This includes popular choices such as the logistic, hinge and exponential losses. 

\section{Asymptotics for Adversarial Training with $\pmb{\ell_\infty}$ Perturbations}\label{sec:ell_inf}

In this section, we focus on the case of bounded $\ell_\infty$-perturbations, i.e. the adversarial error in \eqref{eq:adv_error_lin} is considered for $q=\infty$.  
Specifically, let $\widehat{\thc_n}$ be a solution to  the following robust minimization:
\bea\label{eq:main_erm_Linfty}
\hspace{-0.1in}\min_{\thc_n \in \R^n} \;\frac{1}{m}\sum_{i=1}^m  \max_{\|\db_i\|_{\infty}\le\, \frac{\eps_{\rm tr}}{\sqrt{n}}}\Lm \left(y_i \left\langle\x_i+\db_i,\thc_n\right\rangle\right) +\la\|\thc_n\|_2^2.
\eea

In our asymptotic setting, $\eps_{\rm tr}$ is of constant order and the  factor $1/\sqrt{n}$ in front of it is the proper normalization needed to obtain non-trivial results. We explain this normalization further in Section \ref{sec:sketch}. We consider the case of diagonal covariance matrix i.e., $\Sn=\mathbf{\Lambda}_n$ here and defer the general case of PD covariance matrix to the appendix where we also discuss how final expressions simplify in the case of isotropic features.

Before presenting our main result, we need to introduce some necessary definitions. We define the following min-max optimization over eight scalar variables. Denote $\bar{\vb}\triangleq(\al,\tau_1,w,\mu,\tau_2,\beta,\gamma,\eta)$ and define $f:\R^8\rightarrow\R$ as follows,
$$f_{_{\delta,\mathcal{C}}}(\bar\vb) \triangleq -\gamma w - \frac{\mu^2\tau_2}{2\alpha}\mathcal{C}^2-\frac{\alpha\beta^2}{2\delta\tau_2} - \frac{\alpha\tau_2}{2}+ \frac{\beta\tau_1}{2} + \eta \mu- \frac{\eta^2\alpha}{2\tau_2\mathcal{C}^2},
$$
where $\mathcal{C}=\widetilde\zeta$ and $\zeta$ (defined in Assumption \ref{ass:2}) for data models \eqref{eq:G_mix} and \eqref{eq:binarymodel}, respectively.
We introduce the following min-max objective based on the eight scalars,
\bea
&\min_{\substack{{\alpha,\tau_1,w\in\R_+,}\\ {\mu\in\R}}}\;\;\max_{\substack{{\tau_2,\beta,\gamma\in\R_+,}\\ {\eta\in\R}}} \,   
f_{_{\delta,\mathcal{C}}}(\bar\vb) +\E_{Z_{\al,\mu}}\left[\env{\Lm}{Z_{\al,\mu}-w}{\frac{\tau_1}{\beta}}\right]\nn\\
&\hspace{1.1in}+\eps_{\rm tr}\gamma\,\E_{L,H,T}\left[\env{\ell_1+\frac{r}{\eps_{\rm tr}\gamma}\ell_2^2}{\frac{\alpha\beta}{\tau_2\sqrt{\delta L}}H + \frac{\alpha\eta}{\tau_2\widetilde{\mathcal{C}}}\;T}{\frac{\alpha\eps_{\rm tr}\gamma}{\tau_2 L}} \right],\label{eq:minmax_main}
\eea
where $\widetilde{\mathcal{C}} \triangleq\widetilde\zeta^2 L$ and $\zeta^2$ for models \eqref{eq:G_mix} and \eqref{eq:binarymodel}, respectively, $H\sim\mathcal{N}(0,1)$ and $(T, L,V)\sim\Pi$ where $\Pi$ was defined in Assumption \ref{ass:3}. We also define:
\begin{align}\label{eq:datamodels}
Z_{\al.\mu}\triangleq\begin{cases} \sqrt{\al^2+\mu^2\widetilde{\zeta}^2}\,G+\mu\widetilde{\zeta}^2 \;\;\;{\text {for data model }\eqref{eq:G_mix}},\\[4pt] \alpha G+\mu \zeta S\cdot \psi(\zeta S) \;\;\;{\text {for data model }\eqref{eq:binarymodel}},\end{cases}
\;G,S\simiid\mathcal{N}(0,1).
\end{align}
Notice that the objective function of \eqref{eq:minmax_main} depends explicitly on the sampling ratio $\delta$ and on the training parameter $\eps_{\rm tr}$. Moreover, it depends implicitly on $\thc^\star_n$ and $\mathbf{\Lambda}_n$ via $T$ and $L$, respectively, and on the specific loss $\Lm$ via its Moreau envelope. 


We are now ready to state our main result in Theorem \ref{thm:main}, which establishes a relation between the solutions of \eqref{eq:minmax_main} and the adversarial risk of the robust classifier $\widehat{\thc_n}$.
\begin{theorem}\label{thm:main}
Assume that the training dataset $\{(\x_i,y_i)\}_{i=1}^{m}$, is generated according to either \eqref{eq:G_mix} or \eqref{eq:binarymodel} with diagonal covariance matrices satisfying Assumptions \ref{ass:1}-\ref{ass:3}. Consider the sequence of robust classifiers $\{\widehat{\thc_n}\}$, obtained by adversarial training in \eqref{eq:main_erm_Linfty} with a convex decreasing loss function $\Lm:\R\rightarrow\R$. 
 Then, the high-dimensional limit for the adversarial test error $(\mathcal{E}_{\ell_\infty,\frac{\eps_{\rm ts}}{\sqrt{n}}})$ is derived as follows,
\bea\label{eq:thm2.1_gen_error}
\left\{\mathcal{E}_{\ell_\infty, \frac{\eps_{\rm ts}}{\sqrt{n}}}\left(\widehat{\thc_n}\right)\right\} \rP  \begin{cases} Q \Big(\frac{\mu^\star\widetilde{\zeta}^2-w^\star\,\eps_{\rm ts}/{\eps_{\rm tr}} }{\sqrt{{\mu^\star}^2\widetilde{\zeta}^2+{\alpha^\star}^2}}\Big)  ~~{\text{for data model}~\eqref{eq:G_mix}}, \\[7pt]\Pro \Big( \mu^\star \zeta\,S\,\psi(\zeta S) + \alpha^\star G < w^\star {\eps_{\rm ts}}/{\eps_{\rm tr}} \Big) ~~{\text{for data model }~\eqref{eq:binarymodel}},\end{cases}
\eea
where $Q(\cdot)$ denotes the Gaussian $Q$-function and $(\alpha^\star,\mu^\star,w^\star)$ is the unique solution to the scalar minimax problem \eqref{eq:minmax_main}.
\end{theorem}
The asymptotics for adversarial error in Theorem \ref{thm:main} are precise in the sense that they hold with probability 1, as $m,n\rightarrow\infty.$ In the following section, we demonstrate these sharp predicted results for different values of problem parameters, in order to assess the impact of each parameter on the adversarial/standard errors. 
\subsection{Numerical Illustrations}\label{sec:num}




In this section, we illustrate the theoretical predictions for various values of the different problem parameters, including $\delta=m/n$ and the attack budgets $\eps_{\rm tr}$ and $\eps_{\rm ts}$. For numerical results here, we focus on the Hinge-loss i.e., $\Lm(t)= \max{(1-t,0)}$ and on the GMM with isotropic features, thus $L$ has a unit mass at 1. Additional experiments on GLM are given in Appendix \ref{sec:add_num_exp}. We further assume that $T$ is standard normal and fix regularization parameter $\la=10^{-4}$. To solve \eqref{eq:minmax_main}, we find the fixed-point solution of the corresponding saddle-point equations (derived in \eqref{eq:eta_main} in Appendix \ref{sec:sys_eqs_iso}) by iterating over these equations. 
 For the numerical results, we set $n=200$ and solve the ERM problem \eqref{eq:main_erm} by gradient descent. The resulting estimator is used to derive the adversarial test error by evaluating \eqref{eq:gen_error_binarymodel} on a test set of $3\times 10^3$ samples. We then average the results over 20 independent experiments. The results for both numerical and theoretical values are depicted in Figures \ref{fig:fig1}-\ref{fig:fig2}. Next, we discuss some of the insights obtained from these figures. 
\paragraph{Impact of $\bm{\delta}$ on standard/adversarial test error.} Figure \ref{fig:fig1} shows the adversarial and standard errors as a function of $\delta=m/n$. Note that both errors decrease as the sampling ratio $\delta$ grows, with the adversarial error approaching the Bayes adversarial error of the corresponding value of $\eps_{\rm ts}$. In Appendix \ref{sec:large_sample_size}, we formally prove that for $\ell_2$ attacks bounded by $\eps_{\rm ts}\in[0,1]$, the robust error of estimators obtained from adversarial training with any $\eps_{\rm tr}\in[0,1]$ converges to the Bayes adversarial error in the infinite sample-size limit i.e., when $\delta\rightarrow\infty$. Next, we highlight another observation regarding the role of data-set size. The second sharp decrease in standard and adversarial test errors appears right after the interpolation threshold $\delta_{\eps_{\rm tr}/\sqrt{n},\Pi}$, which denotes the maximum value of $\delta$ for which the data-points are $(\ell_\infty,\eps_{\rm tr}/\sqrt{n})$-separable (for definition, see the discussion on Robust Separability in Section \ref{sec:dis}). Such constantly decreasing behavior of error is in contrast to the corresponding behavior in linear regression with $\ell_2$ perturbations and $\ell_2$ loss as in \cite{javanmard2020reg} where a double-descent behavior was observed. This double-descent behavior can be considered as extensions of the double-descent behavior in standard ERM (first observed in numerous high-dimensional machine learning models \cite{belkin2018reconciling,belkin2019two,hastie2019surprises}), to the adversarial training case.  
 However, Figure \ref{fig:fig1} signifies that in binary robust classification by using decreasing losses such as the hinge-loss, the double-descent behavior does not appear for any value of $\eps_{\rm tr}$ and $\eps_{\rm ts}$. Additionally, in light of Figure \ref{fig:fig1}, one can measure for all $\delta>0$, the sub-optimality gap of standard/adversarial errors compared to the Bayes error. Such results, signify the critical role of training data-set size on obtaining robust and accurate estimators. Relevant results were obtained in \cite{schmidt2018adversarially}, where the authors derive bounds on the standard/adversarial error of a simple averaging estimator. However, our analysis is precise and holds for the broader case of convex decreasing losses. Finally, we highlight an important observation from Figure \ref{fig:fig1} (right): For highly over-parametrized models (very small $\delta$), standard accuracy remains the same for different choice of $\eps_{\rm tr}$. As $\delta$ grows, adversarial training (perhaps surprisingly) seems to improve the standard accuracy; however, for very large $\delta$, increasing $\eps_{\rm tr}$ hurts standard accuracy. It is also worth mentioning that similar results on the role of data-set size on standard accuracy was empirically observed in \cite{tsipras2018robustness} for neural network training of real-world data-sets such as MNIST. 
 \par
 \begin{figure*}[t!]
\centering
\begin{subfigure}{}
  \includegraphics[width=.41\linewidth,height= 5.1cm]{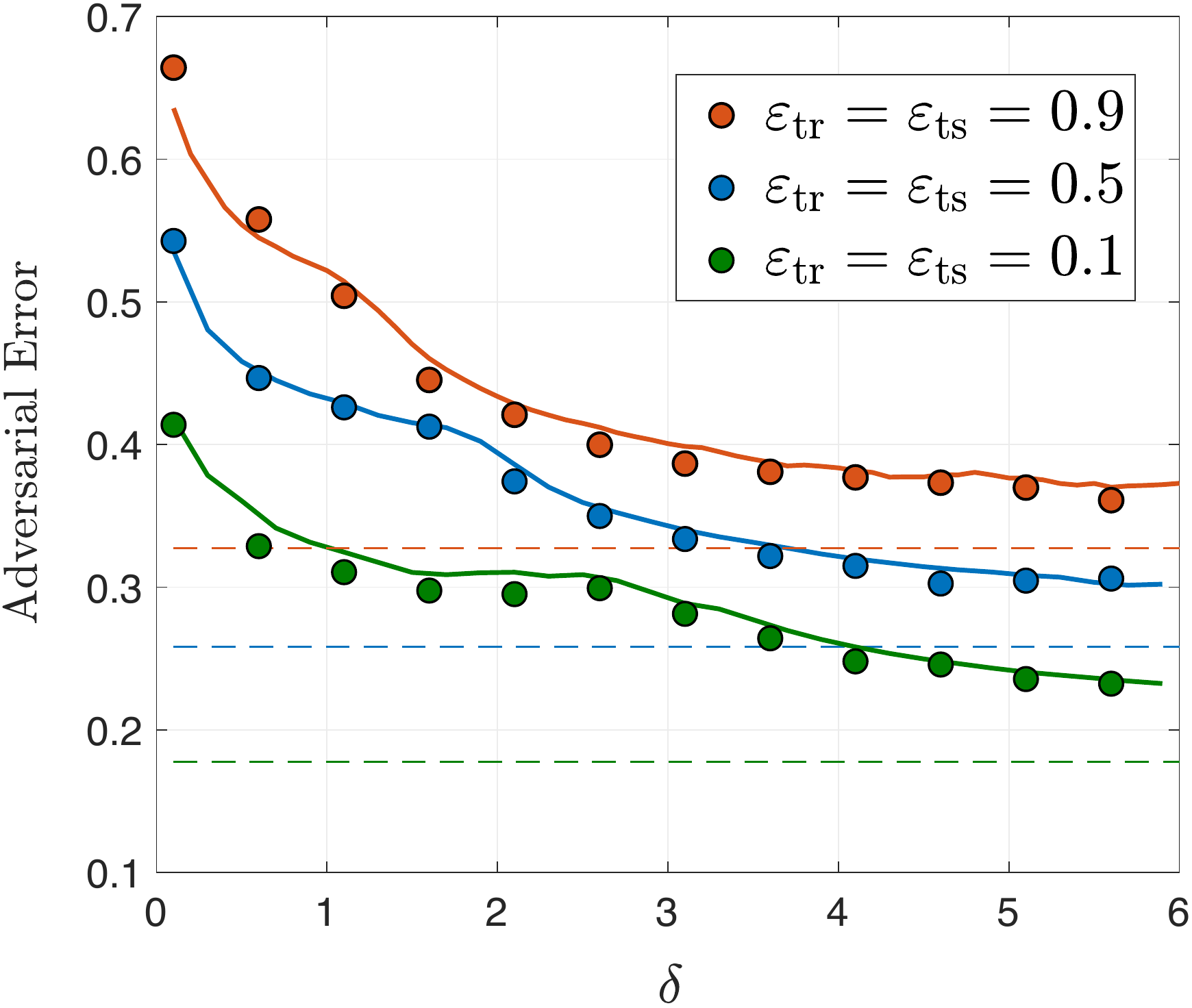}
  \label{fig:1}
\end{subfigure}%
\;\;\;\;\;\;
\begin{subfigure}{}
  \includegraphics[width=.41\linewidth,height=5.1cm]{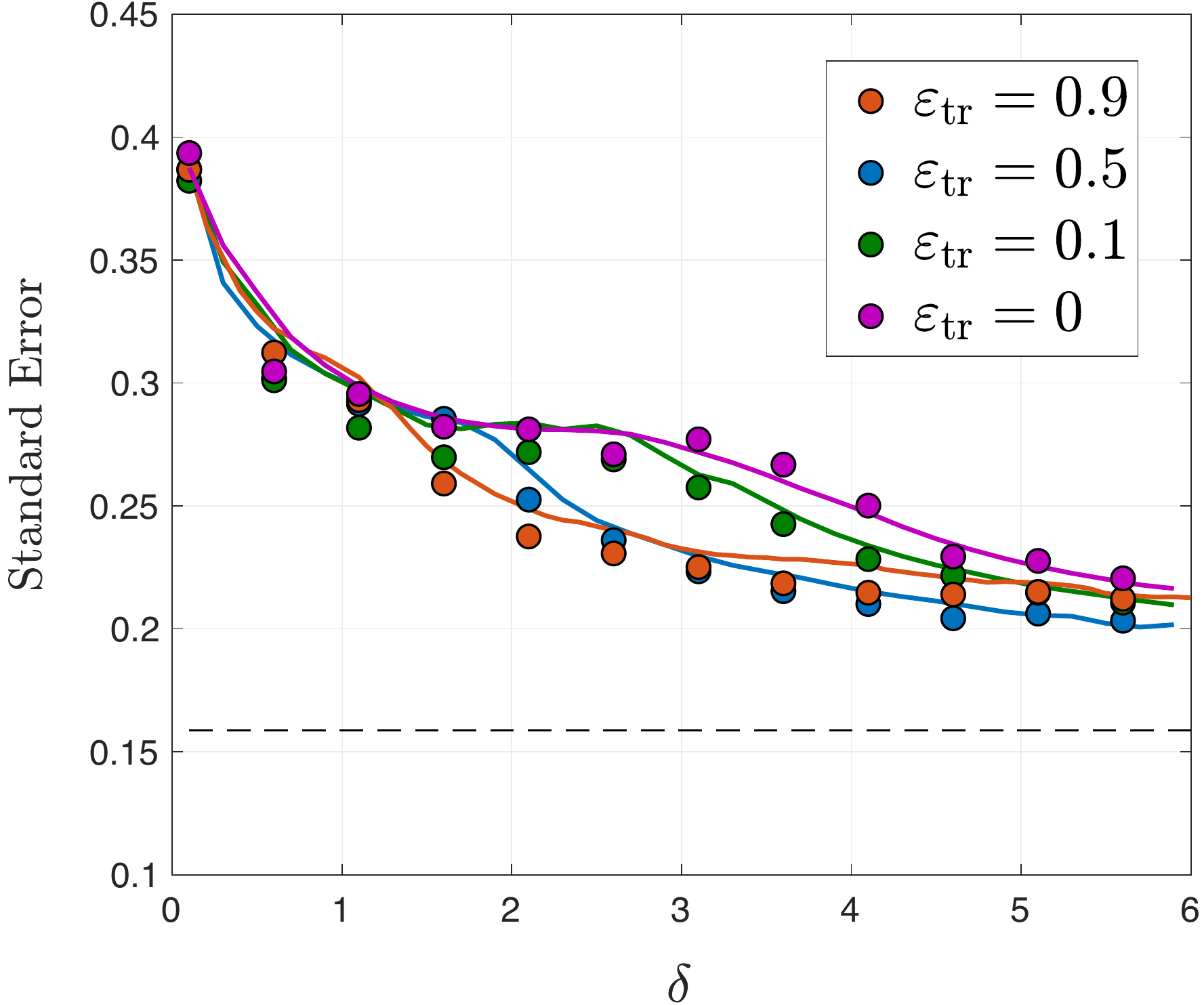}
  \label{fig:4}
\end{subfigure}%
\caption{Adversarial/Standard test error based on $\delta:=m/n$. Solid lines correspond to theoretical predictions while markers denote the empirical results derived by solving ERM using gradient descent. The dashed lines denote the Bayes adversarial error (left) and the Bayes standard error (right). Note that the adversarial error of estimators obtained from adversarial training, approaches the Bayes adversarial error as $\delta$ gets larger.}
\label{fig:fig1}
\end{figure*}

\begin{figure*}[t!]
\centering
\begin{subfigure}{}
  \includegraphics[width=.32\linewidth,height= 4cm]{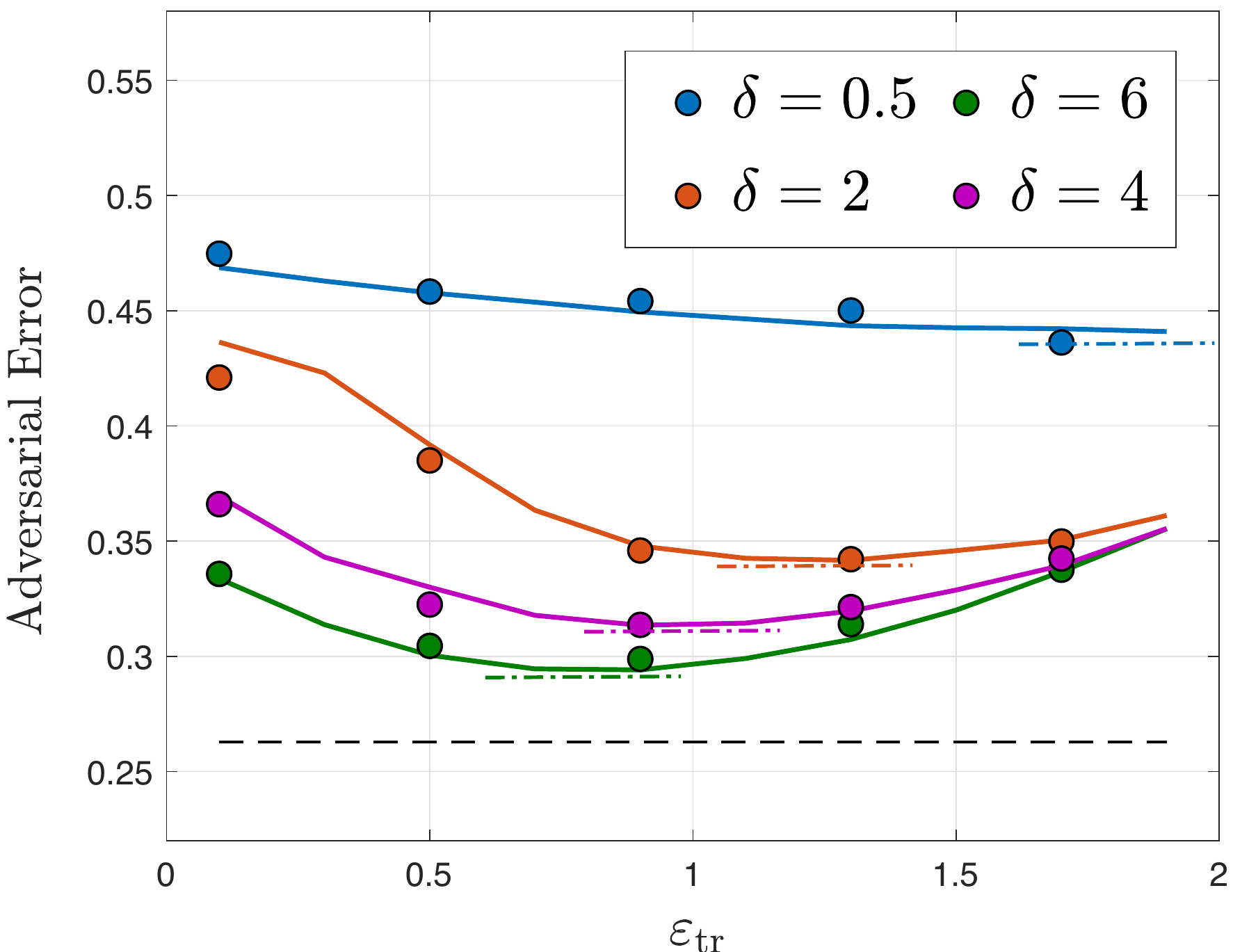}
  \label{fig:2}
\end{subfigure}%
\begin{subfigure}{}
  \includegraphics[width=.32\linewidth,height= 4cm]{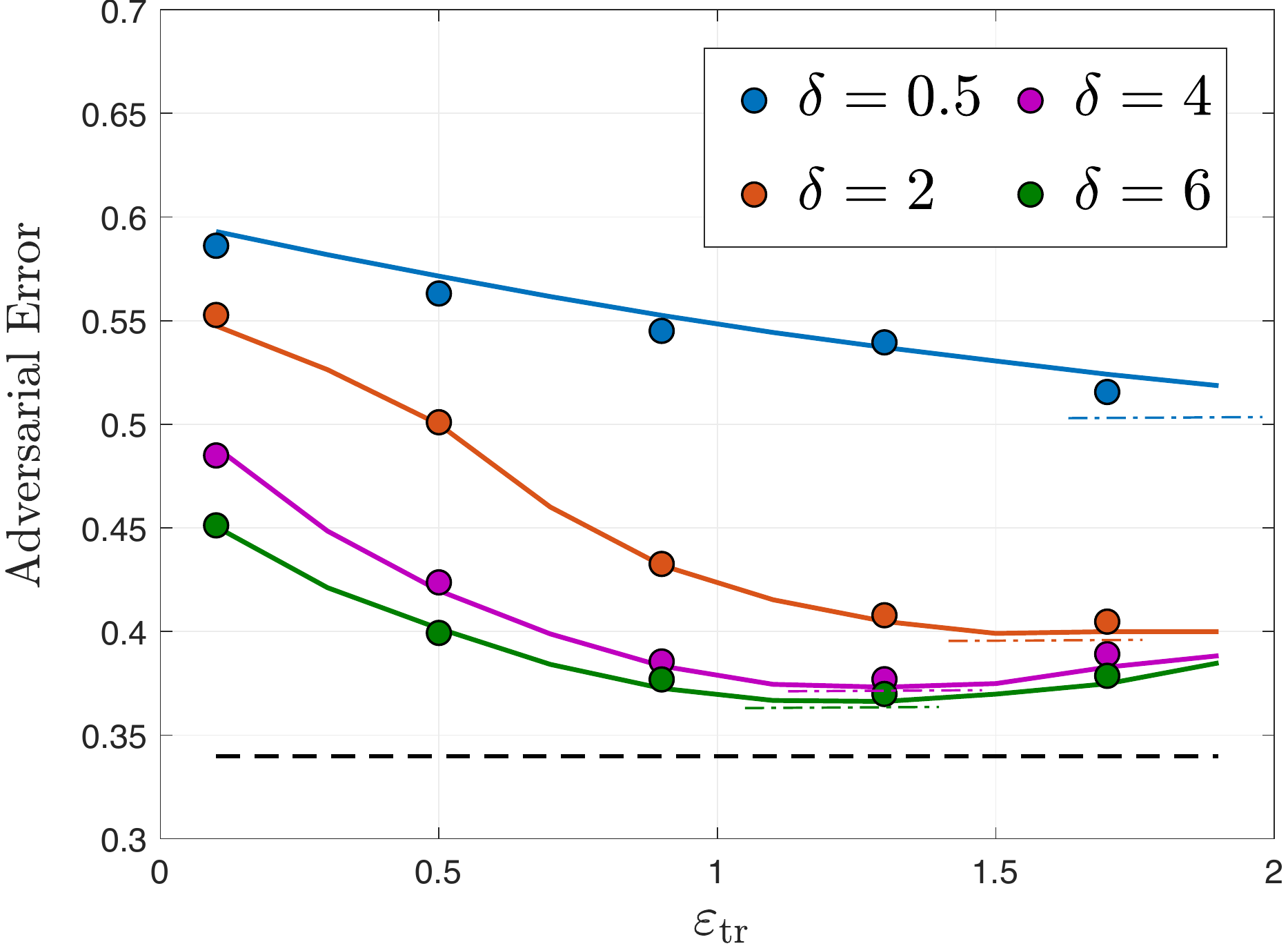}
  \label{fig:3}
\end{subfigure}%
\begin{subfigure}{}
  \includegraphics[width=.325\linewidth,height= 4cm]{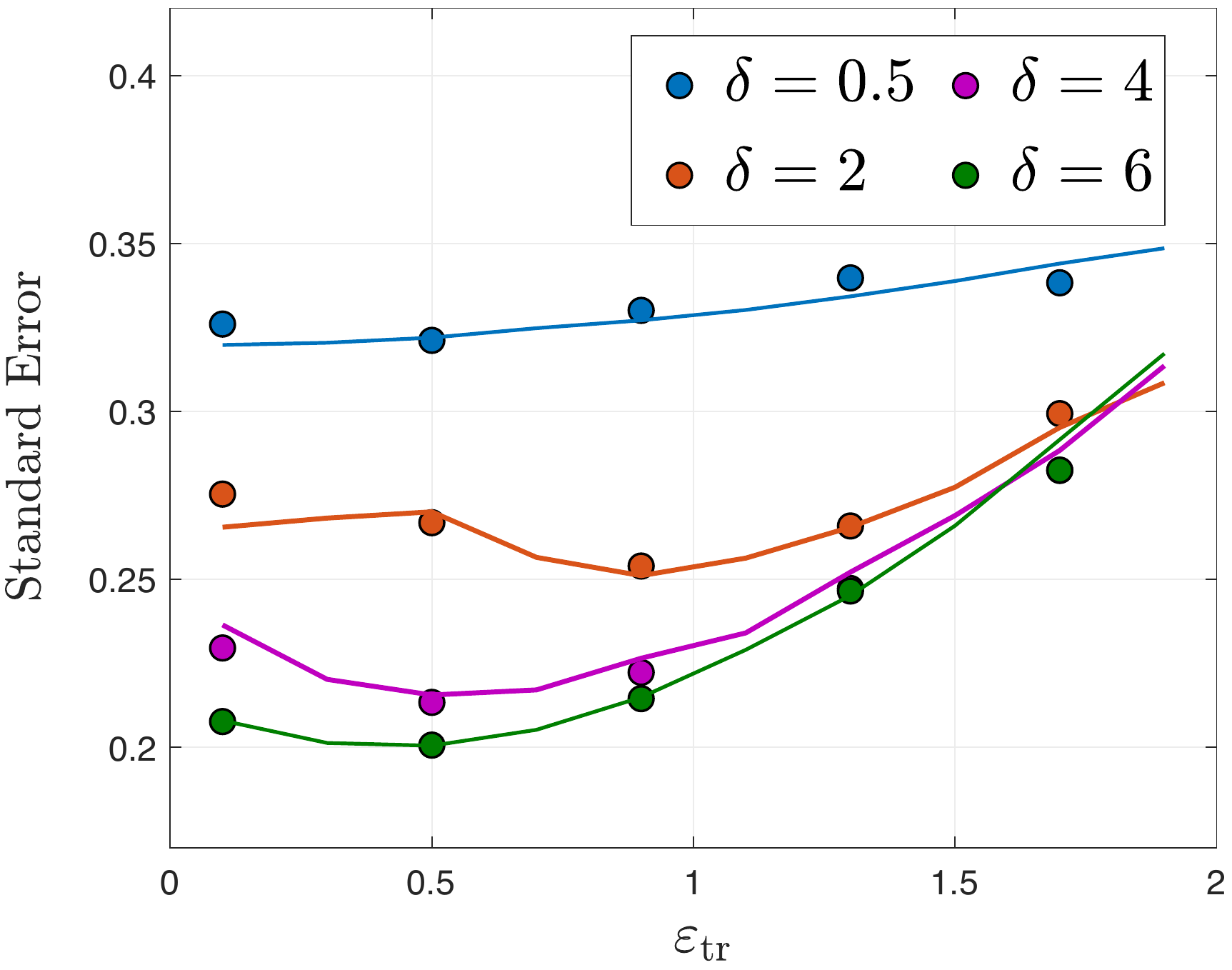}
  \label{fig:5}
\end{subfigure}%

\caption{Theoretical (solid lines) and Empirical (markers) results for the impact of adversarial training on the adversarial test error for $\eps_{\rm ts} = 0.5$ (Left) and $\eps_{\rm ts} = 0.9$ (Middle). The blacked dashed lines denote the Bayes adversarial error for the corresponding values of $\eps_{\rm ts}$. The colored dashed lines depict the optimal value of each curve. Note that the optimal value of $\eps_{\rm tr}$ decreases as $\delta$ grows. Right: Impact of adversarial training on the standard test error, illustrating that adversarial training can improve standard accuracy.}
\label{fig:fig2}
\end{figure*}
\paragraph{Impact of $\bm{\eps}_{\rm tr}$ on standard/adversarial test error.} Adversarial and Standard error curves based on the hyper-parameter $\eps_{\rm tr}$ are illustrated in Figure \ref{fig:fig2}. Note that the adversarial error behavior based on $\eps_{\rm tr}$ is informative on the role of data-set size on the optimal value of $\eps_{\rm tr}$. The top  figures show that the optimal value of $\eps_{\rm tr}$ is typically larger than $\eps_{\rm ts}$. Also note that as $\delta$ gets smaller, larger values of $\eps_{\rm tr}$ are preferred for robustness. Figure \ref{fig:fig2}(Right) illustrates the impact of $\eps_{\rm tr}$ on the standard error, where similar to Figure \ref{fig:fig1}(Right), we observe that adversarial training can help standard accuracy. In particular, we observe that in the under-parameterized regime where $\delta>\delta_{{\eps_{\rm tr}}/{\sqrt{n}},\Pi}$ (as we will define in Section \ref{sec:dis}), adversarial training with small values of $\eps_{\rm tr}$ is beneficial for accuracy. As $\delta$ increases, such gains vanish and indeed adversarial training seems to hurt standard accuracy when $\delta$ is sufficiently large. 
\subsection{Proof Sketch}\label{sec:sketch}
The complete proof of Theorem \ref{thm:main} is deferred to the appendix. Here, we provide an outline of the key steps in deriving \eqref{eq:minmax_main} and \eqref{eq:thm2.1_gen_error}. 
\paragraph{Reducing \eqref{eq:main_erm_Linfty} to a Minimization Problem.}For a decreasing loss function, the  maximization over the perturbation ${\db}$ can be derived in closed-form. In fact, it can be shown that 
$\db_i^\star \triangleq -\frac{\eps_{\rm tr}y_i\,\sign(\thc_n)}{\sqrt{n}},$ optimizes the inner maximization in \eqref{eq:main_erm_Linfty}. Therefore, \eqref{eq:main_erm_Linfty} is equivalent to,
\bea
\min_{\thc_n \in \R^n} \frac{1}{m}\sum_{i=1}^{m} \Lm \left(y_i \left\langle\x_i,\thc_n\right\rangle -\frac{\eps_{\rm tr}}{\sqrt{n}}\|\thc_n\|_1\right)\; {+\;\la\left\|\thc_n\right\|_2^2}. \label{eq:ermaftermax}
\eea
From \eqref{eq:ermaftermax}, we can also see why the normalization of $\eps_{\rm tr}$ is needed in \eqref{eq:main_erm_Linfty}. Recall that (for model \eqref{eq:binarymodel}, for instance), $\x_i\sim\Nn(\mathbf{0},\Sn)$ and $\|\thc_n^\star\|_2\rP1$. For simplicity assume here that $\Sn=\mathbb{I}_n$. For fixed $\thc$, the argument $y_i \langle\x_i,\thc\rangle$ behaves as $\|\thc\|_2 S f(S)$, where $S\sim\Nn(0,1)$. Thus, for $\thc$s that are such that $\|\thc\|_2=\Theta(1)$ (which ought to be the case for ``good" classifiers in view of $\|\thc_n^\star\|_2\approx1$), the term  $y_i \langle\x_i,\thc\rangle$ is an $\Theta(1)$-term. Now, thanks to the normalization $1/\sqrt{n}$ in \eqref{eq:main_erm_Linfty}, the second term $\frac{\eps_{\rm tr}}{\sqrt{n}}\|\thc\|_1$ in \eqref{eq:ermaftermax} is also of the same order. Here, we used again the intuition that  $\|\thc\|_1=\Theta(\sqrt{n})$, as is the case for the true $\thc^\star$. Our analysis formalizes these heuristic explanations. Finally, we remark that there is nothing specific about $q=\infty$ in the reduction \eqref{eq:ermaftermax}. The same reduction holds for any $q\geq 1$, with $\|\thc\|_1$ in \eqref{eq:ermaftermax} substituted by the dual norm $\|\thc\|_{p}$.

\paragraph{The Key Statistics for the Adversarial Error.} Our key observation is that the asymptotics of the adversarial error  of a sequence of arbitrary classifiers $\{\thc_n\}$, depend on the asymptotics of a few key statistics of $\{\thc_n\}$. Specifically, we show the following important lemma. Similar to before, there is nothing special here to $q=\infty$, so we state this result for general $q$. 
\begin{lemma}\label{lem:gen_error0}
Fix $q\geq 1$ and let $\ell_p$ denote the dual norm of $\ell_q$. Let $\widetilde{\thc_n^\star}\triangleq\Sn^{1/2}\thc_n^\star$  for data model \eqref{eq:binarymodel} and $\widetilde{\thc_n^\star}\triangleq\Sn^{-1/2}\thc_n^\star$ for data model \eqref{eq:G_mix}. Further, for both models, define projection matrices $\Th$ and $\Th^\perp$ as follows,
$ \Th\triangleq \widetilde{\thc_n^\star}\widetilde{\thc_n^\star}^\top / \|\widetilde{\thc_n^\star}\|_2,\;\Thp\triangleq \mathbb{I}_{n} - \Th.$
 Further, let $\eps$ and $\eps'$ (possibly scaling with the problem dimensions) be the upper-bounds on norm of the adversarial perturbation during training and test time, respectively. 
With this notation, assume that the sequence of $\{\thc_n\}$ is such that the following limits are true for the statistics $\left\|\thc_n\right\|_p$ , $\|\Th\Sn^{1/2}\thc_n\|_2$ and $\|\Th^\perp\Sn^{1/2}\thc_n\|_2$,
$$
\{\left\|\thc_n\right\|_p\}\stackrel{P}{\rightarrow} u, ~~~~\{\|\Th\Sn^{1/2}\thc_n\|_2\} \stackrel{P}{\rightarrow} \mu,~~~~ \{\|\Th^\perp\Sn^{1/2}\thc_n\|_2\} \stackrel{P}{\rightarrow} \alpha,
$$
where $\mathcal{C}=\widetilde\zeta,\zeta$, for GMM and GLM, respectively. Then, the adversarial test error satisfies,
\bea\label{eq:gen_err_adv0}\hspace{-.1in}
\left\{\mathcal{E}_{\ell_q, {\eps'}} (\thc_n)\right\}\rP\begin{cases} Q \left(\frac{\mu\widetilde\zeta^2-u\,\eps' }{\sqrt{{\mu}^2\widetilde\zeta^2+{\alpha}^2}}\right) ~~{\text {for  model }\eqref{eq:G_mix}},\\[6pt]\Pro \Big( \mu \zeta\,S\,  \psi(\zeta S) + \alpha G - u {\eps'} < 0 \Big) ~~ {\text {for model }\eqref{eq:binarymodel}}.\end{cases}
\eea
\end{lemma}
The detailed proof of the lemma is deferred to the appendix. There are essentially two steps in establishing the result. The first is to exploit the decreasing nature of the 0/1-loss to explicitly optimize over $\db_i$. This optimization results in the dual norm $\|\thc_n\|_{p}$. The second step is to consider the change of variables $\thc_n\Rightarrow\widetilde{\thc_n}\triangleq\Sn^{1/2}\thc_n$ and decompose $\widetilde{\thc_n}$ on its projection on $\Sn^{1/2}\thc_n^\star$ and its complement. In the notation of the lemma, $\widetilde{\thc_n}=\Theta\widetilde{\thc_n}+\Theta^\perp\widetilde{\thc_n}$. The Gaussianity of the feature vectors together with orthogonality of the two components in the decomposition of $\thc_n$ explain the appearance of the Gaussian variables $S$ and $G$ in \eqref{eq:gen_err_adv0}. 
%
When applied to $\ell_\infty$-perturbations, Lemma \ref{lem:gen_error0} reduces the goal of computing asymptotics of the adversarial risk of $\widehat{\thc_n}$ to computing asymptotics of the corresponding statistics $\|\Sn^{-1/2}\widetilde{\thc_n}\|_1$, $\|\Th\widetilde{\thc_n}\|_2$, and $\|\Th^\perp\widetilde{\thc_n}\|_2$. 
\paragraph{Scalarizing the Objective Function.} The previous two steps set the stage for the core of the analysis, which we outline next. Thanks to step 1, we are now asked to analyze the statistical properties of a convex optimization problem. On top of that, due to step 2, the outcomes of the analysis ought to be asymptotic predictions for the quantities $\|\Sn^{-1/2}\widetilde{\thc_n}\|_1$, $\|\Th\widetilde{\thc_n}\|_2$ and $\|\Th^\perp\widetilde{\thc_n}\|_2$. However, note that the term $\|\Sn^{-1/2}\widetilde{\thc_n}\|_1$ appears inside the loss function. In particular, this is a new challenge, specific to robust optimization compared to previous analysis of standard regularized ERM. Moreover, both of the terms $\|\Sn^{-1/2}\widetilde{\thc_n}\|_1$ and $\|\Sn^{-1/2}\widetilde{\thc_n}\|_2^2$ are not decomposable based on $\|\Th\widetilde{\thc_n}\|_2$ and $\|\Th^\perp\widetilde{\thc_n}\|_2$, due to the presence of the term $\Sn^{-1/2}$. The first step to overcome these challenges is to identify the appropriate minimax Auxiliary Optimization (AO) problem that is probabilistically equivalent to \eqref{eq:ermaftermax}. The second step is to scalarize the AO based on Lagrangian equivalent formulation. Finally, we perform a probabilistic analysis of the scalar AO. This results in the deterministic minimax problem in \eqref{eq:minmax_main}. See the appendix for details.
\section{Asymptotics for Adversarial Training with $\pmb{\ell_2}$ Perturbations}\label{sec:ell_2}
%
When $q=2$, the optimization problem in \eqref{eq:main_erm} is equivalent to the following, by choosing $\db_i=-y_i\eps_{\rm tr}\thc/\|\thc\|_2$,
\bea
\min_{\thc_n} \frac{1}{m}\sum_{i=1}^{m} \Lm \left(y_i \left\langle\x_i,\thc_n\right\rangle -\eps_{\rm tr}\|\thc_n\|_2\right)\; {+\;\la\left\|\thc_n\right\|_2^2}. \label{eq:ermaftermax-q2}
\eea
 Here, we assume $\{\Sn\}$ to be a sequence of positive definite matrices. Denote $\widetilde{\vb}\triangleq(\al,\tau_1,\tau_3,w,\mu,\tau_2,\beta,\gamma,\eta)$ and define $g:\R^9\rightarrow\R$ as follows,
$$g_{_{\delta,\mathcal{C},\eps_{\rm tr}}}(\widetilde{\vb}) \triangleq -\gamma w - \frac{\mu^2\tau_2}{2\alpha}\mathcal{C}^2-\frac{\alpha\beta^2}{2\delta\tau_2} - \frac{\alpha\tau_2}{2}+ \frac{\beta\tau_1}{2} + \eta \mu- \frac{\eta^2\alpha}{2\tau_2\mathcal{C}^2}+\frac{\eps_{\rm tr}\gamma\tau_3}{2},
$$
where recall that $\mathcal{C} \triangleq \widetilde\zeta$ and $\zeta$ for models \eqref{eq:G_mix} and \eqref{eq:binarymodel}, respectively. With this notation, we introduce the following min-max problem,
\bea
&\min_{\substack{{\alpha,\tau_1,\tau_3,w\in\R_+,}\\ {\mu\in\R}}}\;\;\max_{\substack{{\tau_2,\beta,\gamma\in\R_+,}\\ {\eta\in\R}}} \,   
\;\;\;g_{_{\delta,\mathcal{C},\eps_{\rm tr}}}(\widetilde\vb) \;\;+\;\;\E_{Z_{\al,\mu}}\left[\env{\Lm}{Z_{\al,\mu}-w}{\frac{\tau_1}{\beta}}\right]  \nn\\
&\hspace{.5in}+ \frac{\eta^2\al^2}{\tau_2^2\mathcal{C}^4}\left(\frac{\eps_{\rm tr}\gamma}{2\tau_3}+r\right){\E}_{L}\left[\frac{\frac{\mathcal{C}^4\beta^2}{\eta^2\delta}+ \widetilde L}{\frac{\eps_{\rm tr}\gamma\alpha+2\tau_3 r\alpha}{\tau_2\tau_3}+L} \right],\label{eq:minimax_final_ell2_main}
\eea
where we define $\widetilde L \triangleq 1/L$ and $L$ for models \eqref{eq:G_mix} and \eqref{eq:binarymodel}, respectively and the random variables $L$ and $Z_{\al,\mu}$ are defined same as in \eqref{eq:minmax_main}. 
\begin{theorem}\label{cor:GLM_ell2}
Consider the same setting as in Theorem \ref{thm:main}, only here assume that $q=2$ and $\{\Sn\}$ is a sequence of positive definite matrices satisfying Assumptions \ref{ass:1}-\ref{ass:3}. Let $(\alpha^\star,\mu^\star,w^\star)$ be the unique solution to the minimax problem \eqref{eq:minimax_final_ell2_main}.
 Then, the high-dimensional limit for the adversarial test error ($\mathcal{E}_{\ell_2,\eps_{\rm ts}}$) satisfies the following,
\bea\label{eq:thm_gen_error_ell2}
\left\{\mathcal{E}_{\ell_2, \eps_{\rm ts}}\left(\widehat{\thc_n}\right) \right\}\rP  \begin{cases} Q \Big(\frac{\mu^\star\widetilde{\zeta}^2-w^\star\,\eps_{\rm ts}/{\eps_{\rm tr}} }{\sqrt{{\mu^\star}^2\widetilde{\zeta}^2+{\alpha^\star}^2}}\Big)  ~~{\text{for data model}~\eqref{eq:G_mix}}, \\[7pt]\Pro \Big( \mu^\star \zeta S \psi(\zeta S) + \alpha^\star G < w^\star {\eps_{\rm ts}}/{\eps_{\rm tr}} \Big) ~~{\text{for data model }~\eqref{eq:binarymodel}}.\end{cases}
\eea
\end{theorem}
Compared to Theorem \ref{thm:main}, note here that the asymptotic prediction only depends on the total energy of $\thc^\star_n$(which was assumed to be 1 in Assumption \ref{ass:2}) and not its empirical distribution $T$. We present numerical illustrations on $\ell_2$-attacks in Appendix \ref{sec:add_num_exp}, where we also discuss how the data-set size and attack budgets, affect the adversarial and standard test errors. 
\section{Further Discussion on the Asymptotic Results}\label{sec:dis}

\begin{remark}[Training with no Regularization and \emph{Robust Separability}]
 An instance of special interest in practice is solving the \emph{unregularized} version of \eqref{eq:main_erm}.
\bea\label{eq:main_erm_Linfty_la0}
\min_{\thc_n } \;\frac{1}{m}\sum_{i=1}^m  \max_{\|\db_i\|_{q}\le\, \eps} \Lm \left(y_i \left\langle\x_i+\db_i,\thc_n\right\rangle\right).
\eea
Following the same proof techniques as above, we can show that the formulas predicting the statistical behavior of this unconstrained version are given by the same formulas as in Theorem \ref{thm:main} with $r=0$ and also provided that  the sampling ration $\delta$ is large enough so that a certain robust separability condition holds. 
In what follows, we describe this condition. 
We start with some background on (standard) data separability. Recall, that training data $\{(\x_i,y_i)\}$ are linearly separable if  and only if
$$
\exists\thc\in\R^n\quad\text{s.t.}\quad y_i\langle\x_i,\thc\rangle \geq 1,~\forall i\in[m].
$$
Now, we say that data are $(\ell_q,\eps)$-separable if and only if
$$
\exists\thc\in\R^n\quad\text{s.t.}\quad y_i\langle\x_i,\thc\rangle - \eps\|\thc\|_p \;\geq 1,~\forall i\in[m].
$$
Note that (standard) linear separability is equivalent to $(\ell_q,0)$-separability as defined above. Moreover, it is clear that $(\ell_q,\eps)$-separability implies $(\ell_q,0)$-separability for any $\eps\geq 0$. Recent works have shown that in the proportional limit data from the GLM are $(\ell_q,0)$-separable if and only if the sampling ratio satisfies $\delta<\delta_{\psi}$ \cite{candes2018phase, sur2019modern,montanari2019generalization,deng2019model} for some $\delta_{\psi}>2$.  Here, the subscript $\psi$ denotes dependence of the phase-transition threshold $\delta_{\psi}$ on the link function $\psi$ of the GLM. We conjecture that there is a threshold  $\delta_{\psi,\eps,\Pi}$, depending on $\eps$, the link function $\psi$ and the probability distribution $\Pi$ such that data are  $(\ell_q,\eps)$-separable if and only if  $\delta<\delta_{\psi,\eps,\Pi}$. We believe that our techniques can be used to prove this conjecture and determine $\delta_{\psi,\eps,\Pi}$, but we leave this interesting question to future work. Instead here, we simply note that based on the above discussion, if such a threshold exists, then it must satisfy $\delta_{\psi,\eps,\Pi}\leq \delta_{\psi,0,\Pi}$, for all values of $\eps$, and in fact it is a decreasing function of $\eps$.
Now let us see how this notion relates to solving \eqref{eq:main_erm_Linfty} and to our asymptotic characterization of its performance. Recall from \eqref{eq:ermaftermax} that the robust ERM for decreasing losses reduces to the minimization
$
\min_{\thc} \sum_{i=1}^{m} \Lm (y_i \langle\x_i,\thc\rangle - \eps\|\thc\|_p). 
$
Thus, using again the decreasing nature of the loss, it can be checked that the solution to the objective function above becomes unbounded for $\thc$ such that the argument of the loss is positive for any $i\in[m]$. This is equivalent to the condition of  $(\ell_q,\eps)$-separability. In other words, when data are $(\ell_q,\eps)$-separable, the robust estimator is unbounded. Recall from Section \ref{sec:sketch} that the minimax optimization variables $w,\mu,\alpha$ represent the limits of $\|\widehat{\thc_n}\|_p$, $\|\Th\Sn^{1/2}\widehat{\thc_n}\|_2$, and $\|\Th^\perp\Sn^{1/2}\widehat{\thc_n}\|_2$. Thus, if $\widehat{\thc_n}$ is unbounded, then $w^\star,\mu^\star,\alpha^\star$ are not well defined. In accordance with this, we conjecture that the minimax problem \eqref{eq:minmax_main} for $r=0$ (corresponding to \eqref{eq:main_erm_Linfty_la0}) has a solution if and only if the data are \emph{not} $(\ell_q,\eps)$-separable, equivalently, iff $\delta>\delta_{\psi,\eps,\Pi}$. Equivalent results are applicable to the Gaussian-Mixture models.
\end{remark} 

\begin{remark}[On {Statistical} Limits in Adversarial Training] The asymptotics in \eqref{eq:thm_gen_error_ell2} imply that for $\ell_2$ perturbations and isotropic features, the adversarial error depends on the ratio $\alpha^\star/\mu^\star$. In fact, it can be seen that smaller values of the ratio lead to decreased adversarial error. This leads to an interesting conclusion: \emph{In order to find the hyper-parameter $\eps_{\rm tr}$ that  minimizes the adversarial error, it suffices to tune  $\eps_{\rm tr}$ to minimize the ratio $\alpha^\star/\mu^\star$}. A similar conclusion can be made for the case  of $\ell_\infty$ perturbations, by noting from \eqref{eq:thm2.1_gen_error} that the adversarial error is characterized in a closed form in terms of $(\alpha^\star,\mu^\star,w^\star)$. In view of these observations, our sharp guarantees for the performance of adversarial training open the way to answering questions on the statistical limits and optimality of adversarial training, e.g. \emph{how to optimally tune $\eps_{\rm tr}$? How to optimally choose the loss function and what is the best minimum values of adversarial error achieved by the family of robust estimators in \eqref{eq:main_erm}? How do these answers depend on the adversary budget and/or the sampling ratio $\delta$?} Fundamental questions of this nature have been recently addressed in the non-adversarial case  based on the corresponding saddle-point equations for standard ERM, e.g.,  \cite{bean2013optimal,celentano2019fundamental,mai2019large,mai2019high,taheri2020sharp,taheri2021fundamental}. Theorems \ref{thm:main} and \ref{cor:GLM_ell2} are the first steps towards such extensions to the adversarial settings. 
\end{remark}

\section{Conclusions and Future Directions}
In this paper, we studied the generalization behavior of adversarial training in a binary classification setting. Our results included the adversarial error and standard error of estimators obtained by $\ell_q$-norm bounded perturbations, in the high-dimensional setting. In particular, we derived precise theoretical predictions for the performance of adversarial training for the GLM and GMM. Numerical simulations validate theoretical predictions even for relatively small problem dimensions and demonstrate the role of sampling ratio $\delta$, data model and the values of $q$, $\eps_{\rm tr},\eps_{\rm ts}$ on adversarial robustness. Finally, we remark that there is nothing special about the choice of regularization (ridge-regularization is considered in this paper) and the current analysis can be easily extended to general convex regularization functions. There are a number of directions left for future work. Going beyond Gaussian data distributions or proving universality results are among the most interesting future directions. Our results do not perfectly demonstrate the recently emerged benign overfitting phenomenon, as we are considering linear models. In this regard, despite the fact that considering a mismatched linear model in our setting can be insightful(e.g., similar to \cite{hastie2019surprises,montanari2019generalization,deng2019model}), it is preferable to study the asymptotic behavior of adversarial training for other models such as Neural Tangent or Random Features. Such extensions are considered challenging, since the inner maximization of ERM does not take a closed-form solution. However, we believe that by approximating the inner maximization, e.g., by using a linear approximation of loss as used in FGSM, our analysis is applicable for such extensions. One other natural question is considering attacks other than $\ell_q$-norm attacks considered in the present paper.


\bibliographystyle{alpha}

\clearpage
\bibliography{main_adv}
\appendix

%
%




%

%

\clearpage
\section*{Appendix}
  \hypersetup{linkcolor=black}

\tableofcontents
  \hypersetup{linkcolor=darkred}

%
\addtocontents{toc}{\protect\setcounter{tocdepth}{3}}
\section{Additional Numerical Experiments}\label{sec:add_num_exp}
\subsection{Experiments on $\ell_2$ perturbations and GLM}
In this section, we complement the numerical illustrations of Section \ref{sec:num}, by considering the case of Signed measurements as well as extending to the $\ell_2$-perturbations case. We focus on the Hinge-loss and for simulation results we set $n=200,\la=0,\Sn=\mathbb{I}_n$ and average the results over 20 experiments. Figures \ref{fig:fig3}(Top) depict the adversarial/standard errors for the signed measurements. Notably, based on Figure \ref{fig:fig3}(Top left), one can observe that adversarial attacks are successful in GLM, as for a fixed $\delta$, adversarial training does not seem to improve noticeably the adversarial error (the error bars are obtained by 10 experiments). However, note the critical role of data-set size on both standard and adversarial errors as depicted in Figure \ref{fig:fig3}(Top right). Similar to the GMM, here we also observe that both adversarial and standard errors are decreasing based on $\delta$ in both cases of $q=2,\infty$. 
\begin{figure*}[h]
\centering
\begin{subfigure}{}
\centering
  \includegraphics[width=.42\linewidth,height= 5.3cm]{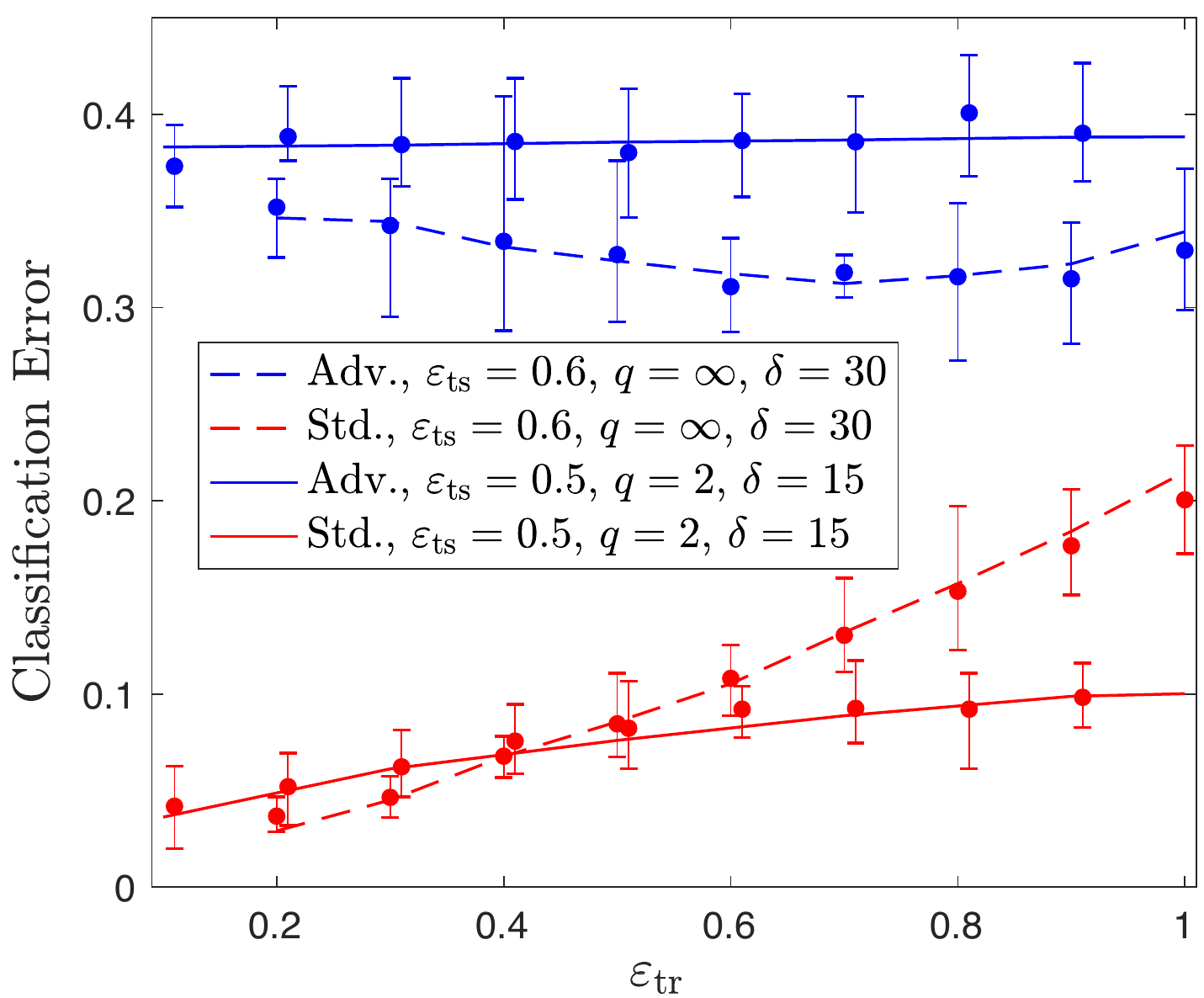}
  \label{fig:6}
\end{subfigure}%
\;\;
\begin{subfigure}{}
  \includegraphics[width=.42\linewidth,height= 5.7cm]{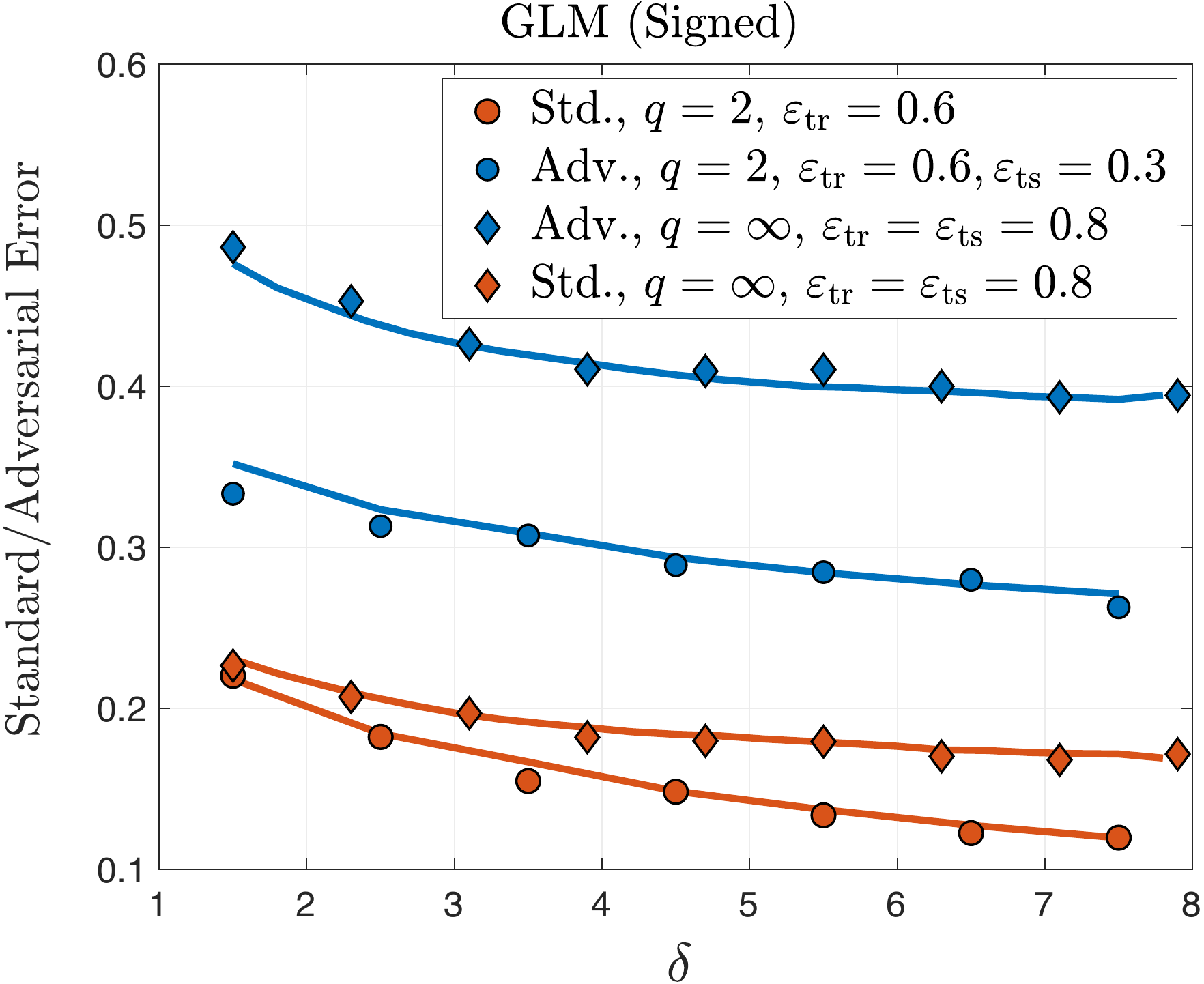}
  \label{fig:7}
\end{subfigure}%
\\
\begin{subfigure}{}
\centering
  \includegraphics[width=.43\linewidth,height= 6.2cm]{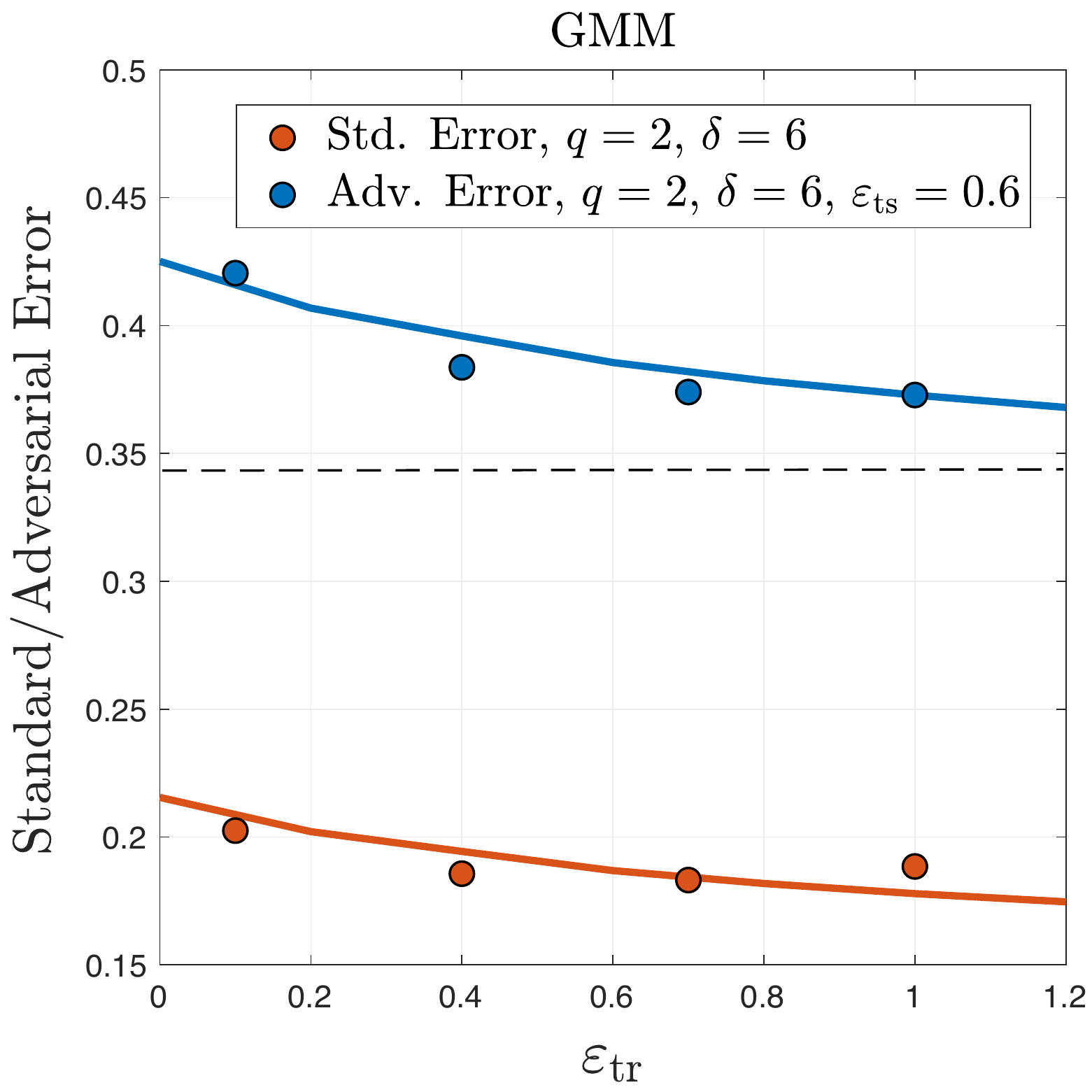}
  \label{fig:8}
\end{subfigure}%
\;
\begin{subfigure}{}
  \includegraphics[width=.43\linewidth,height= 6.2cm]{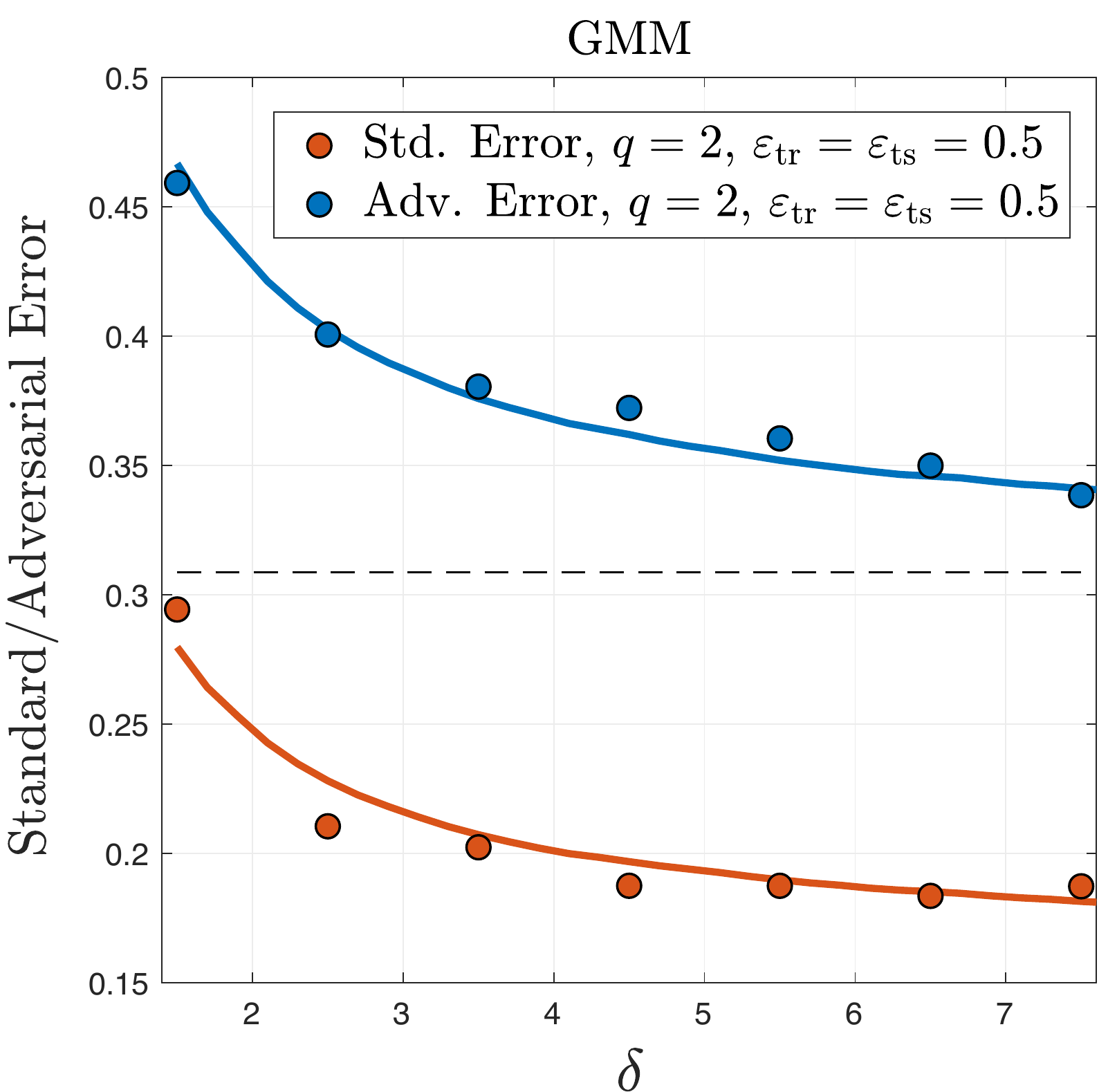}
  \label{fig:9}
\end{subfigure}%
\caption{Adversarial and Standard Errors for the Signed model (Top) and the Gaussian-mixture model (Bottom). The dashed lines denote the Bayes adversarial error for the corresponding values of $\eps_{\rm ts}$.}
\label{fig:fig3}
\end{figure*}
\begin{figure*}[h]
\centering
  \includegraphics[width=.5\linewidth,height= 6.2cm]{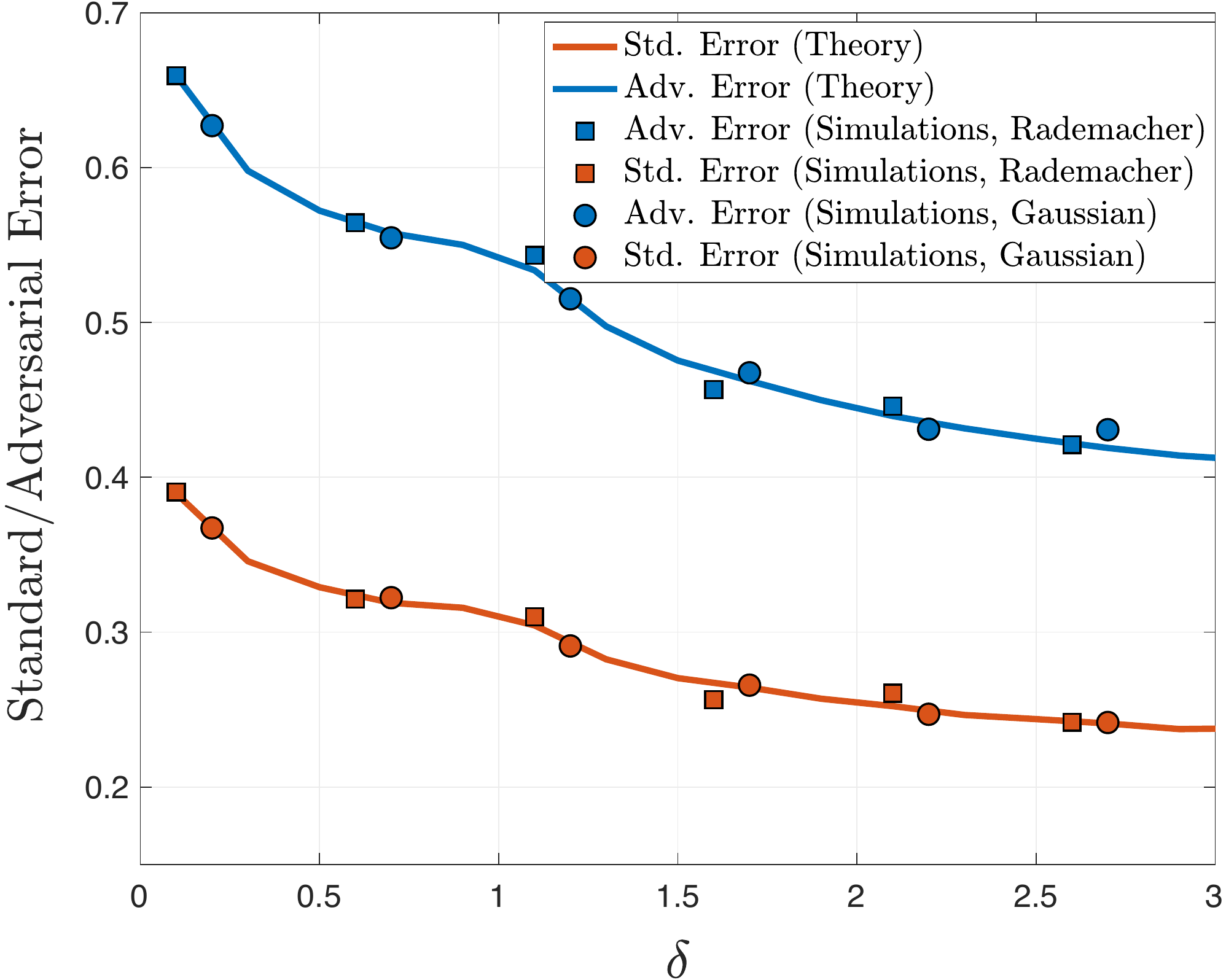}
  \caption{Empirical results for Rademacher(squares) and Gaussian(circles) data distributions in a generative data model alongside the theoretical curves. Here $q=\infty$ and $\eps_{\rm tr} = \eps_{\rm ts} =1$. The perfect match between theory and experiments supports the conjectured universality property in adversarial training.}
  \label{fig:fig4}
\end{figure*}%
Figures \ref{fig:fig3}(Bottom) depict the error curves for the GMM and $q=2$. Perhaps surprisingly, here we see that more aggressive adversarial training improves the standard error as the error curve is strictly decreasing with respect to $\eps_{\rm tr}.$ We also highlight that unlike the $q=\infty$ case where there was a finite optimal choice of $\eps_{\rm tr}$, here increasing $\eps_{\rm tr}$, always helps the robust accuracy. Note also the role of $\delta$ on error curves, especially by increasing $\delta$, both errors decrease and notably the adversarial error approaches the Bayes optimal error. For a formal proof of this phenomenon, see also the discussion in Appendix \ref{sec:large_sample_size}.
 
\subsection{Universality in Adversarial Robustness}
Thus far, we focused on Gaussian data. One may wonder whether our theoretical results extend to other data distributions. We conjecture that our results enjoy the \emph{universality} property, i.e., the same asymptotic formulas in Theorem \ref{thm:main} and Theorem \ref{cor:GLM_ell2}, hold when data is sampled from a \emph{sub-Gaussian} distribution. Figure \ref{fig:fig4} illustrates the empirical results for the adversarial and standard error of Gaussian-mixture model as well as a model obtained by the mixture of Rademacher distributions, i.e.,
$$ 
\x_i | y_i \;\sim\; y_i\thc_{n}^\star+ {{\rhoc}_i}, \;\;\;\;\Pro(y_i=1) = \pi\in[0,1],
$$
where each entry of $\rhoc_i\in\R^n$ is distributed iid from Rademacher distribution.
Note the perfect agreement between theory and simulation for both standard and adversarial errors, which supports the universality conjecture.
For standard ERM, the universality property has been studied in numerous recent works e.g., see \cite{bayati2015universality,oymak2018universality,abbasi2019universality}. Extending such results to the adversarial training case is left for future work. 
\section{Asymptotic Analysis of Adversarial Training for $\pmb{{\ell}_\infty}$ Perturbations}\label{sec:proof_thm_1}
In this section, we provide an asymptotic analysis for adversarial training with $\ell_\infty$ perturbations. First, we consider the case of general $\Sn$ and then show how our theoretical results simplify when $\Sn$ is diagonal and when $\Sn=\mathbb{I}_n$. We focus on Generalized linear models \eqref{eq:binarymodel}. The corresponding analysis for Gaussian-Mixture models \eqref{eq:G_mix} is deferred to Section \ref{sec:GMM_analysis}.

We begin with proving the key statistics required for the high-dimensional asymptotics. 
\subsection{Adversarial Error of an Arbitrary Estimator}
In the following lemma, we characterize the asymptotic adversarial error under $\ell_q, q\geq 1$ perturbations of an arbitrary sequence of estimators $\{\thc_{n}\}_{n=1}^{\infty}$ (where $\thc_n\in\R^n$), in terms of the high-dimensional limits for the key statistics $\|\thc_{n}\|_p, \|\Th\Sn^{1/2}\thc_{n}\|_2$ and $\|\Thp\Sn^{1/2}\thc_{n}\|_2$, where $p$ is such that $1/p+1/q=1$. We assume that the adversary has budget $\eps$.


First, we formalize the adversarial test error in the next lemma, which is a restatement of Lemma \ref{lem:gen_error0} specialized to GLM.

\begin{lemma}\label{lem:gen_error}
The high-dimensional limit of the adversarial test error for the Generalized Linear models with a given sequence of classifiers $\{\thc_{n}\}$ is given as follows,
\bea\label{eq:gen_err_adv}
\{\mathcal{E}_{\ell_q, {\eps}}^{\text{GLM}}(\thc_n)\} \rP \Pro \Big( \mu \zeta\, S \cdot\psi(\zeta S) + \alpha G - u {\eps}  < 0 \Big)
\eea
where $G,S \simiid \mathcal{N}(0,1)$ and provided that 
$$\left\|\Sn^{-1/2}\widetilde{\thc_{n}}\right\|_p\rP u,\;\;\left\langle\widetilde{\thc_n^\star},\widetilde{\thc_{n}}\right\rangle\Big/\left\|\widetilde{\thc_n^\star}\right\|_2^2\rP \mu ,\;\;\left\|\Th^\perp\widetilde{\thc_{n}}\right\|_2\rP \alpha,$$
 for $\ell_p$-norm denoting the dual of the $\ell_q$-norm, $\widetilde{{\thc}_n}\triangleq \Sn^{1/2}\thc_n,\widetilde{\thc_{n}^\star} \triangleq \Sn^{1/2}\thc_{n}^\star$ and $\Thp \in \R^{n\times n}$ defined as follows:
 $$ \Thp\triangleq \mathbb{I}_{n} - \Th, \;\;\Th\triangleq\frac{ \widetilde{\thc_{n}^\star}\widetilde{\thc_{n}^\star}^{\top}}{\left\|\widetilde{\thc_n^\star}\right\|_2^2}.$$
Moreover, in the special case of $q=2$ and $\Sn=\mathbb{I}_n$, by denoting $\sigma\triangleq\alpha/\mu$, \eqref{eq:gen_err_adv} simplifies to, 
\bea\label{eq:gen_err_adv_2}
\{\mathcal{E}_{\ell_2, \eps}^{\text{GLM}}(\thc_n)\} \rP \Pro \left( \frac{S\psi(S) + \sigma G}{\sqrt{\sigma^2+1}}  < \eps \right)  .
\eea
\end{lemma}
\begin{proof}
First, for $\thc_{n}\neq 0$ note the following chain of equalities:
\begin{align*}
\max_{\|\db\|_q\leq \eps} \one_{\{y\neq \sign\left\langle\x+\db,\thc_{n}\right\rangle\}} 
&=\max_{\|\db\|_q\leq \eps} \one_{\{y\left\langle\x,\thc_{n}\right\rangle + y \left\langle\db,\thc_{n}\right\rangle <0 \}} \\
&=\one_{\{ y\left\langle\x,\thc_{n}\right\rangle + \min_{\|\db\|_q\leq \eps}y \left\langle\db,\thc_{n}\right\rangle <0 \}}\\
&=\one_{\{ y\left\langle\x,\thc_{n}\right\rangle - \eps \|\thc_{n}\|_p<0 \}},
\end{align*}
where in the last line we used the fact that $\ell_p$ is the dual norm of $\ell_q$ norm.  Thus, we can write
\bea
\E_{\x,y}\left[\max_{\|\db\|_q<\eps} \one_{\{y\neq \sign\left\langle\x+\delta,\thc_{n}\right\rangle\}}\right] &= \Pro \Big(y\neq\sign\left(\langle\x,\thc_{n}\rangle-y\eps\|\thc_{n}\|_p\right)\Big) \nn\\[3pt]
&=  \Pro \Big(y\langle\x,\thc_{n}\rangle - \eps\|\thc_{n}\|_p <0\Big)\nn \\[3pt]
&= \Pro \Big (y\left\langle\bar{\x},\Sn^{1/2}\thc_{n}\right\rangle - \eps\|\thc_{n}\|_p <0\Big)\nn\\[3pt]
&=  \Pro \Big(y\left \langle\bar{\x},\Th\Sn^{1/2}\thc_{n}\right\rangle + y\left\langle \bar\x, \Th^\perp\Sn^{1/2}\thc_{n}\right\rangle - \eps\|\thc_{n}\|_p <0\Big)\nn\\[3pt]
&= \Pro \Big(y \left\langle\bar{\x},\Th\widetilde{\thc_{n}}\right\rangle + y\left\langle \bar\x, \Th^\perp\widetilde{\thc_{n}}\right\rangle - \eps\|\Sn^{-1/2}\widetilde{\thc_{n}}\|_p <0\Big),
\label{eq:ge_limit}
\eea
Where $\bar{\x}$ is a standard Gaussian vector. Also, for the labels $y$ we have,
$$y=\psi\left(\left\langle\x,\thc_n^\star\right\rangle\right)=\psi(\langle\bar{\x},\widetilde{\thc_{n}^\star}\rangle) = \psi(\langle\bar{\x},\Th\widetilde{\thc_{n}^\star}\rangle).$$
Now, by Gaussianity of $\bar{\x}$ and since $\Th\Th^\perp = \mathbf{0}_n, \Th+\Th^\perp = \mathbb{I}_n$, we find that $\langle\bar{\x},\Th\widetilde{\thc_{n}}\rangle$ and $y$ are both independent of $\langle\bar\x,\Th^\perp\widetilde{\thc_{n}}\rangle$. Therefore, we can replace $\langle\bar{\x},\Th^\perp\widetilde{\thc_{n}}\rangle$ by $\langle\bar{\bar{\x}},\Th^\perp\widetilde{\thc_{n}}\rangle$ for some standard Gaussian vector $\bar{\bar\x}$ \emph{independent} of $\bar\x$. Then, by rotational invariance of $\bar{\bar\x}$ and since $y$ is independent of it and takes values $\pm1$,  $y\bar{\bar\x}^{\top}\Th^\perp\widetilde{\thc_{n}}$ is distributed as $\bar{\bar\x}^{\top}\Th^\perp\widetilde{\thc_{n}}$. 
But, again by rotational invariance of the gaussian distribution, we have that for $G,S\simiid\Nn(0,1),$ 
\bea
&\left\langle\bar{\x},\Theta\widetilde{\thc_{n}^\star}\right\rangle \sim \|\widetilde{\thc_n^\star}\|_2\,S,\nn\\[4pt]
 &\left\langle\bar{\bar\x},\Th^\perp\widetilde{\thc_{n}}\right\rangle\sim\|\Th^\perp\widetilde{\thc_{n}}\|_2 G,\nn\\[4pt]
 &\left\langle\bar{\x},\Theta\widetilde{\thc_{n}}\right\rangle=\frac{\left\langle\widetilde{\thc_{n}},\widetilde{\thc_{n}^\star}\right\rangle}{\|\widetilde{\thc_n^\star}\|_2^2}\left\langle\bar\x,\widetilde{\thc_{n}^\star}\right\rangle \sim  \frac{\left\langle\widetilde{\thc_{n}},\widetilde{\thc_{n}^\star}\right\rangle}{\|\widetilde{\thc_n^\star}\|_2^2} \|\widetilde{\thc_n^\star}\|_2\,S. \nn
 \eea

Next, recall that $ \left\|\widetilde{\thc_n^\star}\right\|_2 = {\thc_n^\star}^\top\Sn{\thc_n^\star} \rightarrow \zeta$ based on Assumption \ref{ass:2} and note the lemma's assumptions on convergence of $\|\Sn^{-1/2}\widetilde{\thc_{n}}\|_p, \;\langle\widetilde{\thc_{n}},\widetilde{\thc_{n}^\star}\rangle/\|\widetilde{\thc_n^\star}\|_2^2$ and $\|\Theta^\perp\widetilde{\thc_{n}}\|_2$. Combining with the above, we deduce that, 
\bea
y\rP \psi(\zeta\,S),\;\;\; \left\langle\bar{\x},\Th\widetilde{\thc_{n}}\right\rangle \rP \mu \zeta\, S, \;\;\;\left\langle\bar{\x},\Th^\perp\widetilde{\thc_{n}}\right\rangle \rP \alpha G.
\eea 

Putting this together with \eqref{eq:ge_limit} gives the limit in \eqref{eq:gen_err_adv} for GLM. To derive \eqref{eq:gen_err_adv_2}, note that when $q=2$ and $\Sn=\mathbb{I}_n$, it holds that $\zeta=1$ and $u=\sqrt{\alpha^2+\mu^2}$ due to
$$
\|\widetilde{\thc_{n}}\|_2 = \|\Th\widetilde{\thc_{n}} + \Th^\perp\widetilde{\thc_{n}}\|_2 =  \sqrt{\|\Th\widetilde{\thc_{n}}\|_2^2 + \|\Th^\perp\widetilde{\thc_{n}}\|_2^2 } \rP \sqrt{\alpha^2+\mu^2}. 
$$
This concludes the proof. 
\end{proof}

\subsection{Case I: Correlated Features with General Covariance Matrix}\label{sec:proof_GLM}
For all $\x\in\R^n,\tau,C\in\R_+$ and a PD matrix $\mathbf{S}\in\R^{n\times n}$, we define, 
\bea
\env{\left(\small{\ellb_1}+C\, \ellb_2^2,\,\mathbf{S}\right)}{\x}{\tau} \triangleq \min_{\y\in\R^n}\;\; \frac{1}{2\tau}\left\|\mathbf{S}^{1/2}\left(\x-\y\right)\right\|^2_2 \;+\; \|\y\|_1 \;+\; C\|\y\|_2^2.\label{eq:moreau_sigma}.
\eea
Assume that the PD covariance matrix $\Sn$ and the true vector $\thc^\star_n$ satisfy the following limit for all constants $c_1,c_2,c_3,c_4\in\R^+\times\R^+\times\R\times\R^+$, and the standard Gaussian vector $\h\in\R^n$,
\bea
\frac{1}{n}\env{\left(\small{\ellb_1}+c_1\ellb_2^2,\,\Sn\right)}{c_2\Sn^{-1/2}\h + c_3 {\sqrt{n}\thc_{n}^\star}}{c_4} \rP \bar{M}_{c_1}\left(c_2,c_3;c_4\right).
\eea
Following the same notation as in \eqref{eq:minmax_main}, we introduce the following min-max objective based on eight scalars,
\bea
&\min_{\substack{{\alpha,\tau_1,w\in\R_+,}\\ {\mu\in\R}}}\;\max_{\substack{{\tau_2,\beta,\gamma\in\R_+,}\\ {\eta\in\R}}} \,   
f_{_{\delta,\mathcal{C}}}(\bar\vb) \;+\;\;\E_{G,S}\left[\env{\Lm}{\al G + \mu \zeta\, S \,\psi(\zeta S)-w}{\frac{\tau_1}{\beta}}\right] \nn\\&\hspace{2in} +\eps_{\rm tr}\gamma\, \bar{M}_{\frac{r}{\gamma\eps_{\rm tr}}}\left(\frac{\alpha\beta}{\tau_2\sqrt{\delta}},\frac{\alpha\eta}{\tau_2\zeta^2};\frac{\alpha\gamma\eps_{\rm tr}}{\tau_2}\right).
\label{eq:minmax_app_thm5}
\eea
\begin{theorem}\label{thm:5}
Assume that the training dataset $\{(\x_i,y_i)\}_{i=1}^{m}$, is generated according to Generalized Linear models \eqref{eq:binarymodel} with PD covariance matrices satisfying Assumptions \ref{ass:1}-\ref{ass:3}. Consider the sequence of robust classifiers $\{\widehat{\thc_n}\}$, obtained by adversarial training in \eqref{eq:main_erm_Linfty} with a convex decreasing loss function $\Lm:\R\rightarrow\R$. Then, the high-dimensional limit for the adversarial test error $(\mathcal{E}_{\ell_\infty,\frac{\eps_{\rm ts}}{\sqrt{n}}})$ is derived as follows,
\bea\label{eq:thm5_gen_error}
\left\{\mathcal{E}_{\ell_\infty, \frac{\eps_{\rm ts}}{\sqrt{n}}}^{\tiny{GLM}}\left(\widehat{\thc_n}\right)\right\} \rP  \Pro \Big( \mu^\star \zeta\,S\,\psi(\zeta S) + \alpha^\star G < w^\star {\eps_{\rm ts}}/{\eps_{\rm tr}} \Big),
\eea
where $(\alpha^\star,\mu^\star,w^\star)$ is the unique solution to the scalar minimax problem \eqref{eq:minmax_app_thm5}.
\end{theorem}

\begin{proof}
Recall that for the GLM we have $\x_i\simiid \mathcal{N}(\mathbf{0},\Sn)$. Therefore, the decreasing nature of the loss leads to the following simplification in \eqref{eq:ermaftermax}:
\bea
\widehat{\thc_{n}}:&=
\min_{\thc_{n} \in \R^n}\; \max_{\substack{{\|\db_i\|_{\infty}\le \eps}\\[2pt]{i\in[m]}}} \;\frac{1}{m}\sum_{i=1}^m \Lm \left(y_i \left\langle\x_i+\db_i,\thc_{n}\right\rangle\right)  + \la\|\thc_{n}\|_2^2 \label{eq:mm_app}\\[4pt] 
&=\min_{\thc_{n} \in \R^n} \frac{1}{m}\sum_{i=1}^m \Lm(y_i\left\langle\x_i,\thc_{n}\right\rangle-\eps\|\thc_{n}\|_1) + \la\|\thc_{n}\|_2^2\;\nn\\[4pt]&= \min_{\substack{{\thc_{n}\in \R^n,\vb\in \R^m}\\ {v_i=y_i\x_i^{\top}\thc_{n}}}}  \frac{1}{m}\sum_{i=1}^m \Lm(v_i-\eps\|\thc_{n}\|_1) +\la\|\thc_{n}\|_2^2. \label{eq:dual}
\eea
In the last expression above, we have introduced additional variables $v_i$. This redundancy will allow us to write again the optimization as a minimax problem, but this time in a different ---more convenient in terms of analysis--- form compared to \eqref{eq:mm_app}. Specifically, the minimization in \eqref{eq:dual} is equivalent to the following:
\bea
 \min_{\thc_{n}\in \R^n,\vb\in \R^m}\max_{\ub\in\R^m}  \;\;\;\frac{1}{m}\sum_{i=1}^m \Lm(v_i-\eps\|\thc_{n}\|_1) + \frac{1}{m}\sum_{i=1}^m u_i\left(y_i\left\langle\x_i,\thc_{n}\right\rangle-v_i\right)+ \la\|\thc_{n}\|_2^2 \label{eq:afterdual0}.
\eea
We introduce the variable $\widetilde{\thc_{n}}\triangleq\Sn^{1/2}\thc_{n}$ and $\bar{\x}_i\triangleq\Sn^{-1/2}\x_i$  thus $\bar{\x}_i\simiid\mathcal{N}(\mathbf{0},\mathbb{I}_n)$. Based on the new notation, \eqref{eq:afterdual0} can be rewritten as:

\bea
\min_{\widetilde{\thc_{n}}\in \R^n,\vb\in \R^m}\max_{\ub\in\R^m}\;\;\;  \frac{1}{m}\sum_{i=1}^m \Lm\left(v_i-\eps\left\|\Sn^{-1/2}\widetilde{\thc_{n}}\right\|_1\right) + \frac{1}{m}\sum_{i=1}^m &u_i\left(y_i\left\langle\bar{\x}_i,\widetilde{\thc_{n}}\right\rangle-v_i\right)\nn\\&+\la\left\|\Sn^{-1/2}\widetilde{\thc_{n}}\right\|_2^2 \label{eq:afterdual}.
\eea
Next, we define the projection matrices $\Th, \Thp \in \R^{n\times n}$ based on $\widetilde{\thc_{n}^\star}\triangleq\Sn^{1/2}\thc_{n}^\star$ as follows,  
$$\Th\triangleq \frac{\widetilde{\thc_{n}^\star}\widetilde{{\thc_{n}^\star}}^\top}{\left\|\widetilde{\thc_n^\star}\right\|_2^2},\;\; \Thp\triangleq \mathbb{I}_{n} - \Th.$$
Since $\Th + \Thp = \mathbb{I}_{n}$, we deduce that \eqref{eq:afterdual} is equivalent to, 
\bea\label{eq:beforevector}
\min_{\widetilde{\thc_{n}}\in \R^n,\vb\in \R^m}\max_{\ub\in\R^m} \;\;\; &\frac{1}{m}\sum_{i=1}^m \Lm\left(v_i-\eps\left\|\Sn^{-1/2}\widetilde{\thc_{n}}\right\|_1\right) - \frac{1}{m}\sum_{i=1}^m u_iv_i + \frac{1}{m}\sum_{i=1}^m u_iy_i\left\langle\bar{\x}_i,\Th\widetilde{\thc_{n}}\right\rangle \\
&+ \frac{1}{m}\sum_{i=1}^m u_iy_i\left\langle\bar{\x}_i,\Thp\widetilde{\thc_{n}}\right\rangle
+\la\left\|\Sn^{-1/2}\widetilde{\thc_{n}}\right\|_2^2.\nn
\eea
Splitting $\widetilde{\thc_{n}}$ based on $\Th, \Thp$ has two purposes. First it immediately reveals the two terms $\|\Th\widetilde{\thc_{n}}\|_2$ and $\|\Thp\widetilde{\thc_{n}}\|_2$ of interest to us in view of Lemma \ref{lem:gen_error} . Second, as we will see, it allows the use of the CGMT.  

For compactness we write \eqref{eq:beforevector} in vector notation,
\bea
\min_{\widetilde{\thc_{n}}\in \R^n,\vb\in \R^m}\max_{\ub\in\R^m} \;\;\; \frac{\mathbf{1}_m^{\top}}{m} \Ellb\left(\vb-\eps\left\|\Sn^{-1/2}\widetilde{\thc_{n}}\right\|_1\mathbf{1}_m\right) - &\frac{\langle\ub,\vb\rangle}{m} + \frac{\left\langle\ub, Y\bar{X}\Th\widetilde{\thc_{n}}\right\rangle}{m} + \frac{\left\langle\ub, Y \bar{X}\Thp\widetilde{\thc_{n}}\right\rangle}{m} \nn\\&+  \la\left\|\Sn^{-1/2}\widetilde{\thc_{n}}\right\|_2^2,\label{eq:vectoront}
\eea
where
\begin{align}
\Ellb (\vb)&\triangleq [\Lm(v_1);\;\Lm(v_2);\;\cdots;\;\Lm(v_m)] \in \R^{m\times 1},\nn\\
Y &\triangleq {\rm{diag}}(y_1,\;y_2,\cdots,\;y_m) \in \R^{m\times m},\nn\\
 \bar{X} &\triangleq [\bar{\x}_1^{\top};\;\bar{\x}_2^{\top};\cdots;\;\bar{\x}_m^{\top}] \in \R^{m\times n}.
\end{align}

Before proceeding, we recall our main tool the Convex Gaussian Min-max Theorem \cite{sto,stoLASSO,COLT} which relies on Gordon's Gaussian Min-max theorem. The Gordon's Gaussian comparison inequality \cite{gorThm} compares the min-max value of two doubly indexed Gaussian processes $\mathcal{X}_{\w,\ub}, \mathcal{Y}_{\w,\ub}$ based on how their autocorrelation functions compare,
\begin{subequations}\label{eq:cgmt_PO_AO}
\begin{align}
\mathcal{X}_{\w,\ub} &\triangleq \ub^{\top} \G \w + \Gamma(\w,\ub),\\
\mathcal{Y}_{\w,\ub} &\triangleq \norm{\w}_2 \g^{\top} \ub + \norm{\ub}_2 \h^{\top} \w + \Gamma(\w,\ub),
\end{align}
\end{subequations}
where: $\G\in\mathbb{R}^{m\times n}$, $\g \in \mathbb{R}^m$, $\h\in\mathbb{R}^n$, they all have entries iid Gaussian; the sets $\mathcal{S}_{\w}\subset\R^n$ and $\mathcal{S}_{\ub}\subset\R^m$ are compact; and, $\Gamma: \mathbb{R}^n\times \mathbb{R}^m \to \mathbb{R}$. For these two processes, define the following (random) min-max optimization programs, which we refer to as the \emph{primary optimization} (PO) problem and the \emph{auxiliary optimization} (AO). 
\begin{subequations}
\begin{align}\label{eq:PO_loc}
\Phi(\G)&=\min\limits_{\w \in \mathcal{S}_{\w}} \max\limits_{\ub\in\mathcal{S}_{\ub}}\mathcal{X}_{\w,\ub},\\
\label{eq:AO_loc}
\phi(\g,\h)&=\min\limits_{\w \in \mathcal{S}_{\w}} \max\limits_{\ub\in\mathcal{S}_{\ub}} \mathcal{Y}_{\w,\ub}.
\end{align}
\end{subequations}
According to the version of the CGMT in Theorem 6.1 in \cite{master}, if the sets $\mathcal{S}_{\w}$ and $\mathcal{S}_{\ub}$ are convex and $\psi$ is continuous \emph{convex-concave} on $\mathcal{S}_{\w}\times \mathcal{S}_{\ub}$, then, for any $\nu \in \mathbb{R}$ and $t>0$, it holds
\begin{equation}\label{eq:cgmt}
\mathbb{P}\left( \abs{\Phi(\G)-\nu} > t\right) \leq 2 \mathbb{P}\left(  \abs{\phi(\g,\h)-\nu} > t \right).
\end{equation}
In words, concentration of the optimal cost of the AO problem around $\mu$ implies concentration of the optimal cost of the corresponding PO problem around the same value $\mu$.  Moreover, starting from \eqref{eq:cgmt} and under strict convexity conditions, the CGMT shows that concentration of the optimal solution of the AO problem implies concentration of the optimal solution of the PO to the same value. For example, if minimizers of \eqref{eq:AO_loc} satisfy $\norm{\w^\ast(\g,\h)}_2 \to \zeta^\ast$ for some $\zeta^\ast>0$, then, the same holds true for the minimizers of \eqref{eq:PO_loc}: $\norm{\w^\ast(\G)}_2 \to \zeta^\ast$ (Theorem 6.1(iii) in \cite{master}). Thus, one can analyze the AO to infer corresponding properties of the PO, the premise being of course that the former is simpler to handle than the latter. \\
Returning to our minimax problem \eqref{eq:vectoront}, we observe that the objective is convex in $(\widetilde{\thc_{n}},\vb)$ and concave in $\ub$. 
Also note that the term $Y\bar{X}\Th\widetilde{\thc_{n}}$ is independent of $Y\bar{X}\Thp\widetilde{\thc_{n}}$ as the entries of $Y$ depend only on $\bar{X}\Th$ which is orthogonal to $\bar{X}\Thp$, i.e., based on the definition of $\Th$ we have
$$y_i=\psi(\x_i^{\top}\thc_{n}^\star) = \psi(\bar{\x}_i^{\top}\widetilde{\thc_{n}^\star}) = \psi(\bar{\x}_i^{\top}\Th\widetilde{\thc_{n}^\star}).$$ 

Therefore, along the same lines as in the proof of Lemma \ref{lem:gen_error}, we can substitute $\bar{X}\Thp$ by $\hat{X}\Thp$ for a standard Gaussian matrix $\hat{X}$ that is independent of $\bar{X}$ and everything else in the objective of \eqref{eq:vectoront}.  Thus, we can use CGMT for PO in \eqref{eq:vectoront} with the choice
$$
\Gamma\left(\{\widetilde{\thc_{n}},\vb\},\ub\right)\triangleq \frac{\mathbf{1}_m^{\top}}{m} \Ellb\left(\vb-\eps\left\|\Sn^{-1/2}\widetilde{\thc_{n}}\right\|_1\mathbf{1}_m\right) - \frac{\langle\ub,\vb\rangle}{m} + \frac{\left\langle\ub, Y\bar{X}\Th\widetilde{\thc_{n}}\right\rangle}{m} + \la\left\|\Sn^{-1/2}\widetilde{\thc_{n}}\right\|_2^2.
$$
With this, we derive the following AO for \eqref{eq:vectoront}, 
\bea
\min_{\widetilde{\thc_{n}}\in \R^n,\vb\in \R^m}\max_{\ub\in\R^m}  \;\;&\frac{\mathbf{1}_m^{\top}}{m} \Ellb\left(\vb-\eps\left\|\Sn^{-1/2}\widetilde{\thc_{n}}\right\|_1\mathbf{1}_m\right) -\frac{\langle\ub,\vb\rangle}{m} + \frac{\left\langle\ub, Y\bar{X}\Th\widetilde{\thc_{n}}\right\rangle}{m} \nn\\ &+ \frac{\ub^{\top}Y\g \left\|\Thp\widetilde{\thc_{n}}\right\|_2}{m} + \frac{\|\ub^{\top}Y\|_2 \left\langle\h,\Thp\widetilde{\thc_{n}}\right\rangle}{m}
+\la\left\|\Sn^{-1/2}\widetilde{\thc_{n}}\right\|_2^2\,,\label{eq:afterCGMT}
\eea
where $\g\in \R^m$, $\h\in \R^n$ have entries i.i.d. standard normal.  Note that similar to \cite{master} and despite the fact that we are working with finite dimensional matrices now, we will consider the asymptotic limit at the end of the approach. Thus as the final optimization has a bounded solution in the high-dimensional limit, we can relax the assumption of compactness of the domain of optimization which is needed for CGMT.

To proceed, we observe that $Y\g \sim \mathcal{N}(0,1)$ and $\|\ub^{\top} Y\|_2 = \|\ub\|_2$. So, next we can optimize w.r.t $\ub$ to find that: 
\begin{align*}
&\max_{\ub\in\R^m} -\frac{\langle\ub,\vb\rangle}{m} + \frac{\left\langle\ub, Y\bar{X}\Th\widetilde{\thc_{n}}\right\rangle}{m} + \frac{\langle\ub,\g\rangle \left\|\Thp\widetilde{\thc_{n}}\right\|_2}{m} + \frac{\|\ub\|_2 \left\langle\h,\Thp\widetilde{\thc_{n}}\right\rangle}{m}  \\=
&\max_{\ub\in\R^m, \frac{\|\ub\|_2}{\sqrt{m}} = \beta}  \frac{1}{m}\left\langle\ub,-\vb + Y\bar{X}\Th\widetilde{\thc_{n}} + \g \left\|\Thp\widetilde{\thc_{n}}\right\|_2\right\rangle + \frac{\beta \left\langle\h,\Thp\widetilde{\thc_{n}}\right\rangle}{\sqrt{m}}  \\=
& \max_{\beta\in \R_{+}} \frac{\beta}{\sqrt{m}} \left\|-\vb + Y\bar{X}\Th\widetilde{\thc_{n}} + \g \left\|\Thp\widetilde{\thc_{n}}\right\|_2 \right\|_2 +  \frac{\beta \left\langle\h,\Thp\widetilde{\thc_{n}}\right\rangle}{\sqrt{m}}.
\end{align*}
Hence we replace this in \eqref{eq:afterCGMT} to simplify the objective as follows, 
\bea\label{eq:minmax}
\min_{\widetilde{\thc_{n}}\in \R^n,\vb\in \R^m}\max_{\beta\in\R_+} \;\;\; &\frac{\mathbf{1}_m^{\top}}{m} \Ellb\left(\vb-\eps\left\|\Sn^{-1/2}\widetilde{\thc_{n}}\right\|_1\mathbf{1}_m\right)  + \frac{\beta}{\sqrt{m}} \left\|-\vb + Y\bar{X}\Th\widetilde{\thc_{n}}+ \g \left\|\Thp\widetilde{\thc_{n}}\right\|_2 \right\|_2  \nn\\
&\hspace{1.5in}\;+\;  \frac{\beta \left\langle\h,\Thp\widetilde{\thc_{n}}\right\rangle}{\sqrt{m}}+ \la\left\|\Sn^{-1/2}\widetilde{\thc_{n}}\right\|_2^2.
\eea
Next, our trick is to dualize the the term $\eps\left\|\Sn^{-1/2}\widetilde{\thc_{n}}\right\|_1$ inside the loss function. For this, we first introduce an extra optimization variable $w>0$ along with the constraint $w = \eps\left\|\Sn^{-1/2}\widetilde{\thc_{n}}\right\|_1$ and then turn this into an unconstrained min-max problem. This yields the following equivalent formulation of \eqref{eq:minmax},
\bea
&\min_{\substack{{\widetilde{\thc_{n}}\in \R^n,\vb\in \R^m}\\ { w=\eps\left\|\Sn^{-1/2}\widetilde{\thc_{n}}\right\|_1}}}\max_{\beta\in\R_+} \;\; \frac{\mathbf{1}_m^{\top}}{m} \Ellb\left(\vb-w\mathbf{1}_m\right)  + \frac{\beta}{\sqrt{m}} \left\|-\vb + Y\bar{X}\Th\widetilde{\thc_{n}}+ \g \left\|\Thp\widetilde{\thc_{n}}\right\|_2 \right\|_2\nn\\&\hspace{1.5in} \;+\;  \frac{\beta \left\langle\h,\Thp\widetilde{\thc_{n}}\right\rangle}{\sqrt{m}}  +
 \la\left\|\Sn^{-1/2}\widetilde{\thc_{n}}\right\|_2^2  \label{eq:wdual}\\
&=\min_{\widetilde{\thc_{n}}\in \R^n,\vb\in \R^m,w\in\R_+}\;\max_{\beta,\gamma\in\R_+}\;\; \frac{\mathbf{1}_m^{\top}}{m} \Ellb(\vb - w\mathbf{1}_m) + \gamma\left(\eps\left\|\Sn^{-1/2}\widetilde{\thc_{n}}\right\|_1-w\right)  +\la\left\|\Sn^{-1/2}\widetilde{\thc_{n}}\right\|_2^2\nn\\
&\hspace{1.5in}+ \frac{\beta}{\sqrt{m}} \left\|-\vb + Y\bar{X}\Th\widetilde{\thc_{n}} + \g \left\|\Thp\widetilde{\thc_{n}}\right\|_2 \right\|_2 +  \frac{\beta \left\langle\h,\Thp\widetilde{\thc_{n}}\right\rangle}{\sqrt{m}}\,.\label{eq:minmax_2}
\eea

The key reason behind this reformulation is to allow optimization with respect to $\widetilde{\thc_{n}}$ which is the primary variable of interest in the objective function. As we will see, our goal is optimizing with respect to the direction of $\Th^\perp\widetilde{\thc_{n}}$ and $\Th\widetilde{\thc_{n}}$, which according to Lemma \ref{lem:gen_error} comprise the terms parametrizing the adversarial error of the estimator $\widetilde{\thc_{n}}$. To do this, we introduce the slack variable $\widetilde{\rhoc_{n}}$ for $\widetilde{\thc_{n}}$ (equivalently $\rhoc_{n}$ for $\thc_{n}$ where $\rhoc_{n}\triangleq\Sn^{-1/2}\widetilde{\rhoc_{n}}$) and rewrite the optimization problem \eqref{eq:minmax_2}, 
\begin{align}
&\min_{\substack{{\widetilde{\thc_{n}}\in \R^n,\vb\in \R^m,w\in\R_+}\\ {{\rm s.t.\;}\; \Sn^{-1/2}\widetilde{\thc_{n}}=\rhoc_{n}}}} \;\max_{\beta,\gamma\in\R_+} \frac{\mathbf{1}_m^{\top}}{m} \Ellb(\vb - w\mathbf{1}_m) + \gamma\left(\eps\left\|\Sn^{-1/2}\widetilde{\thc_{n}}\right\|_1-w\right) +  \la\left\|\Sn^{-1/2}\widetilde{\thc_{n}}\right\|_2^2\nn\\
&\qquad\qquad\qquad+ \frac{\beta}{\sqrt{m}} \left\|-\vb + Y\bar{X}\Th\widetilde{\thc_{n}} + \g \left\|\Thp\widetilde{\thc_{n}}\right\|_2 \right\|_2 +  \frac{\beta \left\langle\h,\Thp\widetilde{\thc_{n}}\right\rangle}{\sqrt{m}}\nn\\
& = \min_{\widetilde{\rhoc_{n}},\widetilde{\thc_{n}}\in \R^n,\vb\in \R^m,w\in\R}\;\max_{\beta,\gamma\in\R_+,\lac\in\R^n} \frac{\mathbf{1}_m^{\top}}{m} \Ellb(\vb - w\mathbf{1}_m)  - \gamma w + \eps\gamma \left\|\Sn^{-1/2}\widetilde{\rhoc_{n}}\right\|_1 + r\left\|\Sn^{-1/2}\widetilde{\rhoc_n}\right\|_2^2\nn\\& \qquad\qquad\;\;\;\;\;\; + \left\langle\frac{\lac}{\sqrt{n}},\widetilde{\rhoc_{n}}-\widetilde{\thc_{n}}\right\rangle\;+\frac{\beta}{\sqrt{m}} \left\|-\vb + Y\bar{X}\Th\widetilde{\thc_{n}} + \g \left\|\Thp\widetilde{\thc_{n}}\right\|_2\right\|_2 +  \frac{\beta \left\langle\h,\Thp\widetilde{\thc_{n}}\right\rangle}{\sqrt{m}}  \label{eq:minmax_3}\,.
\end{align}
In \eqref{eq:minmax_3}, we applied the Lagrangian method to both of terms $\left\|\Sn^{-1/2}\widetilde{\thc_{n}}\right\|_1$ and $\left\|\Sn^{-1/2}\widetilde{\thc_{n}}\right\|_2^2$. This is essential to scalarizing the objective function based on $\Th\widetilde{\thc_n}$ and $\Thp\widetilde{\thc_n}$, which is our next step. As a remark and as we will see in Section \ref{sec:GLM,infty,caseIII}, only in the special case of $\Sn=\mathbb{I}_n$, it is possible to apply the Lagrangian to the $\ell_1$ norm and simply decompose $\|\widetilde{\thc_n}\|_2^2$ as 
$$\|\widetilde{\thc_{n}}\|_2^2\;\; = \|\Th\widetilde{\thc_{n}}\|_2^2 \;\;+\;\; \|\Thp\widetilde{\thc_{n}}\|_2^2\, .$$ 
Now, we  can finally optimize w.r.t the direction of $\Thp\widetilde{\thc_{n}}$. First, note that
$$\left\langle\lac,\widetilde{\rhoc_{n}}-\widetilde{\thc_{n}}\right\rangle = \left\langle\lac, \Th\left(\widetilde{\rhoc_{n}}-\widetilde{\thc_{n}}\right)\right\rangle + \Big\langle\lac, \Thp \widetilde{\rhoc_{n}}\Big\rangle - \left\langle\lac, \Thp \widetilde{\thc_{n}}\right\rangle.$$ 
With this decomposition, we can optimize w.r.t. $\Thp\widetilde{\thc_{n}}$ as follows,
\bea
&\min_{\Thp\widetilde{\thc_{n}}\in\R^n}  - \left\langle\frac{\lac}{\sqrt{n}},\Th^\perp\widetilde{\thc_{n}}\right\rangle+ \frac{\beta}{\sqrt{m}} \left\|-\vb + Y\bar{X}\Th\widetilde{\thc_{n}} + \g \left\|\Thp\widetilde{\thc_{n}}\right\|_2 \right\|_2 +  \frac{\beta \left\langle\h,\Thp\widetilde{\thc_{n}}\right\rangle}{\sqrt{m}} \nn\\= 
& \min_{\Thp\widetilde{\thc_{n}}\in\R^n, \|\Thp \widetilde{\thc_{n}}\|_2=\alpha}     \left\langle -\frac{\lac}{\sqrt{n}} + \frac{\beta \h}{\sqrt{m}},\Thp\widetilde{\thc_{n}}\right\rangle +  \frac{\beta}{\sqrt{m}} \left\|-\vb + Y\bar{X}\Th\widetilde{\thc_{n}} + \alpha\g  \right\|_2\nn\\[4pt]=
&\min_{\alpha\in\R_+} -\alpha\left\|-\frac{\Th^\perp\lac}{\sqrt{n}}+\frac{\beta }{\sqrt{m}}\Th^\perp\h \right\|_2 + \frac{\beta}{\sqrt{m}} \Big\|-\vb + Y\bar{X}\Th\widetilde{\thc_{n}} + \alpha\g \Big\|_2  \label{eq:minmax_4}\,.
\eea

By replacing \eqref{eq:minmax_4} in \eqref{eq:minmax_3} we have,
\bea
&\min_{\widetilde{\rhoc_{n}},\Th\widetilde{\thc_{n}}\in \R^n,\vb\in \R^m,w, \alpha\in\R_+}\;\;\max_{\beta, \gamma\in\R_+,\lac\in\R^n} \;\; \frac{\mathbf{1}_m^{\top}}{m} \Ellb(\vb - w\mathbf{1}_m)  - \gamma w + \eps\gamma \left\|\Sn^{-1/2}\widetilde{\rhoc_{n}}\right\|_1 \nn
\\  &\hspace{1in}+ \left\langle\frac{\lac}{\sqrt{n}}, \Th\left(\widetilde{\rhoc_{n}}-\widetilde{\thc_{n}}\right)\right\rangle+ \; r\left\|\Sn^{-1/2}\widetilde{\rhoc_n}\right\|_2^2 + \left\langle\frac{\lac}{\sqrt{n}},\Th^\perp\widetilde{\rhoc_{n}}\right\rangle \nn\\&\hspace{1in}-\alpha\left\|-\frac{\Th^\perp\lac}{\sqrt{n}}+\frac{\beta }{\sqrt{m}}\Th^\perp\h \right\|_2 + \frac{\beta}{\sqrt{m}} \Big\|-\vb + Y\bar{X}\Th\widetilde{\thc_{n}} + \alpha\g \Big\|_2.\,\label{eq:beforeChristos'trick}
\eea


We replace $\eps$ with $\eps_{\rm tr}/\sqrt{n}$ specialized to the case of $q=\infty$. Such normalization is necessary to guarantee the boundedness of the solutions to \eqref{eq:beforeChristos'trick} when $\eps_{\rm tr} = \mathcal{O}(1)$. To continue, we will use the same trick as in \cite{COLT} that $x= \min_{\tau\in\R_+}  \frac{x^2}{2\tau} + \frac{\tau}{2}$ for every $x\in\R_+$. Thus we may rewrite the last two terms based on the squared $\ell_2$ norm by introducing two new variables $\tau_1,\tau_2\in \R_+$ to obtain the following new objective, 
\bea
&\min_{\widetilde{\rhoc_{n}},\Th\widetilde{\thc_{n}}\in \R^n,\vb\in \R^m,w,\alpha,\tau_1\in\R_+}\;\;\max_{\tau_2,\beta,\gamma\in\R_+,\lac\in\R^n} \;\; \frac{\mathbf{1}_m^{\top}}{m} \Ellb(\vb - w\mathbf{1}_m)  - \gamma w +\frac{\eps_{\rm tr}\gamma}{\sqrt{n}} \left\|\Sn^{-1/2}\widetilde{\rhoc_{n}}\right\|_1 \nn\\[5pt]
&\hspace{.4in}+\left\langle\frac{\lac}{\sqrt{n}}, \Th\left(\widetilde{\rhoc_{n}}-\widetilde{\thc_{n}}\right)\right\rangle \;\;+\;\; r\left\|\Sn^{-1/2}\widetilde{\rhoc_n}\right\|_2^2 \;\; + \;\;\left\langle\frac{\lac}{\sqrt{n}},\Th^\perp\widetilde{\rhoc_{n}}\right\rangle\nn\\&\hspace{.4in}-\frac{\alpha}{2\tau_2 n}\left\|-\Th^\perp\lac+\frac{\beta}{\sqrt{\delta}}\Th^\perp\h \right\|_2^2 - \frac{\alpha\tau_2}{2}+ \frac{\beta}{2\tau_1 m}\Big\|-\vb + Y\bar{X}\Th\widetilde{\thc_{n}} + \alpha\g \Big\|_2^2 + \frac{\beta\tau_1}{2} \label{eq:minmax_5}\,,
\eea
where we also used the fact that $m/n=\delta$. By the following chain of equations, we simplify the maximization with respect to $\lac$, 
\bea
&\;\;\;\max _{\lac\in\R^n}\;\;\;\;\;\left\langle\frac{\lac}{\sqrt{n}},\Th^\perp\widetilde{\rhoc_{n}}\right\rangle -\frac{\alpha}{2\tau_2 n}\left\|-\Th^\perp\lac+\frac{\beta}{\sqrt{\delta}}\Th^\perp\h \right\|_2^2   +\left\langle\frac{\lac}{\sqrt{n}}, \Th\left(\widetilde{\rhoc_{n}}-\widetilde{\thc_{n}}\right)\right\rangle\nn\\[7pt]
&\hspace{.02in}=\max _{\lac\in\R^n} \;\;\;\;\;-\frac{\alpha}{2n\tau_2} \left\| \Th^\perp\left(\frac{\beta}{\sqrt{\delta}}\h-\lac+\frac{\tau_2\widetilde{\rhoc_{n}}\sqrt{n}}{\alpha}\right)\right\|_2^2 + \frac{\tau_2}{2n\alpha}\left\|\Th^\perp\left(\widetilde{\rhoc_{n}}\sqrt{n}+ \frac{\alpha\beta}{\tau_2\sqrt{\delta}}\h\right) \right\|_2^2 \nn \\[7pt]
&\hspace{1in}- \frac{\alpha\beta^2}{2m\tau_2}\left\|\Th^\perp\h\right\|_2^2  + \left\langle\frac{\lac}{\sqrt{n}}, \Th\left(\widetilde{\rhoc_{n}}-\widetilde{\thc_{n}}\right)\right\rangle \label{eq:lambda_first}\\[7pt]
&\hspace{.02in}=\max _{\Th\lac\in\R^n} \max_{\Thp\lac\in\R^n}\;\;\;\; -\frac{\alpha}{2n\tau_2} \left\| \Th^\perp\left(\frac{\beta}{\sqrt{\delta}}\h-\lac+\frac{\tau_2\widetilde{\rhoc_{n}}\sqrt{n}}{\alpha}\right)\right\|_2^2 + \frac{\tau_2}{2n\alpha}\left\|\Th^\perp\left(\widetilde{\rhoc_{n}}\sqrt{n}+ \frac{\alpha\beta}{\tau_2\sqrt{\delta}}\h\right) \right\|_2^2 \nn \\[7pt]
&\hspace{1in}- \frac{\alpha\beta^2}{2m\tau_2}\left\|\Th^\perp\h\right\|_2^2  + \left\langle\frac{\Th\lac}{\sqrt{n}}, \Th\left(\widetilde{\rhoc_{n}}-\widetilde{\thc_{n}}\right)\right\rangle \label{eq:lambda_second}\\[7pt]
&\hspace{.02in}=\max_{\Th\lac\in\R^n}\;\;\; \frac{\tau_2}{2n\alpha}\left\|\Th^\perp\left(\widetilde{\rhoc_{n}}\sqrt{n}+ \frac{\alpha\beta}{\tau_2\sqrt{\delta}}\h\right) \right\|_2^2 - \frac{\alpha\beta^2}{2m\tau_2}\left\|\Th^\perp\h\right\|_2^2  + \left\langle\frac{\Th\lac}{\sqrt{n}}, \Th\left(\widetilde{\rhoc_{n}}-\widetilde{\thc_{n}}\right)\right\rangle \nn\\[7pt]
&\hspace{.02in}=\frac{\tau_2}{2n\alpha}\left\|\Th^\perp\left(\widetilde{\rhoc_{n}}\sqrt{n}+ \frac{\alpha\beta}{\tau_2\sqrt{\delta}}\h\right) \right\|_2^2 - \frac{\alpha\beta^2}{2\delta\tau_2}.\label{eq:minmax_6}
\eea
In deriving \eqref{eq:lambda_first} we used completion of squares. In \eqref{eq:lambda_second}, we decompose maximization of $\lac$ into $\Th\lac$ and $\Thp\lac$ and used the fact that $\Thp+\Th=\mathbb{I}_n$ and $\Thp\Th=\mathbf{0}_n$. In the last line we used the fact that $\|\Thp\h\|_2^2\rightarrow n$. We note that the last line is true subject to the constraint $\Th\widetilde{\rhoc_{n}} = \Th\widetilde{\thc_{n}}$, which ensures boundedness of the min-max objective. We include this constraint in the next step of the proof. Therefore, inserting \eqref{eq:minmax_6} back in \eqref{eq:minmax_5} we derive, 
\bea\label{eq:minmax_7}
&\min_{\substack{{\widetilde{\rhoc_{n}},\Th\widetilde{\thc_{n}}\in \R^n,\vb\in \R^m,w, \alpha,\tau_1\in\R_+}\\ {\text{s.t.}\;\; \Th\widetilde{\rhoc_{n}}=\Th\widetilde{\thc_{n}}}}}\;\;\ \max_{\gamma, \tau_2,\beta\in\R_+} \;\; \frac{\mathbf{1}_m^{\top}}{m} \Ellb(\vb - w\mathbf{1}_m)  - \gamma w +  \frac{\eps_{\rm tr}\gamma}{\sqrt{n}} \left\|\Sn^{-1/2}\widetilde{\rhoc_{n}}\right\|_1  \nn\\&\hspace{1in}\;+ \; r\left\|\Sn^{-1/2}\widetilde{\rhoc_n}\right\|_2^2+ \frac{\tau_2}{2n\alpha}\left\|\Th^\perp\left(\widetilde{\rhoc_{n}}\sqrt{n}+ \frac{\alpha\beta}{\tau_2\sqrt{\delta}}\h\right) \right\|_2^2   -\frac{\alpha\beta^2}{2\delta\tau_2} - \frac{\alpha\tau_2}{2}\nn\\&\hspace{1in}+ \frac{\beta}{2\tau_1 m} \Big\|-\vb +Y\bar{X}\Th\widetilde{\thc_{n}} + \alpha\g \Big\|_2^2 + \frac{\beta\tau_1}{2}.
\eea
Recalling $\Th^\perp\triangleq\mathbb{I}-\Th$, we can deduce
\bea
\frac{1}{n}\left\|\Th^\perp\left(\widetilde{\rhoc_{n}}\sqrt{n}+ \frac{\alpha\beta}{\tau_2\sqrt{\delta}}\h\right) \right\|_2^2 = \frac{1}{n}&\left\|\widetilde{\rhoc_{n}}\sqrt{n} + \frac{\alpha\beta}{\tau_2\sqrt{\delta}}\h\right\|_2^2 - \left\|\Th\widetilde{\rhoc_{n}}\right\|_2^2 - \frac{\alpha^2\beta^2}{n\tau_2^2\delta}\left\|\Th\h\right\|_2^2 \nn\\[4pt]
&- 4\frac{\alpha\beta}{\tau_2\sqrt{m}}\h^{\top}\Th\widetilde{\rhoc_{n}}. \label{eq:beforeVani}
\eea
Since $\| \widetilde{\thc_{n}^\star} \|_2 \rP \zeta$ where $\zeta=\mathcal{O}(1)$ by Assumption \ref{ass:2}, we can see that $$\|\Th\h\|^2_2 = \mathcal{O}(1) ,\;\;\; \h^{\top}\Th\widetilde{\rhoc_{n}} = \h^{\top}\Th\widetilde{\thc_{n}}= \mu \h^\top\widetilde{\thc_{n}^\star}=\mathcal{O}(1),$$ 
which implies that the last two terms in \eqref{eq:beforeVani} vanish asymptotically and we have that,
\bea
\frac{1}{n}\left\|\Th^\perp\left(\widetilde{\rhoc_{n}}\sqrt{n}+ \frac{\alpha\beta}{\tau_2\sqrt{\delta}}\h\right) \right\|_2^2 &= \frac{1}{n}\left\|\widetilde{\rhoc_{n}}\sqrt{n} + \frac{\alpha\beta}{\tau_2\sqrt{\delta}}\h\right\|_2^2 - \left\|\Th\widetilde{\rhoc_{n}}\right\|_2^2 \nn\\
& =  \frac{1}{n}\left\|\widetilde{\rhoc_{n}}\sqrt{n} + \frac{\alpha\beta}{\tau_2\sqrt{\delta}}\h\right\|_2^2 - \mu^2\left\|\widetilde{\thc_{n}^\star}\right\|_2^2.\nn
\eea
The last line is due to the constraint in \eqref{eq:minmax_7} i.e., $\Th\widetilde{\rhoc_{n}} = \Th\widetilde{\thc_{n}}$ (or equivalently $\frac{\langle\widetilde{\thc_{n}^\star},\widetilde{\rhoc_{n}}\rangle}{\left\|\widetilde{\thc_n^\star}\right\|_2^2} = \mu$, based on the definition of $\Th$ and $\mu$).
Therefore by plugging this in \eqref{eq:minmax_7} and introducing the Lagrangian multiplier $\eta\in\R$, \eqref{eq:minmax_7} can be equivalently rewritten as follows,
\bea
&\min_{\substack{{\widetilde{\rhoc_{n}},\Th\widetilde{\thc_{n}}\in \R^n,\vb\in \R^m,w, \alpha,\tau_1\in\R_+}\\ {\text{s.t.}\;\; \left\langle\widetilde{\thc_{n}^\star},\widetilde{\rhoc_n}\right\rangle\Big/{\left\|\widetilde{\thc_n^\star}\right\|_2^2}} = \mu}}\;\;\ \max_{\gamma, \tau_2,\beta\in\R_+} \, \frac{\mathbf{1}_m^{\top}}{m} \Ellb(\vb - w\mathbf{1}_m)  - \gamma w +\frac{\eps_{\rm tr}\gamma}{\sqrt{n}} \left\|\Sn^{-1/2}\widetilde{\rhoc_{n}}\right\|_1 \; \nn
\\&\hspace{.9in}+ r\left\|\Sn^{-1/2}\widetilde{\rhoc_n}\right\|_2^2+ 
\frac{\tau_2}{2n\alpha}\left\|\widetilde{\rhoc_{n}}\sqrt{n}+ \frac{\alpha\beta}{\tau_2\sqrt{\delta}}\h \right\|_2^2 - \frac{\mu^2\tau_2\left\|\widetilde{\thc_n^\star}\right\|_2^2}{2\alpha}
-\frac{\alpha\beta^2}{2\delta\tau_2} - \frac{\alpha\tau_2}{2}\nn\\[4pt]&\hspace{.9in}+ \frac{\beta}{2\tau_1 m} \left\|-\vb +Y\bar{X}\Th\widetilde{\thc_{n}} + \alpha\g \right\|_2^2  + \frac{\beta\tau_1}{2}  \nn\\= 
&\min_{\substack{{\widetilde{\rhoc_{n}}\in \R^n,\vb\in \R^m},\\ {w, \alpha,\tau_1\in\R_+,\mu\in\R}}}\;\;\ \max_{\gamma, \tau_2,\beta\in\R_+,\eta\in\R} \, \frac{\mathbf{1}_m^{\top}}{m} \Ellb(\vb - w\mathbf{1}_m)  - \gamma w +\frac{\eps_{\rm tr}\gamma}{\sqrt{n}}\left\|\Sn^{-1/2}\widetilde{\rhoc_{n}}\right\|_1 + r\left\|\Sn^{-1/2}\widetilde{\rhoc_n}\right\|_2^2  \nn\\ 
&\;\;+\frac{\tau_2}{2n\alpha}\left\|\widetilde{\rhoc_{n}}\sqrt{n}+ \frac{\alpha\beta}{\tau_2\sqrt{\delta}}\h \right\|_2^2 - \frac{\mu^2\tau_2\left\|\widetilde{\thc_n^\star}\right\|_2^2}{2\alpha}
-\frac{\alpha\beta^2}{2\delta\tau_2} - \frac{\alpha\tau_2}{2}+ \frac{\beta}{2\tau_1 m} \left\|-\vb + \mu Y\bar{X}\widetilde{\thc_{n}^\star} + \alpha\g \right\|_2^2  \nn \\ &\;\; + \frac{\beta\tau_1}{2}   + \eta\left(\mu-\frac{\langle\widetilde{\thc_{n}^\star},\widetilde{\rhoc_{n}}\rangle}{\big\|\widetilde{\thc_n^\star}\big\|_2^2}\right).
\label{eq:mimax_8}
\eea
Minimization w.r.t $\vb$ can be written based on the moreau-envelope of $\Ellb$:
\bea
&\min_{\vb\in \R^m}  \frac{\mathbf{1}_m^{\top}}{m} \Ellb\left(\vb - w\mathbf{1}_m\right) + \frac{\beta}{2\tau_1 m} \left\|-\vb + \mu Y\bar{X}\widetilde{\thc_{n}^\star} + \alpha\g \right\|_2^2 \nn\\&\hspace{2in}= \frac{1}{m}\env{\small{\Ellb}}{\mu Y\bar{X}\widetilde{\thc_{n}^\star} + \alpha\g-w\mathbf{1}_m}{\frac{\tau_1}{\beta}}.
\eea
Our final key step is to write the minimization with respect to $\widetilde{\rhoc_{n}}\in\R^n$ based on the Moreau-envelope of the $\ellb_1 + \ellb_2^2$ norms. To this end, we rewrite the terms in \eqref{eq:mimax_8} consisting of $\widetilde{\rhoc_{n}}$ as following,
\begin{align}
&\min_{\widetilde{\rhoc_{n}}\in\R^n}  \frac{\eps_{\rm tr}\gamma}{\sqrt{n}}\left\|\Sn^{-1/2}\widetilde{\rhoc_{n}}\right\|_1 +  r\left\|\Sn^{-1/2}\widetilde{\rhoc_n}\right\|_2^2 + \frac{\tau_2}{2n\alpha}\left\|\widetilde{\rhoc_{n}}\sqrt{n}+ \frac{\alpha\beta}{\tau_2\sqrt{\delta}}\h \right\|_2^2 -\eta\frac{\left\langle\widetilde{\thc_{n}^\star},\widetilde{\rhoc_{n}}\right\rangle}{\left\|\widetilde{\thc_n^\star}\right\|_2^2}\nn \\[4pt] =
&\min_{\widetilde{\rhoc_{n}}\in\R^n} \frac{\eps_{\rm tr}\gamma}{\sqrt{n}}\left\|\Sn^{-1/2}\widetilde{\rhoc_{n}}\right\|_1 + r\left\|\Sn^{-1/2}\widetilde{\rhoc_n}\right\|_2^2+ \frac{\tau_2}{2\alpha n} \left\|\widetilde{\rhoc_{n}} \sqrt{n}+ \frac{\alpha\beta}{\tau_2\sqrt{\delta}}\h - \frac{\eta\alpha\sqrt{n}}{\tau_2\left\|\widetilde{\thc_n^\star}\right\|_2^2}\widetilde{\thc_{n}^\star}\right\|^2_2 \nn\\&\hspace{1.5in}-\frac{\eta^2\alpha}{2\tau_2\left\|\widetilde{\thc_n^\star}\right\|^2_2} - \frac{\alpha\beta\eta}{\sqrt{m}{\tau_2}\left\|\widetilde{\thc_n^\star}\right\|_2^2}\left\langle\widetilde{{\thc_{n}^\star}},\h \right\rangle \nn\\[4pt]=
&\min_{\widetilde{\rhoc_{n}}\in\R^n} \frac{\eps_{\rm tr}\gamma}{\sqrt{n}}\left\|\Sn^{-1/2}\widetilde{\rhoc_{n}}\right\|_1 + r\left\|\Sn^{-1/2}\widetilde{\rhoc_n}\right\|_2^2+ \frac{\tau_2}{2\alpha n} \left\|\widetilde{\rhoc_{n}} \sqrt{n}+ \frac{\alpha\beta}{\tau_2\sqrt{\delta}}\h - \frac{\eta\alpha\sqrt{n}}{\tau_2\left\|\widetilde{\thc_n^\star}\right\|^2_2}\widetilde{\thc_n^\star}\right\|^2_2 \nn\\&\hspace{1.5in}- \frac{\eta^2\alpha}{2\tau_2\left\|\widetilde{\thc_n^\star}\right\|^2_2}  \label{eq:beforeME1}\\[4pt]=
& \min_{\widetilde{\rhoc_{n}}\in\R^n} \frac{\eps_{\rm tr}\gamma}{n}\left\|\widetilde{\rhoc_{n}}\sqrt{n}\right\|_1 +\frac{r}{n}\left\|\widetilde{\rhoc_n}\sqrt{n}\right\|_2^2\;\nn\\
&\;\;\;\;+ \frac{\tau_2}{2\alpha n} \left\|\Sn^{1/2}\left(\widetilde {\rhoc_{n}} \sqrt{n}+ \frac{\alpha\beta}{\tau_2\sqrt{\delta}}\Sn^{-1/2}\h - \frac{\eta\alpha\sqrt{n}}{\tau_2\left\|\widetilde{\thc_n^\star}\right\|^2_2}\Sn^{-1/2}\widetilde{\thc_{n}^\star}\right)\right\|^2_2 - \frac{\eta^2\alpha}{2\tau_2\left\|\widetilde{\thc_{n}^\star}\right\|^2_2}.
\label{eq:beforeME2}
\end{align}
Here, the first step follows from the completion of squares while the second step follows from the fact that $\widetilde{\thc_{n}^\star}^\top\h=\mathcal{O}(\|\widetilde{\thc_n^\star}\|)=\mathcal{O}(1)$ and thus the last term asymptotically vanishes. Now, note that the minimization w.r.t. $\widetilde{\rhoc_n}$ in \eqref{eq:beforeME2} is equivalent to the following Moreau-Envelope function:
\bea
\frac{ \eps_{\rm tr}\gamma}{n}\env{\left(\small{\ellb_1}+\frac{r}{\Large{\eps}_{ \tiny{\rm tr}}\gamma}\ellb_2^2,\,\Sn\right)}{\frac{\alpha\beta}{\tau_2\sqrt{\delta}}\Sn^{-1/2}\h + \frac{\alpha\eta\sqrt{n}}{\tau_2\left\|\widetilde{\thc_{n}^\star}\right\|^2_2}{\thc_{n}^\star}}{\frac{\alpha\eps_{\rm tr}\gamma}{\tau_2}},\nn
\eea
where recall the definition of $\env{\left(\small{\ellb_1}+C\, \ellb_2^2,\,\mathbf{S}\right)}{\x}{\tau}$ in \eqref{eq:moreau_sigma}.
Thus, the following objective function is derived by replacing the Moreau-envelopes in \eqref{eq:mimax_8},
\bea
&\min_{\substack{{\alpha,\tau_1,w\in\R_+,}\\ {\mu\in\R}}}\;\;\max_{\substack{{\tau_2,\beta,\gamma\in\R_+,}\\ {\eta\in\R}}} \,   
-\gamma w - \frac{\mu^2\tau_2}{2\alpha}\left\|\widetilde{\thc_{n}^\star}\right\|^2_2-\frac{\alpha\beta^2}{2\delta\tau_2} - \frac{\alpha\tau_2}{2}+ \frac{\beta\tau_1}{2} + \eta \mu- \frac{\eta^2\alpha}{2\tau_2\left\|\widetilde{\thc_{n}^\star}\right\|^2_2} \nn \\
&\hspace{1.1in}+\frac{1}{m}\env{{\small{\Ellb}}}{ \mu Y\bar{X}\widetilde{\thc_{n}^\star}+ \alpha\g-w\mathbf{1}_m}{\frac{\tau_1}{\beta}} \nn\\&\hspace{1.1in}+ \frac{ \eps_{\rm tr}\gamma}{n}\env{\left(\small{\ellb_1}+\frac{r}{\Large{\eps}_{ \tiny{\rm tr}}\gamma}\ellb_2^2,\,\Sn\right)}{\frac{\alpha\beta}{\tau_2\sqrt{\delta}}\Sn^{-1/2}\h + \frac{\alpha\eta\sqrt{n}}{\tau_2\left\|\widetilde{\thc_{n}^\star}\right\|^2_2}{\thc_{n}^\star}}{\frac{\alpha\eps_{\rm tr}\gamma}{\tau_2}}.\label{eq:minmax_9}
\eea

We note that based on the definition of $\Th$ the entry $i$ on the diagonal of $Y$, denoted by $y_i$ is derived as $y_i = \psi\left(\left\langle\x_i,\thc_{n}^\star\right\rangle\right) = \psi(\langle\bar{\x}_i,\widetilde{\thc_{n}^\star}\rangle) $, where
$$
\left\langle\bar{\x}_i,\widetilde{\thc_{n}^\star}\right\rangle \sim \mathcal{N}\left(0,{\thc_{n}^\star}^\top\Sn{\thc_{n}^\star}\right).
$$
Therefore it yields that 
 $$
 \mu Y\bar{X}\widetilde{\thc_{n}^\star}\rP \mu\zeta\, \left(\textbf{s} \odot{\Psi}(\zeta\textbf{s})\right),
 $$
  where $\Psi(\zeta\textbf{s})\triangleq \left[\psi(\zeta s_1);\cdots;\psi(\zeta s_m)\right]$ for the vector $\textbf{s}\in\R^m$ with i.i.d standard normal entries $s_i$ and by Assumption \ref{ass:2}, $\zeta$ denotes the high-dimensional limit of ${\thc_{n}^\star}^\top\Sn{\thc_{n}^\star}$.
  Therefore based on the separability of the Moreau-envelope $\mathcal{M}_{\Ellb}$ we have,
  $$
   \frac{1}{m}\env{{\small{\Ellb}}}{ \mu Y\bar{X}\widetilde{\thc_{n}^\star}+ \alpha\g-w\mathbf{1}_m}{\frac{\tau_1}{\beta}} \rP \E_{S,G}\left[\env{\Lm}{\alpha G+\mu \zeta S\cdot \psi(\zeta S)-w}{\frac{\tau_1}{\beta}}\right],
  $$
  for $S,G\simiid\mathcal{N}(0,1)$.
%
Also, it holds that,
\bea
\frac{1}{n}\env{\left(\small{\ellb_1}+\frac{r}{\Large{\eps}_{ \tiny{\rm tr}}\gamma}\ellb_2^2,\,\Sn\right)}{\frac{\alpha\beta}{\tau_2\sqrt{\delta}}\Sn^{-1/2}\h + \frac{\alpha\eta\sqrt{n}}{\tau_2\left\|\widetilde{\thc_{n}^\star}\right\|^2_2}{\thc_{n}^\star}}{\frac{\alpha\eps_{\rm tr}\gamma}{\tau_2}} \rP \bar{M}\left(\frac{\alpha\beta}{\tau_2\sqrt{\delta}},\frac{\alpha\eta}{\tau_2\zeta^2};\frac{\alpha\gamma\eps_{\rm tr}}{\tau_2}\right).\label{eq:melimit}
\eea
Putting these back in \eqref{eq:minmax_9}, we conclude with the objective in \eqref{eq:minmax_app_thm5}. This completes the proof.
\end{proof}
\subsection{Case II: Correlated Features with Diagonal Covariance Matrix (Proof of Theorem \ref{thm:main})}
Note that the Moreau-envelope in \eqref{eq:melimit} is not separable in general and thus the computation of $\mathcal{M}_{\left(\ellb_1+\frac{r}{\eps_{\rm tr}\gamma}\ellb_2^2,\Sn\right)}$ may not be simplified further. By assuming $\Sn$ to be diagonal i.e., $\Sn=\mathbf{\Lambda}_n$ with diagonal entries $\lambda_{n,i}$, $i\in[n]$, it is concluded from \eqref{eq:beforeME2} that the minimization becomes separable over the entires of $\widetilde{\rhoc_n}$. In fact, it is inferred that in this case:
\bea\nn
&\frac{1}{n}\env{\left(\small{\ellb_1}+\frac{r}{\Large{\eps}_{ \tiny{\rm tr}}\gamma}\ellb_2^2,\,\Sn\right)}{\frac{\alpha\beta}{\tau_2\sqrt{\delta}}\Sn^{-1/2}\h + \frac{\alpha\eta\sqrt{n}}{\tau_2\left\|\widetilde{\thc_{n}^\star}\right\|^2_2}{\thc_{n}^\star}}{\frac{\alpha\eps_{\rm tr}\gamma}{\tau_2}} \\&\hspace{1in}
= \frac{1}{n}\sum_{i=1}^{n}\env{\ell_1+\frac{r}{\eps_{\rm tr}\gamma}\ell_2^2}{\frac{\alpha\beta}{\tau_2\sqrt{\delta\lambda_{n,i}}}\h_i + \frac{\alpha\eta\sqrt{n}}{\tau_2\left\|\widetilde{\thc_{n}^\star}\right\|^2_2}{\thc_{n,i}^\star}}{\frac{\alpha\eps_{\rm tr}\gamma}{\tau_2\lambda_{n,i}}}.\label{eq:beforeLips}
\eea
By Assumption \ref{ass:1}, we know that $\lambda_{n,i}\in (c,C)$, for all $i\in[n]$ and all $n\in \mathbb{N}$, where $c>0$. This results in  $\env{\ell_1+\ell_2^2}{\cdot}{\cdot}$ being Pseudo-Lipschitz of order 2. Thus by Assumption \ref{ass:3}, the expression in \eqref{eq:beforeLips} converges in probability to
\bea
\E_{L,H,T}\left[\env{\ell_1+\frac{r}{\eps_{\rm tr}\gamma}\ell_2^2}{\frac{\alpha\beta}{\tau_2\sqrt{\delta L}}H + \frac{\alpha\eta}{\tau_2\zeta^2}T}{\frac{\alpha\eps_{\rm tr}\gamma}{\tau_2 L}} \right],
\eea
for the standard Gaussian random variable $H$ and $(L,T)$ drawn according to distribution $\Pi$. Thus in this case, \eqref{eq:minmax_9} converges to the min-max problem in \eqref{eq:minmax_main}. The proof of the Gaussian-Mixture model is deferred to Appendix \ref{sec:GMM_analysis}. This completes the proof of Theorem \ref{thm:main}. 

\subsection{Case III: Isotropic Features}\label{sec:GLM,infty,caseIII}
When $\Sn=\mathbb{I}_n$, the final expressions can be further simplified, as the term $\|\Sn\widetilde{\thc_n}\|_2^2$ becomes decomposable into $\|\Th\widetilde{\thc_n}\|_2^2$ and $\|\Thp\widetilde{\thc_n}\|_2^2$. Here we focus on the case of $q=\infty$ for GLM and defer the analysis of $q=2$ to Section \ref{sec:GLM_ell_2}. Proceeding with the same notation as in \eqref{eq:minmax_main}, consider the following min-max objective,
\bea
&\min_{\substack{{\alpha,\tau_1,w\in\R_+,}\\ {\mu\in\R}}}\;\;\max_{\substack{{\tau_2,\beta,\gamma\in\R_+,}\\ {\eta\in\R}}} \,   
f_{_{\delta,1}}(\bar\vb) + r \al^2 + r\mu^2 + \E_{G,S}\left[\env{\Lm}{\alpha G+\mu S\, \psi(S)-w}{\frac{\tau_1}{\beta}}\right]\nn\\
&\hspace{1.5in}+\eps_{\rm tr}\gamma\,\E_{H,T}\left[\env{\ell_1}{\frac{\alpha\beta}{\tau_2\sqrt{\delta}}H + \frac{\alpha\eta}{\tau_2}T}{\frac{\alpha\eps_{\rm tr}\gamma}{\tau_2}} \right].\label{eq:minmax_ellinfty_iso}
\eea

\begin{corollary}\label{cor:iso_ell_infty}
Consider the Generalized Linear models \eqref{eq:binarymodel}. Assume the same settings and assumptions as in Theorem \ref{thm:5}, only here assume that $\Sn=\mathbb{I}_n$. Then, the high-dimensional limit for the adversarial test error $(\mathcal{E}_{\ell_\infty,\frac{\eps_{\rm ts}}{\sqrt{n}}})$ is derived as follows,
\bea\label{eq:corr5_gen_error}
\left\{\mathcal{E}_{\ell_\infty, \frac{\eps_{\rm ts}}{\sqrt{n}}}^{\tiny{GLM}}\left(\widehat{\thc_n}\right)\right\} \rP  \Pro \Big( \mu^\star \,S\,\psi( S) + \alpha^\star G < w^\star {\eps_{\rm ts}}/{\eps_{\rm tr}} \Big),
\eea
where $(\alpha^\star,\mu^\star,w^\star)$ is the unique solution to the scalar minimax problem \eqref{eq:minmax_ellinfty_iso}.
\end{corollary}
\begin{proof}
The proof follows the same steps as Theorem \ref{thm:5}. Note here that $\zeta=1$ and the random variable $L=1$. Also, in deriving \eqref{eq:minmax_3}, it suffices to write the Lagrangian equivalent formulation only for the $\ell_1$ loss and write $\|\widetilde{\thc_n}\|_2^2 = \|\Th\widetilde{\thc_n}\|_2^2 + \|\Thp\widetilde{\thc_n}\|_2^2\rP r \al^2+ r\mu^2$, which results in \eqref{eq:minmax_ellinfty_iso}. 
\end{proof}
\subsubsection{A System of Equations.}\label{sec:sys_eqs_iso}
We find solutions to the min-max problem in \eqref{eq:minmax_ellinfty_iso}(the objective of which we denote by $\bar L:\R^8\rightarrow\R$) by forming and solving $\nabla_{\bar{\vb}} \bar L = \mathbf{0}$. To compute $\nabla_{\bar \vb} \bar L$ we leverage properties of Moreau-envelopes and appropriately combine different equations in the system $\nabla _{\bar \vb}\bar L = \mathbf{0}$, so as to simplify the resulting expressions (details are provided below). This leads to the system of eight equations \eqref{eq:eta_main}. Our experiments suggest that the simplifications that lead to these, are important for a simple iterative fixed-point scheme to obtain the theoretical values of $(\alpha^\star,\mu^\star,w^\star)$.  
\begin{align}\label{eq:eta_main}
\begin{split} \hspace{-.7in}
\begin{cases}& \hspace{-.13in}\eta =  \frac{\beta}{\tau_1} \E\big[Z\left(w+\prox{\Lm}{ \mu Z+ \alpha G-w}{\tau_1/{\beta}}\right)\big] - 2\la\mu+\frac{\mu\tau_2}{\alpha} - \kappa\frac{\beta\mu}{\tau_1},\\[5pt]
&\hspace{-.13in}\mu = \E\left[T\cdot \prox{\ell_1}{{\alpha\beta H}/({\tau_2\sqrt{\delta}}) + {\alpha\eta T}/{\tau_2}}{{\alpha\eps_{\rm tr}\gamma}/{\tau_2}}\right],\\[5pt]
&\hspace{-.13in}\gamma =  -\E\left[ \envdx{\Lm}{ \mu Z+ \alpha G-w}{{\tau_1}/{\beta}}\right],\\[5pt]
&\hspace{-.13in}\beta^2 = \E\Big[\Big( \envdx{\Lm}{ \mu  Z+ \alpha G-w}{{\tau_1}/{\beta}}\Big)^2\Big],\\[5pt]
&\hspace{-.13in}\tau_1 =\frac{1}{\sqrt{\delta}} \E\left[H\cdot\prox{\ell_1}{\frac{\alpha\beta H}{\tau_2\sqrt{\delta}} + {\alpha\eta T}/{\tau_2}}{{\alpha\eps_{\rm tr}\gamma}/{\tau_2}}\right],\\[5pt]
&\hspace{-.13in}\alpha = \frac{1}{\beta+2\la\tau_1}({\tau_1\tau_2} + \beta\E\left[G\prox{\Lm}{\alpha G+\mu  Z-w}{{\tau_1}/{\beta}} \right]),\\[4pt]
&\hspace{-.13in}w= \eps_{\rm tr} \E\left[\env{\ell_1}{\frac{\alpha\beta H}{\tau_2\sqrt{\delta}} + {\alpha\eta T}/{\tau_2}}{{\alpha\eps_{\rm tr}\gamma}/{\tau_2}}\right] \\&\hspace{1in}-\frac{\eps_{\rm tr}^2\gamma\alpha}{2\tau_2}\E\Big[\left(\envdx{\ell_1}{\frac{\alpha\beta H}{\tau_2\sqrt{\delta}} + {\alpha\eta T}/{\tau_2}}{{\alpha\eps_{\rm tr}\gamma}/{\tau_2}}\right)^2\Big],\\[6pt]
&\hspace{-.13in}\tau_2^2 = \frac{\alpha^2}{\alpha^2+\mu^2}\Big({\beta^2}/{\delta}+\eta^2 +\eps_{\rm tr}^2\gamma^2\E\Big[\left(\envdx{\ell_1}{\frac{\alpha\beta H}{\tau_2\sqrt{\delta}} + {\alpha\eta T}/{\tau_2}}{{\alpha\eps_{\rm tr}\gamma}/{\tau_2}}\right)^2\Big]\\[4pt]  &\hspace{1in}-\frac{2\beta\eps_{\rm tr}\gamma}{\sqrt{\delta}}\E\left[ H \envdx{\ell_1}{\frac{\alpha\beta H}{\tau_2\sqrt{\delta}} + {\alpha\eta T}/{\tau_2}}{{\alpha\eps_{\rm tr}\gamma}/{\tau_2}}\right]\\&\hspace{1in}-2\eta\eps_{\rm tr}\gamma\E\left[T\envdx{\ell_1}{\frac{\alpha\beta H}{\tau_2\sqrt{\delta}} + {\alpha\eta T}{\tau_2}}{{\alpha\eps_{\rm tr}\gamma}/{\tau_2}}\right]\;\Big),
\end{cases}
\end{split}
\end{align}
where the random variable $Z=S\,\psi(S)$ for GLM and $Z=S+1$ for GMM, and the constant $\kappa=1$ and $2$ for GLM and GMM, respectively. Here, the Proximal operator of a function $f:\R\rightarrow\R$, at $x$ with parameter $\kappa>0$, is defined as follows,
\bea\label{eq:prox_def}
\prox{f}{x}{\kappa}\triangleq\arg\min_{v}\frac{1}{2\kappa}(x-v)^2 + f(v).
\eea

Next, we explain how to derive the Equations \eqref{eq:eta_main} for GLM. The approach for GMM is similar. Before starting, we recall useful properties of Moreau-envelops which we will leverage in deriving the equations. 

\begin{proposition}[\cite{rockafellar2009variational}]\label{propo:mor} Let $\Lm$ be a lower semi-continuous and proper function. 
Denote $
\envdx{\Lm}{x}{\kappa}\triangleq\frac{\partial{\env{\Lm}{x}{\kappa}}}{\partial x}.
$ and $\envdla{\Lm}{x}{\kappa}\triangleq\frac{\partial{\env{\Lm}{x}{\kappa}}}{\partial \kappa}.$ Then the following relations hold between first-order derivatives of Moreau-envelopes and the corresponding proximal operator, 
\bea
\envdx{\Lm}{x}{\tau}&= \frac{1}{\tau}{(x-\prox{\Lm}{x}{\tau})},\label{eq:moreaue_der1}\\
\envdla{\Lm}{x}{\tau}&= -\frac{1}{2\tau^2}{(x-\prox{\Lm}{x}{\tau})^2}\label{eq:moreaue_der2}. 
\eea


\end{proposition}
We proceed with the derivation of the Equations \eqref{eq:eta_main}. First, we start with $\nabla_\mu L$ to find that, 
\bea
\nabla_\mu \bar L &= -\frac{\mu\tau_2}{\alpha} + \eta +  \E\left[S\psi(S)\cdot\envdx{\Lm}{ \mu  S \psi(S)+ \alpha G-w}{\frac{\tau_1}{\beta}}\right] + 2\la\mu\nn \\
&=  -\frac{\mu\tau_2}{\alpha} + \eta +\frac{\beta}{\tau_1}\left(\mu - w\E\left[S\psi(S)\right] - \E\left[S\psi(S)\cdot\prox{\Lm}{ \mu S \psi(S)+ \alpha G-w}{\frac{\tau_1}{\beta}}\right]\right)\nn\\
&+2\la\mu,\nn
\eea
which gives rise to the equation below for finding $\eta^\star$:
\bea\label{eq:eight_eta}
\eta = \frac{\mu\tau_2}{\alpha} - \frac{\beta\mu}{\tau_1} + \frac{\beta w}{\tau_1}\E[S\psi(S)] + \frac{\beta}{\tau_1 } \E\left[S\psi(S)\cdot\prox{\Lm}{ \mu S \psi(S)+ \alpha G-w}{\frac{\tau_1}{\beta}}\right]-2\la\mu.
\eea
By taking derivative w.r.t $\eta$ and rewriting the derivatives based on proximal operators, we derive the equation for $\mu^\star$:
\bea
\nabla_\eta \bar L &= \mu - \frac{\eta\alpha }{\tau_2} + \frac{\eps_{\rm tr}\gamma\alpha}{\tau_2} \E\left[T \envdx{\ell_1}{\frac{\alpha\beta}{\tau_2\sqrt{\delta}}H + \frac{\alpha\eta}{\tau_2}T}{\frac{\alpha\eps_{\rm tr}\gamma}{\tau_2}}\right]\nn \\
&= \mu - \frac{\eta\alpha}{\tau_2} + \frac{\alpha\eta}{\tau_2}\E\left[T^2\right] - \E\left[T\cdot \prox{\ell_1}{\frac{\alpha\beta}{\tau_2\sqrt{\delta}}H + \frac{\alpha\eta}{\tau_2}T}{\frac{\alpha\eps_{\rm tr}\gamma}{\tau_2}}\right],\nn
\eea
which after noting that $\E[Z^2] = 1$, yields the following equation: 
\bea\label{eq:eight_mu}
\mu =  \E\left[T\cdot \prox{\ell_1}{\frac{\alpha\beta}{\tau_2\sqrt{\delta}}H + \frac{\alpha\eta}{\tau_2}T}{\frac{\alpha\eps_{\rm tr}\gamma}{\tau_2}}\right].
\eea
In order to find $\gamma^\star$, we consider $\nabla_w \bar L$ to derive that:
\bea\label{eq:eight_gamma}
\gamma =  -\E\left[ \envdx{\Lm}{ \mu  S \psi(S)+ \alpha G-w}{\frac{\tau_1}{\beta}}\right]
\eea
To proceed, we derive $\nabla_{\tau_1}\bar L$ and $\nabla_{\beta}\bar L$:
\bea
\nabla_{\tau_1} \bar L &= \frac{\beta}{2} + \frac{1}{\beta} \E \left[ \envdla{\Lm}{ \mu S \psi(S)+ \alpha G-w}{\frac{\tau_1}{\beta}}\right] \nn\\&= \frac{\beta}{2} -\frac{1}{2\beta} \E\left[\left( \envdx{\Lm}{ \mu  S \psi(S)+ \alpha G-w}{\frac{\tau_1}{\beta}}\right)^2\right]\label{eq:nabla_t1} \eea
\bea
\nabla_{\beta} \bar L &= -\frac{\alpha\beta}{\delta\tau_2} + \frac{\tau_1}{2} -\frac{\tau_1}{\beta^2}\E \left[\envdla{\Lm}{ \mu  S \psi(S)+ \alpha G-w}{\frac{\tau_1}{\beta}}\right] \nn\\ &+ \frac{\alpha\eps_{\rm tr}\gamma}{\tau_2\sqrt{\delta}} \E \left[H\cdot\envdx{\ell_1}{\frac{\alpha\beta}{\tau_2\sqrt{\delta}}H + \frac{\alpha\eta}{\tau_2}T}{\frac{\alpha\eps_{\rm tr}\gamma}{\tau_2}}\right] \nn \\[5pt]
&=-\frac{\alpha\beta}{\delta\tau_2} - \frac{\tau_1}{2} +\frac{\tau_1}{2\beta^2}\E \left[\left(\envdx{\Lm}{ \mu  S \psi(S)+ \alpha G-w}{\frac{\tau_1}{\beta}}\right)^2\right]\nn \\&+ \frac{\alpha\eps_{\rm tr}\gamma}{\tau_2\sqrt{\delta}} \E \left[H\cdot\envdx{\ell_1}{\frac{\alpha\beta}{\tau_2\sqrt{\delta}}H + \frac{\alpha\eta}{\tau_2}T}{\frac{\alpha\eps_{\rm tr}\gamma}{\tau_2}}\right].\label{eq:nabla_beta}
\eea
\eqref{eq:nabla_t1} yields the equation for deriving $\beta$ i.e.,
\bea\label{eq:eight_beta}
\beta = \left(\E\left[\left( \envdx{\Lm}{ \mu S \psi(S)+ \alpha G-w}{\frac{\tau_1}{\beta}}\right)^2\right]\right)^{1/2}.
\eea
Next, we combine \eqref{eq:nabla_t1} with \eqref{eq:nabla_beta} with proper coefficients to simplify the equations yielding, 
\bea
\frac{\nabla_{\tau_1} \bar L}{\beta} +\frac{\nabla_{\beta} \bar L }{\tau_1} &= 1-\frac{\alpha\beta}{\delta\tau_1\tau_2} + \frac{\alpha\eps_{\rm tr}\gamma}{\sqrt{\delta}\tau_1\tau_2}\E \left[H\cdot\envdx{\ell_1}{\frac{\alpha\beta}{\tau_2\sqrt{\delta}}H + \frac{\alpha\eta}{\tau_2}T}{\frac{\alpha\eps_{\rm tr}\gamma}{\tau_2}}\right] \nn\\
&=1 -\frac{\alpha\beta}{\delta\tau_1\tau_2} +\frac{1}{\sqrt{\delta}\tau_1}\left(\frac{\alpha\beta}{\tau_2\sqrt{\delta}} - \E\left[H\cdot\prox{\ell_1}{\frac{\alpha\beta}{\tau_2\sqrt{\delta}}H + \frac{\alpha\eta}{\tau_2}T}{\frac{\alpha\eps_{\rm tr}\gamma}{\tau_2}}\right]\right),
\eea
which yields the following equation:
\bea\label{eq:eight_tau1}
\tau_1 = \frac{1}{\sqrt{\delta}}\,\E\left[H\cdot\prox{\ell_p}{\frac{\alpha\beta}{\tau_2\sqrt{\delta}}H + \frac{\alpha\eta}{\tau_2}T}{\frac{\alpha\eps_{\rm tr}\gamma}{\tau_2}}\right].
\eea
In a similar way, we derive $\nabla_{\tau_2}\bar L$ and $\nabla_{\alpha} \bar L$:
\bea
\nabla_{\tau_2}\bar L = -\frac{r^2\mu^2}{2\alpha} + \frac{\alpha\beta^2}{2\delta\tau_2^2}-&\frac{\alpha}{2}+ \frac{\eta^2\alpha r^2}{2\tau_2^2} -\frac{\eps_{\rm tr}^2\gamma^2\alpha}{\tau_2^2}\E\left[\envdla{\ell_1}{\frac{\alpha\beta}{\tau_2\sqrt{\delta}}H + \frac{\alpha\eta}{\tau_2}T}{\frac{\alpha\eps_{\rm tr}\gamma}{\tau_2}}\right]\nn\\[5pt]
-&\frac{\alpha\beta\eps_{\rm tr}\gamma}{\tau_2^2\sqrt{\delta}}\E\left[ H\cdot\envdx{\ell_1}{\frac{\alpha\beta}{\tau_2\sqrt{\delta}}H + \frac{\alpha\eta}{\tau_2}T}{\frac{\alpha\eps_{\rm tr}\gamma}{\tau_2}}\right] \nn\\
-&\frac{\alpha\eta\eps_{\rm tr}\gamma}{\tau_2^2}\E\left[T\cdot\envdx{\ell_1}{\frac{\alpha\beta}{\tau_2\sqrt{\delta}}H + \frac{\alpha\eta}{\tau_2}T}{\frac{\alpha\eps_{\rm tr}\gamma}{\tau_2}}\right],\label{eq:nabla_t2}
\eea
\bea
\nabla_{\alpha}\bar L =  \frac{\mu^2\tau_2}{2\alpha^2} - \frac{\beta^2}{2\delta\tau_2}-&\frac{\tau_2}{2}- \frac{\eta^2}{2\tau_2} + \E\left[G\cdot\envdx{\Lm}{\alpha G+\mu  S\psi(S)-w}{\frac{\tau_1}{\beta}}\right] \nn \\
+&\frac{\eps_{\rm tr}^2\gamma^2}{\tau_2}\E\left[\envdla{\ell_1}{\frac{\alpha\beta}{\tau_2\sqrt{\delta}}H + \frac{\alpha\eta}{\tau_2}T}{\frac{\alpha\eps_{\rm tr}\gamma}{\tau_2}}\right]  + 2\la\al \nn\\[5pt]
+&\frac{\beta\eps_{\rm tr}\gamma}{\tau_2\sqrt{\delta}}\E\left[ H\cdot\envdx{\ell_1}{\frac{\alpha\beta}{\tau_2\sqrt{\delta}}H + \frac{\alpha\eta}{\tau_2}T}{\frac{\alpha\eps_{\rm tr}\gamma}{\tau_2}}\right] \nn\\+&\frac{\eta\eps_{\rm tr}\gamma}{\tau_2}\E\left[T\cdot\envdx{\ell_1}{\frac{\alpha\beta}{\tau_2\sqrt{\delta}}H + \frac{\alpha\eta}{\tau_2}T}{\frac{\alpha\eps_{\rm tr}\gamma}{\tau_2}}\right].\label{eq:nabla_alpha}
\eea
First, the following equation is directly followed based on \eqref{eq:nabla_t2}:
\bea
\tau_2^2 = &\frac{2\alpha}{\alpha^2+\mu^2}\bigg(\frac{\alpha\beta^2}{2\delta}+\frac{\eta^2\alpha}{2}+\frac{\eps_{\rm tr}^2\gamma^2\alpha}{2}\E\left[\left(\envdx{\ell_1}{\frac{\alpha\beta}{\tau_2\sqrt{\delta}}H + \frac{\alpha\eta}{\tau_2}T}{\frac{\alpha\eps_{\rm tr}\gamma}{\tau_2}}\right)^2\right]\nn \\ 
-&\frac{\alpha\beta\eps_{\rm tr}\gamma}{\sqrt{\delta}}\E\left[ H\cdot\envdx{\ell_1}{\frac{\alpha\beta}{\tau_2\sqrt{\delta}}H + \frac{\alpha\eta}{\tau_2}T}{\frac{\alpha\eps_{\rm tr}\gamma}{\tau_2}}\right] \nn\\&-\alpha\eta\eps_{\rm tr}\gamma\E\left[T\cdot\envdx{\ell_1}{\frac{\alpha\beta}{\tau_2\sqrt{\delta}}H + \frac{\alpha\eta}{\tau_2}T}{\frac{\alpha\eps_{\rm tr}\gamma}{\tau_2}}\right]\bigg).\label{eq:eight_tau2} \eea
In the next step, we combine \eqref{eq:nabla_t2} and \eqref{eq:nabla_alpha} to derive that,
\bea
\frac{\nabla_{\tau_2}\bar L}{\alpha} + \frac{\nabla_{\alpha}\bar L}{\tau_2} &= \frac{1}{\tau_2}\E\left[G\cdot\envdx{\Lm}{\alpha G+\mu  S\psi(S)-w}{{\tau_1}/{\beta}}\right] -1 +2\la\al/\tau_2\nn\\
&= \frac{\beta}{\tau_1\tau_2}\Big(\alpha - \E\left[G\cdot\prox{\Lm}{\alpha G+\mu S\psi(S)-w}{{\tau_1}/{\beta}} \right]\Big) - 1 + 2\la\al/\tau_2.\nn
\eea
This gives the following equation, based on the stationary point condition:
\bea\label{eq:eight_al}
\alpha = \Big({\tau_1\tau_2} + \beta\E\Big[G\cdot\prox{\Lm}{\alpha G+\mu  S\psi(S)-w}{{\tau_1}/{\beta}} \Big]\Big)\Big/(\beta+2\la\tau_1).
\eea
Finally, the following equation is derived directly based on $\nabla_\gamma \bar L$,
\bea\label{eq:eight_w}
w&= \eps_{\rm tr} \E\left[\env{\ell_1}{\frac{\alpha\beta}{\tau_2\sqrt{\delta}}H + \frac{\alpha\eta}{\tau_2}T}{\frac{\alpha\eps_{\rm tr}\gamma}{\tau_2}}\right] \nn\\&\hspace{.8in}- \frac{\eps_{\rm tr}^2\gamma\alpha}{2\tau_2}\E\left[\left(\envdx{\ell_1}{\frac{\alpha\beta}{\tau_2\sqrt{\delta}}H + \frac{\alpha\eta}{\tau_2}T}{\frac{\alpha\eps_{\rm tr}\gamma}{\tau_2}}\right)^2\right].
\eea
By putting together the equations \eqref{eq:eight_eta}, \eqref{eq:eight_mu}, \eqref{eq:eight_gamma}, \eqref{eq:eight_beta}, \eqref{eq:eight_tau1}, \eqref{eq:eight_tau2}, \eqref{eq:eight_al} and \eqref{eq:eight_w}, we end up with the system of eight equations in \eqref{eq:eta_main} for GLM. The steps required for deriving \eqref{eq:eta_main} for GMM are in a similar fashion. 
\section{Asymptotic Analysis of Adversarial Training for $\pmb{\ell_2}$ Perturbation}
\subsection{Case I: Correlated Features with General Covariance Matrix (Proof of Theorem \ref{cor:GLM_ell2})}
First, we note that when $q=p=2$, the term $\|\Sn^{-1/2}\widetilde{\rhoc_n}\|_p$ (in \eqref{eq:mimax_8}) can be rewritten as follows, 
\bea
\left\|\Sn^{-1/2}\widetilde{\rhoc_{n}}\right\|_2 = \min_{\tau_3\in\R_+}\frac{1}{2\tau_3}\left\|\Sn^{-1/2}\widetilde{\rhoc_{n}}\right\|_2^2+\frac{\tau_3}{2}.
\eea
The reason behind this reformulation is to permit the analysis of the final Moreau-envelope expression. With this, we rewrite the steps previously required to derive \eqref{eq:beforeME2}, as follows
\begin{align}
&\min_{\widetilde{\rhoc_{n}}\in\R^n} \left( \frac{\eps_{\rm tr}\gamma}{2\tau_3}+r\right)\left\|\Sn^{-1/2}\widetilde{\rhoc_{n}}\right\|_2^2 + \frac{\tau_2}{2n\alpha}\left\|\widetilde{\rhoc_{n}}\sqrt{n}+ \frac{\alpha\beta}{\tau_2\sqrt{\delta}}\h \right\|_2^2 -\eta\frac{\left\langle\widetilde{\thc_{n}^\star},\widetilde{\rhoc_{n}}\right\rangle}{\left\|\widetilde{\thc_n^\star}\right\|_2^2}\nn \\[4pt] =
&\min_{\widetilde{\rhoc_{n}}\in\R^n} \left( \frac{\eps_{\rm tr}\gamma}{2\tau_3}+r\right)\left\|\Sn^{-1/2}\widetilde{\rhoc_{n}}\right\|_2^2+ \frac{\tau_2}{2\alpha n} \left\|\widetilde{\rhoc_{n}} \sqrt{n}+ \frac{\alpha\beta}{\tau_2\sqrt{\delta}}\h - \frac{\eta\alpha\sqrt{n}}{\tau_2\left\|\widetilde{\thc_n^\star}\right\|_2^2}\widetilde{\thc_{n}^\star}\right\|^2_2 \nn\\&\hspace{1.5in}-\frac{\eta^2\alpha}{2\tau_2\left\|\widetilde{\thc_n^\star}\right\|^2_2} - \frac{\alpha\beta\eta}{\sqrt{m}{\tau_2}\left\|\widetilde{\thc_n^\star}\right\|_2^2}\left\langle\widetilde{{\thc_{n}^\star}},\h \right\rangle \nn\\[4pt]=
&\min_{\widetilde{\rhoc_{n}}\in\R^n}\left( \frac{\eps_{\rm tr}\gamma}{2\tau_3}+r\right)\left\|\Sn^{-1/2}\widetilde{\rhoc_{n}}\right\|_2^2+ \frac{\tau_2}{2\alpha n} \left\|\widetilde{\rhoc_{n}} \sqrt{n}+ \frac{\alpha\beta}{\tau_2\sqrt{\delta}}\h - \frac{\eta\alpha\sqrt{n}}{\tau_2\left\|\widetilde{\thc_n^\star}\right\|^2_2}\widetilde{\thc_n^\star}\right\|^2_2 - \frac{\eta^2\alpha}{2\tau_2\left\|\widetilde{\thc_n^\star}\right\|^2_2}  \nn\\[4pt]=
& \min_{\widetilde{\rhoc_{n}}\in\R^n} \frac{1}{n}\left( \frac{\eps_{\rm tr}\gamma}{2\tau_3}+r\right)\left\|\widetilde{\rhoc_{n}}\sqrt{n}\right\|_2^2\;\nn\\
&\;\;\;\;+ \frac{\tau_2}{2\alpha n} \left\|\Sn^{1/2}\left(\widetilde \rhoc_{n} \sqrt{n}+ \frac{\alpha\beta}{\tau_2\sqrt{\delta}}\Sn^{-1/2}\h - \frac{\eta\alpha\sqrt{n}}{\tau_2\left\|\widetilde{\thc_n^\star}\right\|^2_2}\Sn^{-1/2}\widetilde{\thc_{n}^\star}\right)\right\|^2_2 - \frac{\eta^2\alpha}{2\tau_2\left\|\widetilde{\thc_{n}^\star}\right\|^2_2}.
\label{eq:orthog}
\end{align}
To proceed, note that $\Sn^{1/2} =  \mathbf{U}_n\mathbf{\Lambda}_n^{1/2}\mathbf{U}_n^{\top}$ where $\mathbf{U}_n$ is an orthogonal matrix, therefore with a change of variable $\mathbf{U}_n^\top\widetilde{\rhoc_n}\Rightarrow\widetilde{\rhoc_n}$, the minimization based on $\widetilde{\rhoc_n}$ in \eqref{eq:orthog} is equivalent to:
\bea
\min_{\widetilde{\rhoc_{n}}\in\R^n} \frac{1}{n}\left( \frac{\eps_{\rm tr}\gamma}{2\tau_3}+r\right)\left\|\widetilde{\rhoc_{n}}\sqrt{n}\right\|_2^2+ \frac{\tau_2}{2\alpha n} \left\|\mathbf{\Lambda}_n^{1/2}\widetilde{ \rhoc_{n}} \sqrt{n}+ \frac{\alpha\beta}{\tau_2\sqrt{\delta}}\mathbf{U}_n^\top\h - \frac{\eta\alpha\sqrt{n}}{\tau_2\left\|\widetilde{\thc_n^\star}\right\|^2_2}\mathbf{U}_n^\top\widetilde{\thc_{n}^\star}\right\|^2_2 .\label{eq:beforeSepr}
\eea
It holds that $\mathbf{U}_n^\top\h\sim\h$ and following Assumption \ref{ass:3}, we have
\bea\label{eq:utheta}
\mathbf{U}_n^\top\widetilde{\thc_{n}^\star} = \mathbf{U}_n^\top\Sn^{1/2}{\thc_{n}^\star} = \mathbf{\Lambda}_n^{1/2} \mathbf{U}_n^\top{\thc_{n}^\star}  = \mathbf{\Lambda}_n^{1/2} \vb_n.
\eea
Therefore optimization over $\widetilde{\rhoc_{n}}$ becomes separable over its entries and \eqref{eq:beforeSepr} is equivalent to
\bea
&\frac{1}{n}\sum_{i=1}^n \min_{\widetilde{\rhoc_{n,i}}\in\R}\left( \frac{\eps_{\rm tr}\gamma}{2\tau_3}+r\right)\widetilde{\rhoc_{n,i}}^2+ \frac{\tau_2\lambda_{n,i}}{2\alpha n} \left(\widetilde{ \rhoc_{n,i}} + \frac{\alpha\beta}{\tau_2\sqrt{\delta\lambda_{n,i}}}\h_i - \frac{\eta\alpha\sqrt{n}}{\tau_2\left\|\widetilde{\thc_n^\star}\right\|^2_2}\vb_{n,i}\right)^2 \nn\\
&=\frac{1}{n}\left(\frac{\eps_{\rm tr}\gamma}{2\tau_3}+r\right)\sum_{i=1}^n \env{\ell_2^2}{\frac{\alpha\beta}{\tau_2\sqrt{\delta\lambda_{n,i}}}\h_i + \frac{\eta\alpha\sqrt{n}}{\tau_2\left\|\widetilde{\thc_n^\star}\right\|^2_2}\vb_{n,i}}{\frac{\eps_{\rm tr}\gamma\al+2\tau_3 r\al}{2\tau_2\tau_3\lambda_{n,i}}}\nn\\
&\rP \left(\frac{\eps_{\rm tr}\gamma}{2\tau_3}+r\right)\E_{H,V,L}\left[\env{\ell_2^2}{\frac{\alpha\beta}{\tau_2\sqrt{\delta L}}H + \frac{\eta\alpha}{\tau_2{\zeta}^2}V}{\frac{\eps_{\rm tr}\gamma\al+2\tau_3 r\al}{2\tau_2\tau_3 L}}\right]\label{eq:convergence_p}\\[3pt]
&=  \left(\frac{\eps_{\rm tr}\gamma}{2\tau_3}+r\right)\left(\frac{\eta^2\al^2}{\tau_2^2{\zeta}^4}\right)\E_{L}\left[\frac{\frac{{\zeta}^4\beta^2}{\eta^2\delta}+L}{\frac{\eps_{\rm tr}\gamma\alpha+2\tau_3 r\alpha}{2\tau_2\tau_3}+L} \right],\label{eq:aftersim}
\eea
where $H$ is standard normal and in \eqref{eq:convergence_p}, we used Assumption \ref{ass:3}, together with the fact that $\ell_2^2$ is pseudo-Lipschitz of order 2. In deriving \eqref{eq:aftersim}, we used the fact that $\env{\ell_2^2}{x}{\tau}=\frac{x^2}{\tau+1}$ and $\E[H^2]=\E[V^2]=1$.
Inserting this back in \eqref{eq:mimax_8} leads to the following objective,
\bea
&\min_{\substack{{\alpha,\tau_1,\tau_3,w\in\R_+,}\\ {\mu\in\R}}}\;\;\max_{\substack{{\tau_2,\beta,\gamma\in\R_+,}\\ {\eta\in\R}}} \,   
-\gamma w - \frac{\mu^2\tau_2}{2\alpha}\zeta^2-\frac{\alpha\beta^2}{2\delta\tau_2} - \frac{\alpha\tau_2}{2}+ \frac{\beta\tau_1}{2} + \eta \mu- \frac{\eta^2\alpha}{2\tau_2\zeta^2} +\frac{\eps_{\rm tr}\gamma\tau_3}{2}\nn \\
&\hspace{1.1in}+\E_{\,G,S}\left[\env{\Lm}{\alpha G+\mu \zeta S\cdot \psi(\zeta S)-w}{\frac{\tau_1}{\beta}}\right]\nn\\&\hspace{1.1in}+\left(\frac{\eps_{\rm tr}\gamma}{2\tau_3}+r\right)\left(\frac{\eta^2\al^2}{\tau_2^2{\zeta}^4}\right)\E_{L}\left[\frac{\frac{{\zeta}^4\beta^2}{\eta^2\delta}+L}{\frac{\eps_{\rm tr}\gamma\alpha+2\tau_3 r\alpha}{\tau_2\tau_3}+L} \right].\label{eq:minmax_10_ell2}
\eea
This completes the proof of Theorem \ref{cor:GLM_ell2}.
\subsection{Case II: Isotropic Features }\label{sec:GLM_ell_2}
Here we derive the minimax objective for $\Sn=\mathbb{I}_n$. We focus here on GLM, the extensions to GMM are achievable in light of the analysis in Section \ref{sec:GMM_analysis}.
\begin{corollary}\label{thm:ell2_iso}
Consider the Generalized Linear model \eqref{eq:binarymodel}. Let $\Sn=\mathbb{I}_n$ and $\|\thc^\star_n\|_2\rP 1$. The high-dimensional limit for the adversarial test error takes the following form,
\bea
\{\mathcal{E}_{\ell_2, \eps}^{\text{GLM}}(\thc_n)\} \rP \Pro \left( \frac{\mu^\star S\psi(S) + \al^\star G}{\sqrt{{\al^\star}^2+{\mu^\star}^2}}  < \eps \right),
\eea
where $(\al^\star,\mu^\star)$ is the unique solution to the following min-max objective,
\bea
&\min_{\mu\in\R,\alpha,\tau\in\R_+}\;\max_{\beta\in\R_+}\; \widetilde{L}\triangleq \frac{\beta\tau}{2} - \frac{\alpha\beta}{\sqrt{\delta}} +\la\al^2+\la\mu^2\nn\\&\hspace{1in}+\E_{G,S}\Big[\env{\Lm}{ \mu S\psi(S)+ \alpha G -\eps_{\rm tr}\sqrt{\alpha^2+\mu^2}}{{\tau}/{\beta}} \Big].\label{eq:minmax_ell_2_is}
\eea
\end{corollary}
%
\begin{proof}
We know that, 
\bea
\widehat{\thc_{n}} =  \min_{\thc_{n} \in \R^n} \frac{1}{m}\sum_{i=1}^m \Lm(y_i\x_i^{\top}\thc_{n}-\eps_{\rm tr}\|\thc_{n}\|_2)  + \la\|\thc_{n}\|_2^2.\nn
\eea
To proceed, we use our approach that derived \eqref{eq:minmax}, to end at a similar expression, here for $p=2$. We omit the steps as they are akin to the steps that led to \eqref{eq:minmax}. We end up with the following objective which is the counterpart of \eqref{eq:minmax} for $q=p=2$. 

\bea\label{eq:minmax_ell2}
\min_{\thc_{n}\in \R^n,\vb\in \R^m}\max_{\beta\in\R_+}  \;\frac{\mathbf{1}_m^{\top}}{m} &\Ellb\left(\vb-\eps_{\rm tr}\|\thc_{n}\|_2\mathbf{1}_m\right)  + \frac{\beta}{\sqrt{m}} \left\|-\vb + YX\Th\thc_{n}+ \g \|\Thp\thc_{n}\|_2 \right\|_2 \; \\ &
+\;  \frac{\beta \h^{\top}\Thp\thc_{n}}{\sqrt{m}}  +\la\|\thc_{n}\|_2^2\;\;= \nn
\eea
\bea
\min_{\thc_{n}\in \R^n,\vb\in \R^m}\max_{\beta,\tau\in\R_+} \;\frac{\mathbf{1}_m^{\top}}{m} &\Ellb\left(\vb-\eps_{\rm tr}\|\thc_{n}\|_2\mathbf{1}_m\right)  + \frac{\beta}{m\tau} \left\|-\vb + YX\Th\thc_{n}+ \g \|\Thp\thc_{n}\|_2 \right\|_2^2 \;+\;  \frac{\beta\tau}{2}  \nn\\& + \frac{\beta \h^{\top}\Thp\thc_{n}}{\sqrt{m}}+\la\|\thc_{n}\|_2^2 \, ,\nn
\eea
where similar to \eqref{eq:minmax_5}, here also \eqref{eq:minmax_ell2} is due to $x= \min_{\tau\in\R_+}  \frac{x^2}{2\tau} + \frac{\tau}{2}$. By minimizing w.r.t. $\thc_{n}$  and denoting $\alpha\triangleq \|\Thp\thc_{n}\|_2$, $\mu\triangleq \|\Th\thc_{n}\|_2$ we have,
\bea
\min_{\vb\in \R^m,\mu\in\R,\alpha\in\R_+}\;\max_{\beta,\tau\in\R_+} &\;\frac{\mathbf{1}_m^{\top}}{m} \Ellb\left(\vb-\eps_{\rm tr}\sqrt{\alpha^2+\mu^2}\mathbf{1}_m\right)  + \frac{\beta}{m\tau} \left\|-\vb + \mu Y X\thc_{n}^\star+ \alpha\g \right\|_2^2 \;\nn\\ &+\;  \frac{\beta\tau}{2} - \frac{\alpha\beta \h}{\sqrt{m}} + \la\|\thc_{n}\|_2^2.\nn
\eea
After $m,n\rightarrow\infty$, one can easily see that the objective simplifies to \eqref{eq:minmax_ell_2_is}.
 Additionally, by replacing $(\alpha^\star,\mu^\star)$ derived as the solution of \eqref{eq:minmax_ell_2_is}, in \eqref{eq:gen_err_adv_2}, we derive the asymptotic error of adversary. This completes the proof
 \end{proof} 
\subsubsection{A System of Equations}\label{sec:sys_eq_q2} Now, we present the corresponding fixed-point equations for the $\ell_2$ case in \eqref{eq:sys_eq_q2}. The equations are obtained by forming $\nabla \widetilde{L} = \mathbf{0}$ based on three variables $(\alpha,\mu,\kappa)$, where $\kappa:=\tau/\beta$.

\bea\label{eq:sys_eq_q2}
 \begin{split}
  \hspace{0in}
 \begin{cases}
 &\hspace{-0.15in}\E_{G,S}\Big[\Big(\envdx{\Lm}{ \mu S\psi(S) + \alpha G -\eps_{\rm tr}\sqrt{\alpha^2+\mu^2}}{\kappa}\Big)^2 \,\Big] =\frac{\al^2}{\kappa^2\delta},\\[8pt]
 &\hspace{-0.15in}\E_{G,S}\left[S\psi(S)\cdot\envdx{\Lm}{ \mu S\psi(S)+ \alpha G -\eps_{\rm tr}\sqrt{\alpha^2+\mu^2}}{\kappa} \right] =-2\la\mu \\&\hspace{1.2in}+\frac{\eps_{\rm tr}\mu}{\sqrt{\al^2+\mu^2}}\E_{G,S}\left[\envdx{\Lm}{ \mu S\psi(S)+ \alpha G -\eps_{\rm tr}\sqrt{\alpha^2+\mu^2}}{\kappa} \right],\\[8pt]
  &\hspace{-0.15in}\E_{G,S}\left[G\cdot\envdx{\Lm}{ \mu S\psi(S)+ \alpha G -\eps_{\rm tr}\sqrt{\alpha^2+\mu^2}}{\kappa} \right] = - 2\alpha\la \\&\hspace{1.2in}+\frac{\eps_{\rm tr}\al}{\sqrt{\al^2+\mu^2}}\E_{G,S}\left[\envdx{\Lm}{ \mu S\psi(S)+ \alpha G -\eps_{\rm tr}\sqrt{\alpha^2+\mu^2}}{\kappa} \right]+\frac{\al}{\delta\kappa} . 
\end{cases}
 \end{split}
 \eea

 Next, we show how to derive the saddle-point equations \eqref{eq:sys_eq_q2} from $\nabla \widetilde{L}=\mathbf{0}$. To derive the first equation in \eqref{eq:sys_eq_q2}, we can see that based on Proposition \ref{propo:mor},
\begin{align}
&\nabla_\tau\widetilde{L} = \frac{\beta}{2} - \frac{1}{2\beta}\E\Big[\left(\envdx{\Lm}{ \mu S\psi(S)+ \alpha G -\eps_{\rm tr}\sqrt{\alpha^2+\mu^2}}{{\tau}/{\beta}}\right)^2 \Big],\label{eq:tau_der_q2}\\
&\nabla_\beta\widetilde{L} = \frac{\tau}{2}- \frac{\alpha}{\sqrt{\delta}} + \frac{\tau}{2\beta^2}\E\Big[\left(\envdx{\Lm}{ \mu S\psi(S)+ \alpha G -\eps_{\rm tr}\sqrt{\alpha^2+\mu^2}}{{\tau}/{\beta}}\right)^2 \Big].\nn
\end{align}
After forming $\frac{\nabla_\tau\widetilde{L}}{\beta} + \frac{\nabla_\beta\widetilde{L}}{\tau} =0$, we can deduce that $\alpha = \tau\sqrt{\delta}.$ Since we defined $\kappa \triangleq \tau/\beta$, it follows that $\beta = \alpha/(\kappa\sqrt{\delta})$. Replacing this in \eqref{eq:tau_der_q2}, yields the first equation in \eqref{eq:sys_eq_q2}. The last two equations in \eqref{eq:sys_eq_q2}, are obtained directly from $\nabla_\mu\widetilde{L}=0$ and $\nabla_\al\widetilde{L}=0$.

For GMM \eqref{eq:G_mix}, the min-max objective and the system of equations are obtained by replacing $S\psi(S)$ with $S+1$, in  \eqref{eq:minmax_ell_2_is} and \eqref{eq:sys_eq_q2}.

\section{The Gaussian-Mixture Model Analysis}\label{sec:GMM_analysis}
\subsection{Adversarial Error of an Arbitrary Estimator}

Next lemma (restatement of Lemma \ref{lem:gen_error0} for GMM) derives the asymptotic error of a given sequence of estimators for the Gaussian-Mixture model. 
\begin{lemma}\label{lem:arbit_gen_error_gmm}
The high-dimensional limit of the Adversarial Error for the Gaussian-Mixture model with a sequence of classifiers $\{\thc_{n}\}$ is given as follows,
\bea\label{eq:GMMRob_error}
\left\{\mathcal{E}_{\ell_q, {\eps}}^{\text{GMM}}(\thc_n)\right\} \rP Q\left(\frac{\mu\widetilde{\zeta}^2-\eps u}{\sqrt{\al^2+\mu^2\widetilde{\zeta}^2}}\right),
\eea
where $Q(\cdot)$ denotes the Gaussian $Q$-function and $u,\mu$ and $\alpha$ are derived as follows,
$$\left\|\Sn^{-1/2}\widetilde{\thc_{n}}\right\|_p\rP u,\;\;\left\langle\widetilde{\thc_n^\star},\widetilde{\thc_{n}}\right\rangle\Big/\left\|\widetilde{\thc_n^\star}\right\|_2^2\rP \mu ,\;\;\left\|\Th^\perp\widetilde{\thc_{n}}\right\|_2\rP \alpha,$$
 for $\ell_p$-norm denoting the dual of the $\ell_q$-norm, $\widetilde{{\thc}_n}\triangleq \Sn^{1/2}\thc_n,\widetilde{\thc_{n}^\star} \triangleq \Sn^{-1/2}\thc_{n}^\star$ and $\Thp \in \R^{n\times n}$ defined as follows:
 $$ \Thp\triangleq \mathbb{I}_{n} - \Th, \;\;\Th\triangleq\frac{ \widetilde{\thc_{n}^\star}\widetilde{\thc_{n}^\star}^{\top}}{\left\|\widetilde{\thc_n^\star}\right\|_2^2}.$$
 Moreover, in the special case of $q=2$ and $\Sn=\mathbb{I}_n$, by denoting $\sigma\triangleq\alpha/\mu$, \eqref{eq:GMMRob_error} simplifies to, 
\bea\label{eq:gen_err_adv_2_GMM}
\left\{\mathcal{E}_{\ell_2, \eps}^{\text{GMM}}(\thc_n)\right\} \rP  Q\left(\frac{\mu}{\sqrt{\al^2+\mu^2}}-\eps\right),
\eea
\end{lemma}
\begin{proof}
Note that here $\widetilde{\thc_n^\star}$ is defined rather differently in GLM. Based on the definition of GMM, we have $\x=y\thc_n^\star  + \z$ for $\z\sim\Nn(\mathbf{0}_n,\Sn)$ and $\z=\Sn^{1/2}\bar{\z}$ for standard Gaussian vector $\bar\z$.  
We can write
\bea
\E_{\x,y}\left[\max_{\|\db\|_q<\eps} \one_{\{y\neq \sign\left\langle\x+\db,\thc_{n}\right\rangle\}}\right] &= \Pro \Big(y\langle\x,\thc_{n}\rangle - \eps\|\thc_{n}\|_p <0\Big)\nn \\
&= \Pro \Big( y\langle\z,\thc_{n}\rangle + \langle\thc_{n},\thc_{n}^\star\rangle - \eps\|\thc_{n}\|_p<0\Big)\label{eq:gm1}\\[4pt]
&= \Pro \left(y\langle\bar{\z},\widetilde{\thc_n}\rangle + \langle\widetilde{\thc_n},\widetilde{\thc_n^\star}\rangle  - \eps\left\|\Sn^{-1/2}\widetilde{\thc_{n}}\right\|_p\right)\nn\\[4pt]
&\hspace{-.5in}= \Pro \left(\langle \bar\z , \Th\widetilde{\thc_{n}}\rangle + \langle \bar\z , \Thp\widetilde{\thc_{n}}\rangle + \langle\widetilde{\thc_{n}},\widetilde{\thc_{n}^\star}\rangle -\eps\left\|\Sn^{-1/2}\widetilde{\thc_{n}}\right\|_p<0\right)\label{eq:gm2},
\eea
where \eqref{eq:gm1} and \eqref{eq:gm2} follow from the definition of the Gaussian-Mixture model and noting that $\z$ is independent of $y$. Since $\langle \bar\z , \Thp\widetilde{\thc_{n}}\rangle$ and $\langle \bar\z , \Th\widetilde{\thc_{n}}\rangle$ are independent, we can deduce that for $G,S\simiid\Nn(0,1)$, it holds that
\begin{align*}
&\left\langle\bar\z,\Th\widetilde{\thc_{n}}\right\rangle  = \frac{\left\langle \widetilde{\thc_n^\star},\widetilde{\thc_n}\right\rangle}{\left\|\widetilde{\thc_n^\star}\right\|_2^2}\left\langle \bar\z,\widetilde{\thc_n^\star}\right\rangle \rP \mu \widetilde{\zeta}S,\\[5pt]
& \left\langle\bar\z,\Thp\widetilde{\thc_{n}} \right\rangle\sim \left\|\Thp\widetilde{\thc_{n}}\right\|_2 \bar\z\rP\alpha G,\\[8pt]
&\left\langle\widetilde{\thc_n},\widetilde{\thc_{n}^\star}\right\rangle\rP \mu\widetilde{\zeta}^2.
\end{align*}

where recall that $\left\|\widetilde{\thc_n^\star}\right\|_2 = {\thc_n^\star}^\top\Sn^{-1}{\thc_n}^\star \rightarrow \widetilde\zeta$ by Assumption \ref{ass:2}. Therefore, from \eqref{eq:gm2}, we infer that, 
$$
\{\mathcal{E}_{\ell_q, {\eps}}^{\text{GMM}}(\thc_n)\} \rP \Pro \Big( \mu \widetilde{\zeta}\left(S+\widetilde{\zeta}\right)+ \alpha G - u {\eps}  < 0 \Big).
$$
This leads to \eqref{eq:GMMRob_error}. 
When $q=2$ and $\Sn=\mathbb{I}_n$, we have that $\widetilde{\zeta}=1$ and noting that $u = \sqrt{\alpha^2+\mu^2},$ leads to \eqref{eq:gen_err_adv_2_GMM}. This completes the proof. 
\end{proof}
\subsection{Asymptotic Analysis of Adversarial Training for the Gaussian-Mixture Model}
In this section, we outline the approach to the proof of Theorem \ref{thm:main} for GMM. In light of the previously described steps for GLM, here we only need to derive the corresponding min-max scalar problem for GMM. 
For the Gaussian-Mixture model we have by definition that $\x_i\sim \mathcal{N}(y_i\thc_{n}^\star,\Sn)$. Thus, the min-max ERM can be equivalently written as follows,
\bea
&\min_{\thc_{n} \in \R^n}\; \max_{\substack{{\|\db_i\|_{\infty}\le \eps}\\[2pt]{i\in[m]}}} \;\frac{1}{m}\sum_{i=1}^m \Lm \Big(y_i \Big\langle \x_i+\db_i,\thc_{n}\Big\rangle\Big) + \la\left\|\thc_{n}\right\|_2^2\nn \\&=
\min_{\thc_{n} \in \R^n} \frac{1}{m}\sum_{i=1}^m \Lm\Big(y_i\Big\langle\x_i,\thc_{n}\Big\rangle-\eps\|\thc_{n}\|_1\Big) + \la\left\|\thc_{n}\right\|_2^2  \nn\\
& = \min_{\widetilde{\thc_{n}} \in \R^n} \frac{1}{m} \sum_{i=1}^m \Lm\left(\left\langle\bar{\z}_i,\widetilde{\thc_{n}}\right\rangle+\left\langle\widetilde{\thc_{n}},\widetilde{\thc_{n}^\star}\right\rangle -\eps\left\|\Sn^{-1/2}\widetilde{\thc_{n}}\right\|_1\right)+ \la\left\|\Sn^{-1/2}\widetilde{\thc_{n}}\right\|_2^2. \nn\\
&= \min_{\widetilde{\thc_{n}} \in \R^n} \frac{1}{m} \sum_{i=1}^m \Lm\left(\left\langle\bar{\z}_i,\Th\widetilde{\thc_{n}}\right\rangle+\left\langle\bar{\z}_i,\Thp\widetilde{\thc_{n}}\right\rangle+ \left\langle\widetilde{\thc_{n}},\widetilde{\thc_{n}^\star}\right\rangle -\eps\left\|\Sn^{-1/2}\widetilde{\thc_{n}}\right\|_1\right)+ \la\left\|\Sn^{-1/2}\widetilde{\thc_{n}}\right\|_2^2 \label{eq:beforeLag_GMM}
\eea
The second step is due to the fact that $\widetilde{\thc_n^\star}\triangleq\Sn^{-1/2}\thc_n^\star,\;\;\widetilde{\thc_n}\triangleq\Sn^{1/2}\thc_n$ and that $y_i$ and $\bar\z_i\simiid\mathcal{N}(\mathbf{0},\mathbb{I}_n)$ are independent for all $i$. In the last step we used the matrices $\Th\triangleq\widetilde{\thc_n^\star}\widetilde{\thc_n^\star}^\top/\|\widetilde{\thc_n^\star}\|_2^2$ and $\Thp\triangleq\mathbb{I}_n-\Th$, to allow scalarization w.r.t. desired quantities $\al,\mu$ and also to allow using CGMT as the random variables $\langle\bar{\z}_i,\Th\widetilde{\thc_{n}}\rangle$ are $\langle\bar{\z}_i,\Thp\widetilde{\thc_{n}}\rangle$ are independent.
Next, similar to \eqref{eq:vectoront}, we can use the Lagrangian multiplier method to obtain that \eqref{eq:beforeLag_GMM} is equivalent to 
\bea
\min_{\widetilde{\thc_{n}}\in \R^n,\vb\in \R^m}\max_{\ub\in\R^m} \;\;\; \frac{\mathbf{1}_m^{\top}}{m} \Ellb\left(\vb-\eps\left\|\Sn^{-1/2}\widetilde{\thc_{n}}\right\|_1\mathbf{1}_m\right) - &\frac{\langle\ub,\vb\rangle}{m} + \frac{\left\langle\ub, \bar{Z}\Th\widetilde{\thc_{n}}\right\rangle}{m} + \frac{\left\langle\ub, \bar{Z}\Thp\widetilde{\thc_{n}}\right\rangle}{m} \nn\\&\hspace{-1in}+\frac{\langle\ub,\mathbf{1}_m\rangle}{m}\left\langle\widetilde{\thc_{n}},\widetilde{\thc_{n}^\star}\right\rangle +  \la\left\|\Sn^{-1/2}\widetilde{\thc_{n}}\right\|_2^2,\label{eq:vectoront_GMM}
\eea

The objective in \eqref{eq:vectoront_GMM} bears close similarity to its GLM counterpart in \eqref{eq:vectoront}. Note that here 
$\bar{Z}\Th\widetilde{\thc_{n}}$ and $\bar{Z}\Thp\widetilde{\thc_{n}}$ have the same role as $Y\bar{X}\Th\widetilde{\thc_{n}}$ and $Y\bar{X}\Thp\widetilde{\thc_{n}}$ in \eqref{eq:vectoront}, respectively. Here we also have an additional term $\frac{\langle\ub,\mathbf{1}_m\rangle}{m}\langle\widetilde{\thc_{n}},\widetilde{\thc_{n}^\star}\rangle$ compared to \eqref{eq:vectoront}. We recall that based on the definition, it holds that $\langle\widetilde{\thc_{n}},\widetilde{\thc_{n}^\star}\rangle\rP\mu\widetilde{\zeta}^2$. Continuing with the same technique described in Section \ref{sec:proof_GLM} that led to the objective \eqref{eq:minmax_9}, we find that for GMM, \eqref{eq:vectoront_GMM} is equivalent to the following min-max problem (details are omitted for brevity):
\bea
&\min_{\substack{{\alpha,\tau_1,w\in\R_+,}\\ {\mu\in\R}}}\;\;\max_{\substack{{\tau_2,\beta,\gamma\in\R_+,}\\ {\eta\in\R}}} \,   
-\gamma w - \frac{\mu^2\tau_2}{2\alpha}\left\|\widetilde{\thc_{n}^\star}\right\|^2_2-\frac{\alpha\beta^2}{2\delta\tau_2} - \frac{\alpha\tau_2}{2}+ \frac{\beta\tau_1}{2} + \eta \mu- \frac{\eta^2\alpha}{2\tau_2\left\|\widetilde{\thc_{n}^\star}\right\|^2_2} \nn \\
&\hspace{.9in}+\frac{1}{m}\env{{\small{\Ellb}}}{ \mu \bar{Z}\widetilde{\thc_{n}^\star}+ \alpha\g+\mu\widetilde{\zeta}^2\mathbf{1}_m-w\mathbf{1}_m}{\frac{\tau_1}{\beta}} \nn\\&\hspace{.9in}+ \frac{ \eps_{\rm tr}\gamma}{n}\env{\left(\small{\ellb_1}+\frac{r}{\Large{\eps}_{ \tiny{\rm tr}}\gamma}\ellb_2^2,\,\Sn\right)}{\frac{\alpha\beta}{\tau_2\sqrt{\delta}}\Sn^{-1/2}\h + \frac{\alpha\eta\sqrt{n}}{\tau_2\left\|\widetilde{\thc_{n}^\star}\right\|^2_2}\Sn^{-1}{\thc_{n}^\star}}{\frac{\alpha\eps_{\rm tr}\gamma}{\tau_2}}.\label{eq:minmax_9_GMM}
\eea

We have $\bar{Z}\widetilde{\thc_n^\star}\sim \widetilde{\zeta}\s$ for a standard Gaussian vector $\s$ independent of $\g$. This leads to
  $$
\frac{1}{m}\env{{\small{\Ellb}}}{ \mu \bar{Z}\widetilde{\thc_{n}^\star}+ \alpha\g+\mu\widetilde{\zeta}^2\mathbf{1}_m-w\mathbf{1}_m}{\frac{\tau_1}{\beta}} \rP \E_{G}\left[\env{\Lm}{\sqrt{\alpha^2+\mu^2\widetilde\zeta^2}G +\mu\widetilde\zeta^2-w}{\frac{\tau_1}{\beta}}\right],
  $$
for standard Gaussian random variable $G$. In particular, when $\Sn$ is a diagonal matrix, we end up with the following min-max problem based on eight scalars:

\bea
&\min_{\substack{{\alpha,\tau_1,w\in\R_+,}\\ {\mu\in\R}}}\;\;\max_{\substack{{\tau_2,\beta,\gamma\in\R_+,}\\ {\eta\in\R}}} \,   
-\gamma w - \frac{\mu^2\tau_2}{2\alpha}\widetilde{\zeta}^2-\frac{\alpha\beta^2}{2\delta\tau_2} - \frac{\alpha\tau_2}{2}+ \frac{\beta\tau_1}{2} + \eta \mu- \frac{\eta^2\alpha}{2\tau_2\widetilde{\zeta}^2} \nn \\
&\hspace{1.1in} + \E_{G}\left[\env{\Lm}{\sqrt{\alpha^2+\mu^2\widetilde\zeta^2}\,G +\mu\widetilde\zeta^2-w}{\frac{\tau_1}{\beta}}\right]
\nn\\
&\hspace{1.1in}+\eps_{\rm tr}\gamma\,\E_{L,H,T}\left[\env{\ell_1+\frac{r}{\eps_{\rm tr}\gamma}\ell_2^2}{\frac{\alpha\beta}{\tau_2\sqrt{\delta L}}H + \frac{\alpha\eta}{\tau_2\widetilde\zeta^2 L}T}{\frac{\alpha\eps_{\rm tr}\gamma}{\tau_2 L}} \right],\label{eq:minmax_final_GMM}
\eea
as desired by Theorem \ref{thm:main}. 

Proof of Theorem \ref{cor:GLM_ell2} for GMM, follows the same steps as GLM, however note that here, due to the definition of $\widetilde{\thc_n^\star}$, \eqref{eq:utheta} changes into
 \bea
 \mathbf{U}_n^\top\widetilde{\thc_{n}^\star} = \mathbf{U}_n^\top\Sn^{-1/2}{\thc_{n}^\star} = \mathbf{\Lambda}_n^{-1/2} \mathbf{U}_n^\top{\thc_{n}^\star}  = \mathbf{\Lambda}_n^{-1/2} \vb_n.
 \eea
 Thus, the resulting min-max objective has the following form,
 \bea
&\min_{\substack{{\alpha,\tau_1,\tau_3,w\in\R_+,}\\ {\mu\in\R}}}\;\;\max_{\substack{{\tau_2,\beta,\gamma\in\R_+,}\\ {\eta\in\R}}} \,   
-\gamma w - \frac{\mu^2\tau_2}{2\alpha}\widetilde{\zeta}^2-\frac{\alpha\beta^2}{2\delta\tau_2} - \frac{\alpha\tau_2}{2}+ \frac{\beta\tau_1}{2} + \eta \mu- \frac{\eta^2\alpha}{2\tau_2\widetilde{\zeta}^2} +\frac{\eps_{\rm tr}\gamma\tau_3}{2}\nn \\
&\hspace{1.1in}+\E_{G}\left[\env{\Lm}{\sqrt{\alpha^2+\mu^2\widetilde\zeta^2}\,G +\mu\widetilde\zeta^2-w}{\frac{\tau_1}{\beta}}\right]
\nn\\&\hspace{1.1in}+\left(\frac{\eps_{\rm tr}\gamma}{2\tau_3}+r\right)\left(\frac{\eta^2\al^2}{\tau_2^2{\widetilde\zeta}^4}\right)\E_{L}\left[\frac{\frac{{\widetilde\zeta}^4\beta^2}{\eta^2\delta}+ L^{-1}}{\frac{\eps_{\rm tr}\gamma\alpha+2\tau_3 r\alpha}{2\tau_2\tau_3}+L} \right].\label{eq:minmax_10_ell2_GMM}
\eea
This together with Lemma \ref{lem:arbit_gen_error_gmm}, yields the proof of Theorem \ref{cor:GLM_ell2} for GMM. 
\subsection{The Large Sample-size Limit}\label{sec:large_sample_size}
In this section, we focus on the $\delta=m/n\rightarrow\infty$ limit. In particular, we consider the Exponential loss $\Lm(t) = \exp(-t)$ and the isotropic Gaussian-mixture model and set $r=0$. We prove that for $q=2$, when $\delta\rightarrow\infty$, the adversarial test error, exactly achieves the Bayes adversarial error derived by \cite{bhagoji2019lower}. The results are summarized in the following corollary.

\begin{corollary}
Consider the Gaussian-mixture model under the same settings as Corollary \ref{thm:ell2_iso}. Let the loss function $\Lm$, be the Exponenttal loss and let $ \delta\rightarrow\infty$. Fix $\eps_{\rm ts}<1$. Then if $\eps_{\rm tr}<1$, the adversarial test error of estimators derived by adversarial training and the Bayes adversarial test error are equal. 
\end{corollary}
\begin{proof}
To see this, note that under these conditions, \eqref{eq:minmax_ell_2_is} takes the following form
\bea
\min_{\mu\in\R,\alpha,\tau\in\R_+}\;\max_{\beta\in\R_+}\; \widetilde{L}_{\delta\rightarrow\infty} =  \frac{\beta\tau}{2} + \E_G\Big[\env{\Lm}{\sqrt{\alpha^2 + \mu^2}\, G  +\mu-\eps_{\rm tr}\sqrt{\alpha^2+\mu^2}}{{\tau}/{\beta}} \Big].\label{eq:minmax_ell_2_is_inftydel}
\eea
In light of Proposition \ref{propo:mor}, $\mathcal{M}_{\Lm}(x;\cdot)$ is a decreasing function for all $x$. This gives,
\bea
\lim_{\delta\rightarrow\infty}\beta^\star(\delta) = \infty ,\;\;\;\;\; \lim_{\delta\rightarrow\infty}\tau^\star(\delta) = 0.
\eea
Since $\lim_{\kappa\rightarrow\infty}\env{\Lm}{x}{\kappa} = \Lm(x)$ for all $x$, we deduce that,
\bea
(\al^\star,\mu^\star )&= \arg \min_{\al\in\R_+,\mu\in\R} \;\;\E_G\left[ \exp\left(\eps_{\rm tr}\sqrt{\al^2+\mu^2}-\mu+\sqrt{\al^2+\mu^2}G\right)\right]\nn\\
&= \arg \min_{\al\in\R_+,\mu\in\R} \;\;\exp\left(\eps_{\rm tr}\sqrt{\al^2+\mu^2}-\mu + (\al^2+\mu^2)/2\right)\nn\\
&=  \arg \min_{\al\in\R_+,\mu\in\R} \;\;\eps_{\rm tr}\sqrt{\al^2+\mu^2}-\mu + (\al^2+\mu^2)/2,\nn
\eea
which results in $(\al^\star,\mu^\star) = (0,1-\eps_{\rm tr})$. Plugging these in \eqref{eq:gen_err_adv_2_GMM}, we derive the following for the large sample-size limit of the generalization error of adversarial training, conditioned on $\eps_{\rm tr}<1$,
\bea\label{eq:lss-ell2}
\lim_{\delta\rightarrow\infty}\mathcal{E}_{\ell_2, \eps_{\rm ts}}^{\text{GMM}}(\widehat{\thc_n}) = Q\Big(1-\eps_{\rm ts}\Big).
\eea
On the other hand, based on \cite{bhagoji2019lower}, the Bayes adversarial error for isotropic GMM is derived as follows,
\bea\nn
\mathcal{E}_{\ell_2, \eps_{\rm ts}}^{\text{GMM}}({\rm OPT}) =  Q\left(\min _{\|\mathbf{z}\|_{q} \leq \varepsilon_{\mathrm{ts}}}\left\|\boldsymbol{\theta}^{\star}-\mathbf{z}\right\|_2\right).
\eea
In particular, noting that $\|\thc^\star\|_2\rP1$ (by Assumption \ref{ass:2}) and $q=2$, we find that the Bayes adversarial error in this case is 
\bea
\mathcal{E}_{\ell_2, \eps_{\rm ts}}^{\text{GMM}}({\rm OPT}) \rP Q\Big(\max\{1-\eps_{\rm ts},0\}\Big).\label{eq:lss-bayes}
\eea
Comparing \eqref{eq:lss-bayes} with \eqref{eq:lss-ell2} reveals that in the infinite sample size limit, if $\eps_{\rm ts}<1$, by choosing any $\eps_{\rm tr}<1$, the test error of adversarial training reaches the Bayes adversarial error. As a remark, it can be readily shown that for the general case of $\|\thc^\star\|_2\rP c$, the same results hold for $\eps_{\rm tr}<c$ and $\eps_{\rm ts}<c$.

\end{proof}

\end{document}